\documentclass[preprint,12pt]{elsarticle}

\usepackage{amssymb}
\usepackage{amsmath}

\usepackage{graphicx}
\usepackage{amsmath}
\usepackage{amssymb}
\usepackage{amsthm}
\usepackage{bbm}
\usepackage{subfigure}
\usepackage{times}
\usepackage{latexsym}
\usepackage{microtype}
\usepackage{color}
\usepackage{xcolor}         
\usepackage{colortbl}
\usepackage{float}
\usepackage{tabularx,multirow}
\usepackage{hyperref}
\usepackage{breakurl}
\makeatletter
\def\UrlAlphabet{%
      \do\a\do\b\do\c\do\d\do\e\do\f\do\g\do\h\do\i\do\j%
      \do\k\do\l\do\m\do\n\do\o\do\p\do\q\do\r\do\s\do\t%
      \do\u\do\v\do\w\do\x\do\y\do\z\do\A\do\B\do\C\do\D%
      \do\E\do\F\do\G\do\H\do\I\do\J\do\K\do\L\do\M\do\N%
      \do\O\do\P\do\Q\do\R\do\S\do\T\do\U\do\V\do\W\do\X%
      \do\Y\do\Z}
\def\UrlDigits{\do\1\do\2\do\3\do\4\do\5\do\6\do\7\do\8\do\9\do\0}
\g@addto@macro{\UrlBreaks}{\UrlOrds}
\g@addto@macro{\UrlBreaks}{\UrlAlphabet}
\g@addto@macro{\UrlBreaks}{\UrlDigits}

\usepackage{makecell}
\usepackage{booktabs}

\usepackage{bm}

\usepackage{ulem}

\usepackage{algpseudocode}
\usepackage{algorithm}
\algnewcommand\algorithmicforeach{\textbf{for}}
\algdef{S}[FOR]{ForEach}[1]{\algorithmicforeach\ #1\ \algorithmicdo}

\newtheorem{theorem}{Theorem}[section]
\newtheorem{assumption}{Assumption}[section]

\makeatletter
\renewcommand*{\thefootnote}{\fnsymbol{footnote}}  
\makeatother

\journal{Information Sciences}

\begin{document}

\begin{frontmatter}



\title{On the Discriminability of Self-Supervised Representation Learning}


\author[a,b]{Zeen~Song\textsuperscript{\dag}} 
\author[a,b]{Wenwen~Qiang\textsuperscript{\dag*}} 
\author[a,b]{Changwen~Zheng}
\author[c]{Fuchun~Sun}
\author[d]{Hui~Xiong}

\affiliation[a]{organization={University of Chinese Academy of Sciences},
            city={Beijing},
            postcode={100190},
            country={China}}

\affiliation[b]{organization={National Key Laboratory of Space Integrated Information System, Institute of Software Chinese Academy of Sciences},
            city={Beijing},
            postcode={100190},
            country={China}}

\affiliation[c]{organization={Department of Computer Science and Technology, Tsinghua University},
            city={Beijing},
            postcode={100190},
            country={China}}

\affiliation[d]{organization={Hong Kong University of Science and Technology},
            country={China}}

\begin{abstract}
Self-supervised learning (SSL) has recently shown notable success in various visual tasks. However, in terms of discriminability, SSL is still not on par with supervised learning (SL). This paper identifies a key issue, the ``crowding problem," where features from different classes are not well-separated, and there is high intra-class variance. In contrast, SL ensures clear class separation. Our analysis reveals that SSL objectives do not adequately constrain the relationships between samples and their augmentations, leading to poorer performance in complex tasks. 
We further establish a theoretical framework that connects SSL objectives to cross-entropy risk bounds, explaining how reducing intra-class variance and increasing inter-class separation can improve generalization.
To address this, we propose the Dynamic Semantic Adjuster (DSA), a learnable regulator that enhances feature aggregation and separation while being robust to outliers. Comprehensive experiments conducted on diverse benchmark datasets validate that DSA leads to substantial gains in SSL performance, narrowing the performance gap with SL. The source code is released at: ~\href{https://github.com/ZeenSong/DSA}{https://github.com/ZeenSong/DSA}.
\let\thefootnote\relax
\footnotetext{\textsuperscript{\dag} Equal contribution.}
\footnotetext{\textsuperscript{*} Corresponding author.}
\end{abstract}




\end{frontmatter}



\section{Introduction}
\label{sec:introduction}

Learning discriminative feature representations in the absence of supervised signals has long been a prominent and widely explored research area in machine learning. Recently, self-supervised learning (SSL) has garnered significant attention due to its remarkable performance on various downstream tasks %
{}%
{{\cite{IS_SSL1}}}, including image classification %
{%
{{\cite{IS_SSL2}}}, object detection \cite{grillBootstrapYourOwn2020}, semantic segmentation \cite{chenExploringSimpleSiamese2021}, and transfer learning %
{}%
{{\cite{qiang2021robust}}}.

Despite the continuous improvement in SSL performance 
{\cite{gui2024survey}}, a notable gap persists between SSL and supervised learning (SL), especially in tasks that require fine-grained discrimination. Understanding the underlying reasons for this gap remains an ongoing challenge. 

To delve deeper into this issue, we conduct  experiments to illustrate the fundamental differences in the data characteristics learned by supervised and self-supervised methods. Specifically, we present a comparative analysis of five well-known SSL methods: SimCLR \cite{chenSimpleFrameworkContrastive2020}, BYOL \cite{grillBootstrapYourOwn2020}, Barlow Twins \cite{zbontarBarlowTwinsSelfSupervised2021}, SwAV \cite{caronUnsupervisedLearningVisual2020}, and MAE \cite{he2022masked}. For comparison, we also visualize the features learned by a supervised method. 

The  {illustration} of features  {produced by} these methods on the ImageNet dataset~\cite{dengImageNetLargescaleHierarchical2009}  {can be seen in} Fig.~\ref{fig:tsne}. From Fig.~\ref{fig:tsne_simclr}  {through} \ref{fig:tsne_supervised}, we observe that features obtained by both SSL and SL methods exhibit clustering characteristics, meaning points of the same class are grouped together. However, SSL methods display a large intra-class variance, which causes points at the edges of different classes to overlap, a phenomenon we call the \textbf{crowding problem}. In contrast, SL methods not only produce fine-grained intra-class features but also show clear separation between different classes. We also provide quantified evidence in Table \ref{tab:quantify}.

\begin{figure*}[t]
     \centering
     \subfigure[{SimCLR}]{\includegraphics[width=0.3\textwidth]{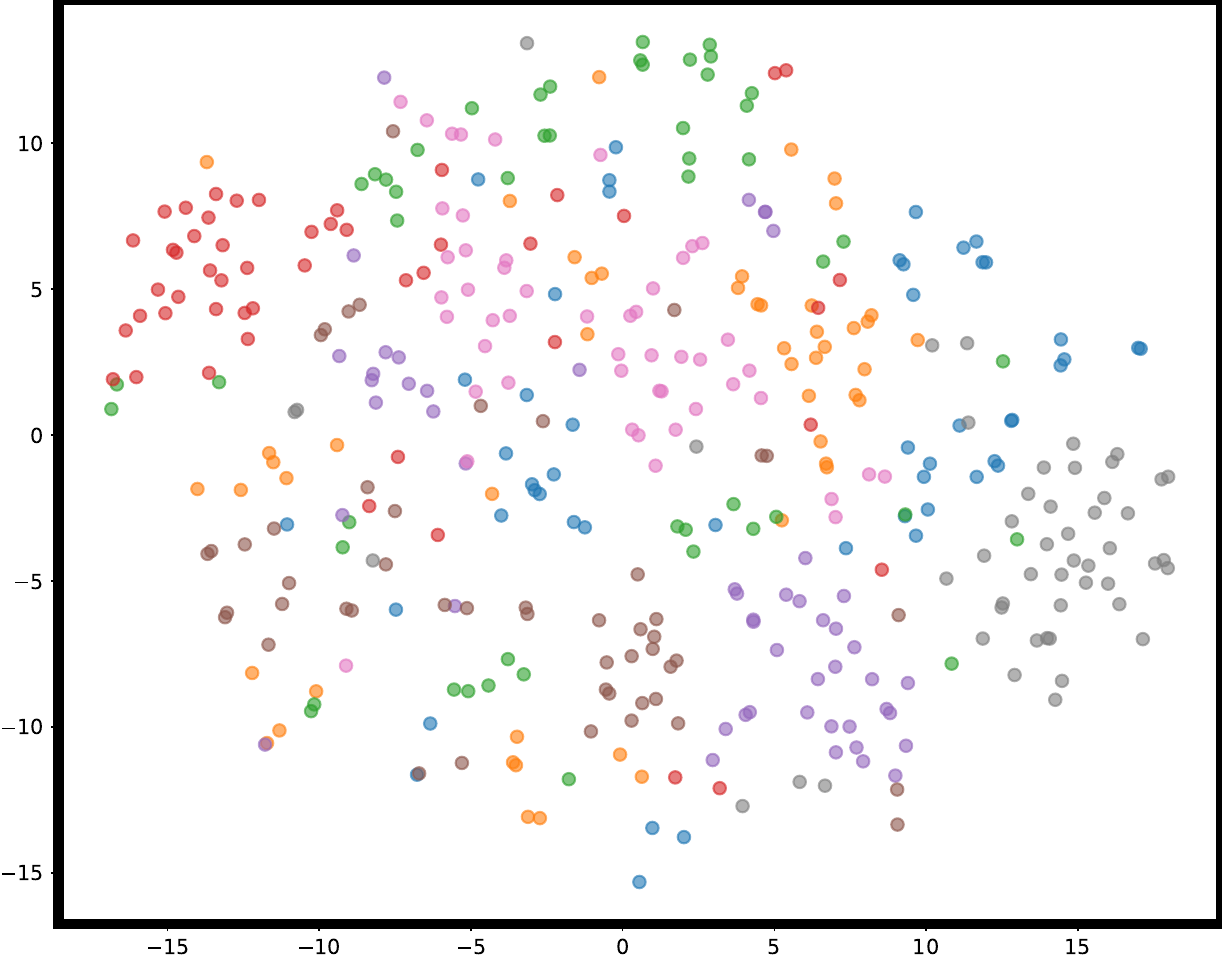}\label{fig:tsne_simclr}}
     \subfigure[{BYOL}]{\includegraphics[width=0.3\textwidth]{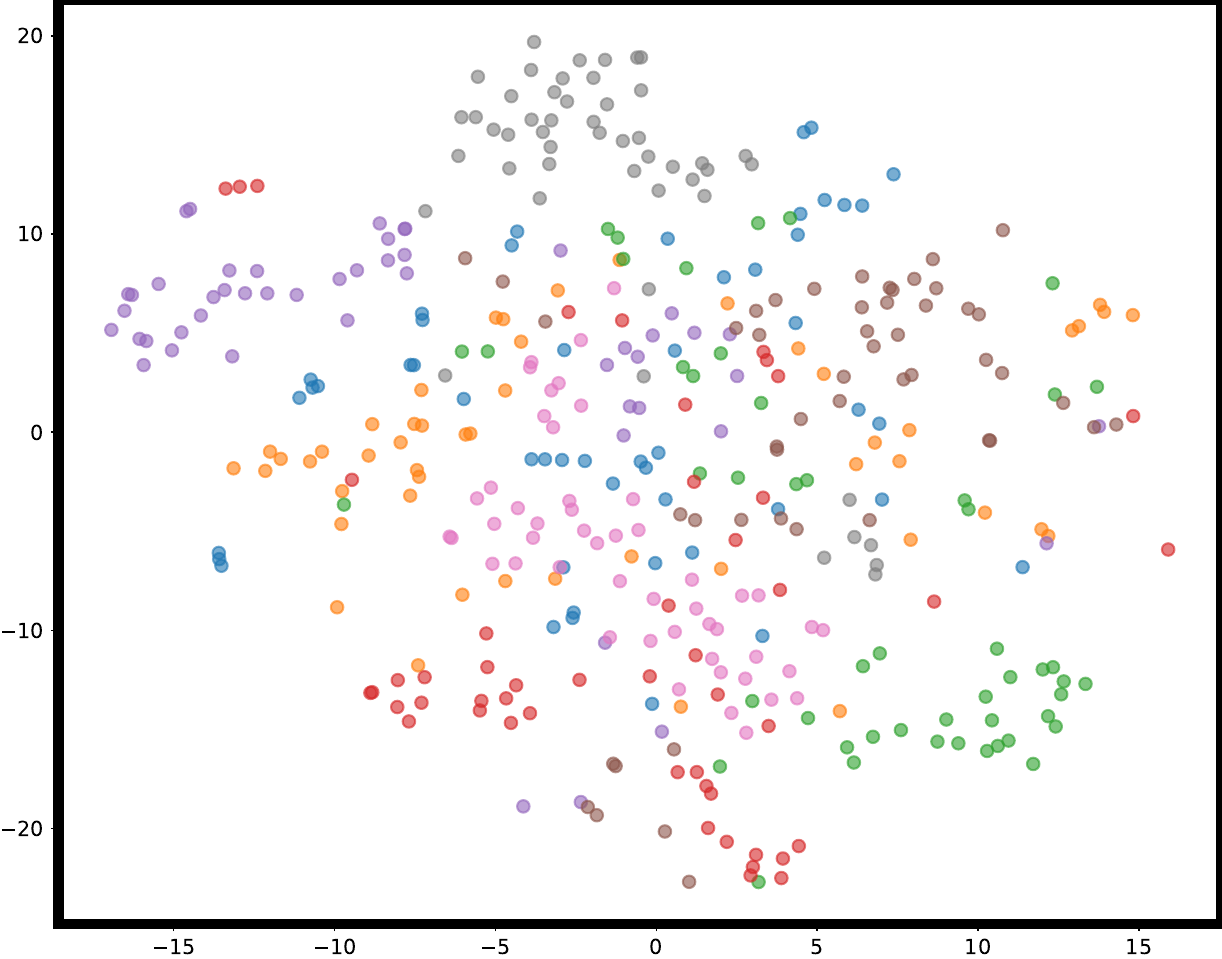}\label{fig:tsne_byol}}
     \subfigure[{Barlow Twins}]{\includegraphics[width=0.3\textwidth]{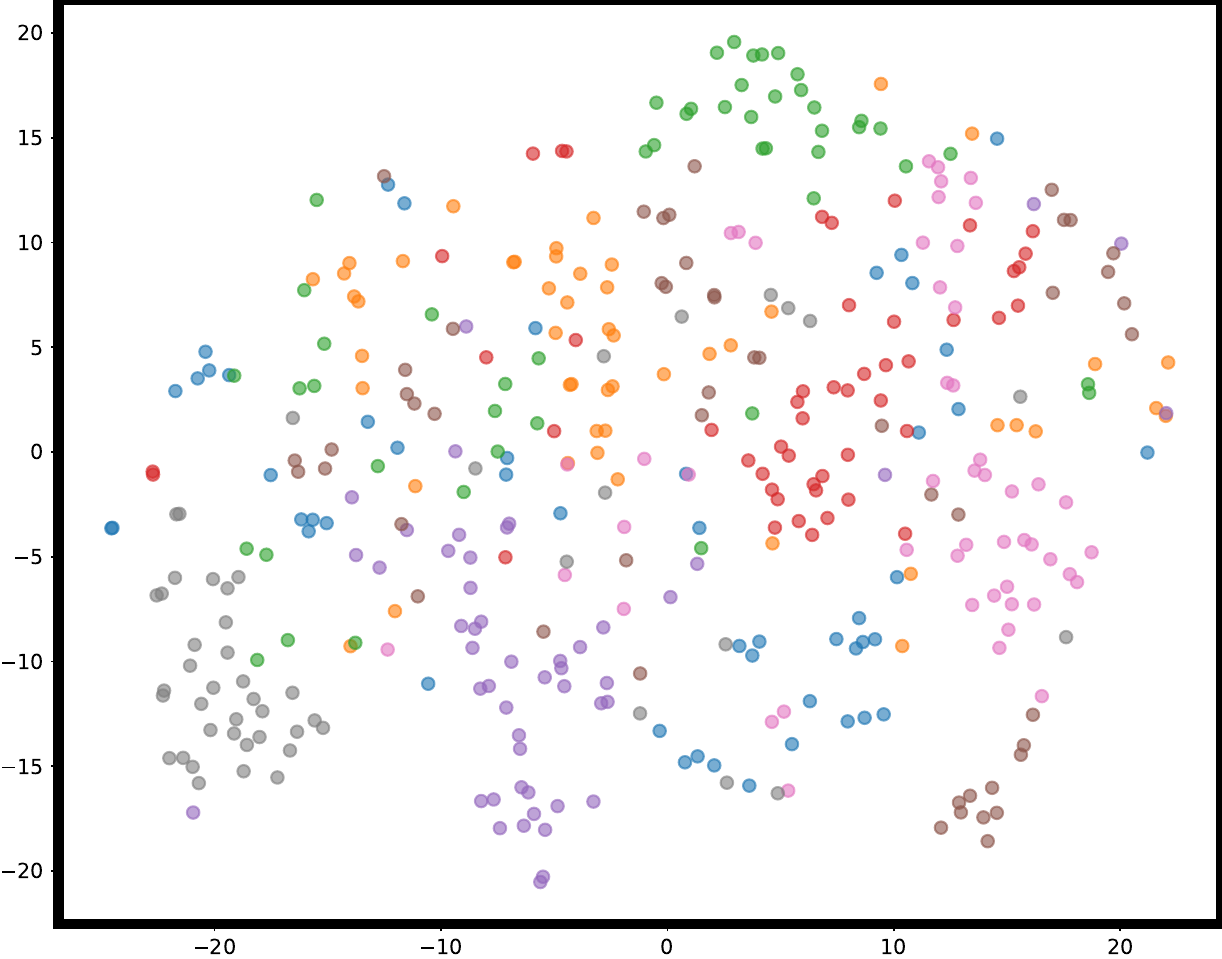}\label{fig:tsne_barlow}}
     \subfigure[{SwAV}]{\includegraphics[width=0.3\textwidth]{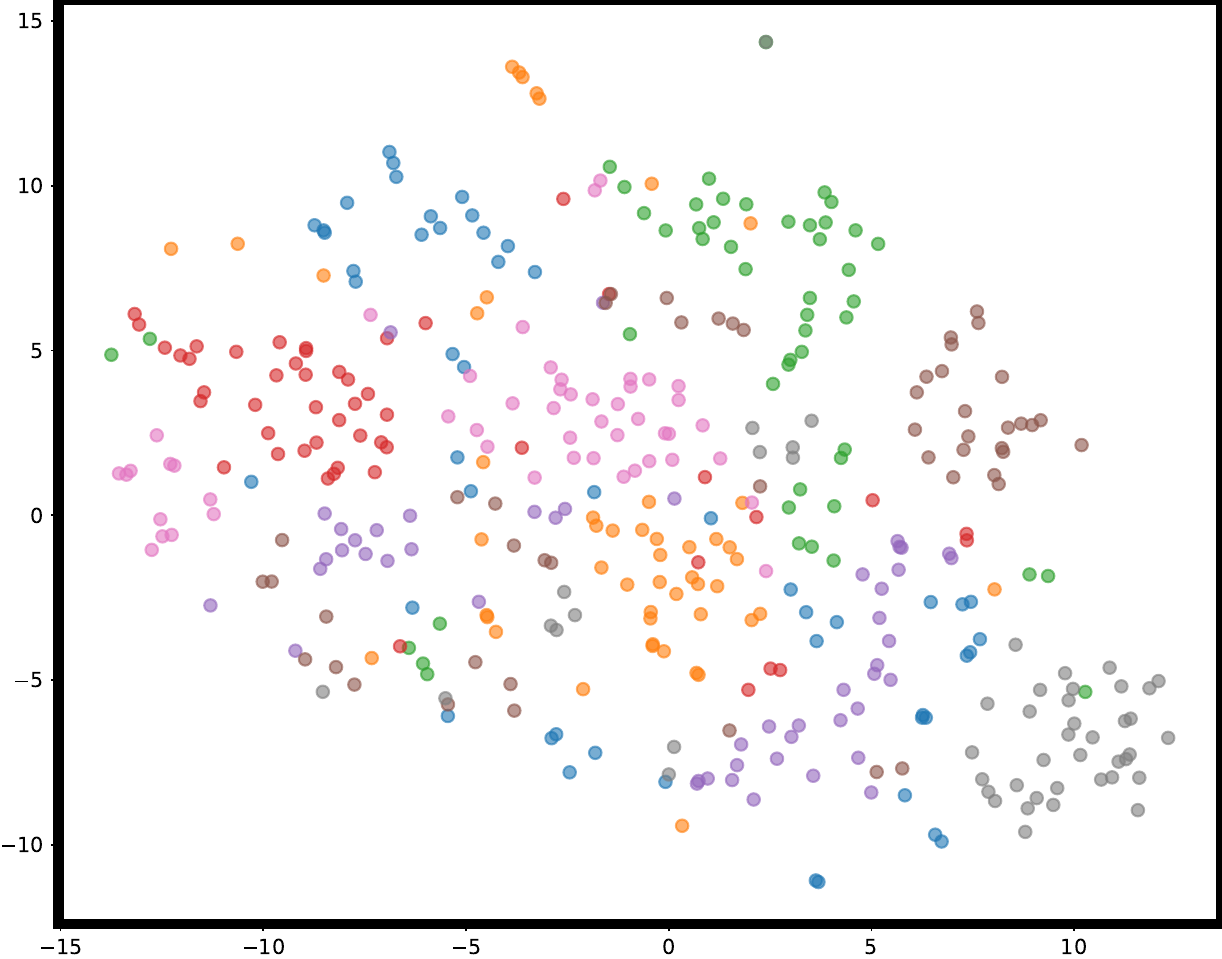}\label{fig:tsne_swav}}
     \subfigure[{MAE}]{\includegraphics[width=0.3\textwidth]{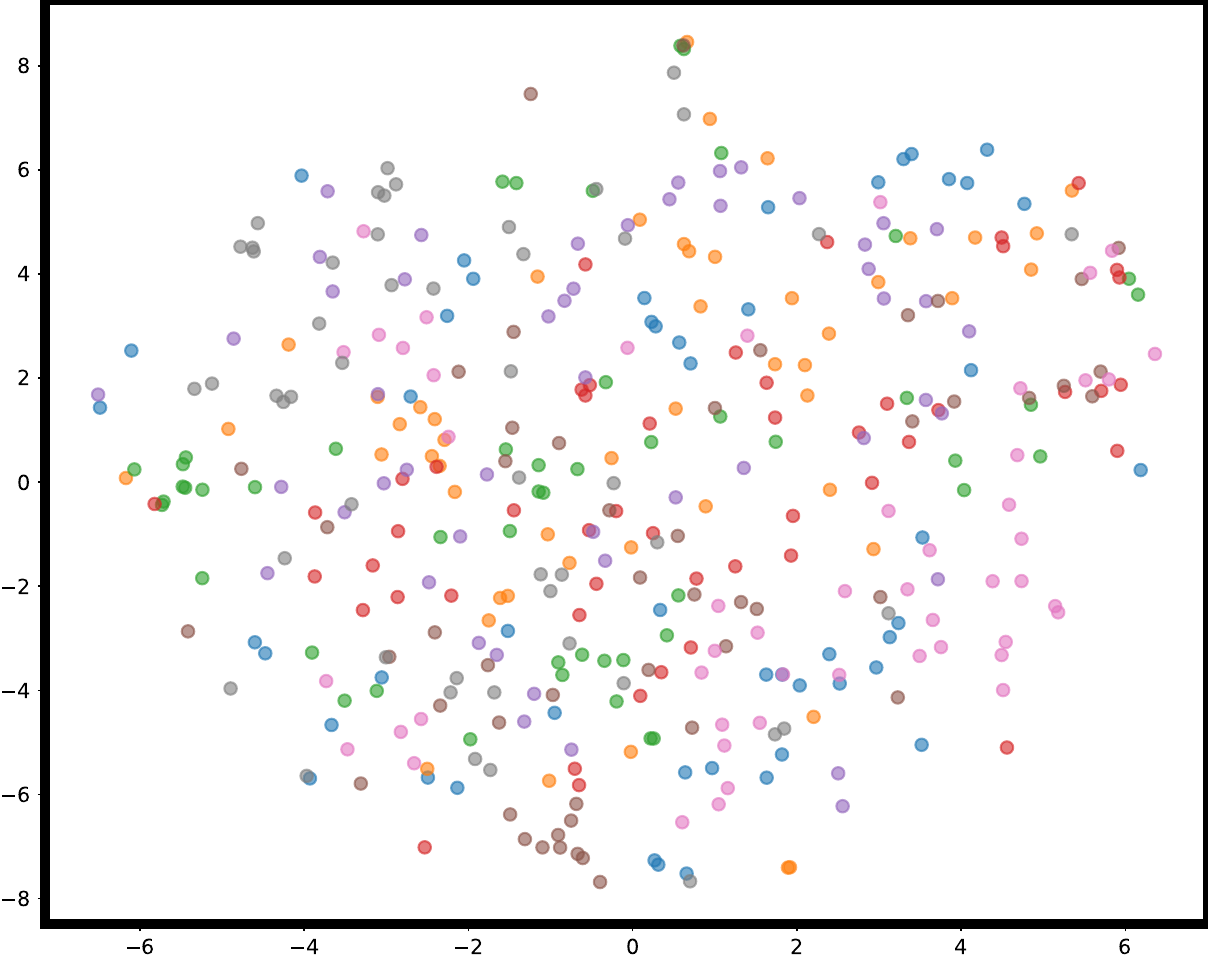}\label{fig:tsne_mae}}
     \subfigure[{Supervised}]{\includegraphics[width=0.3\textwidth]{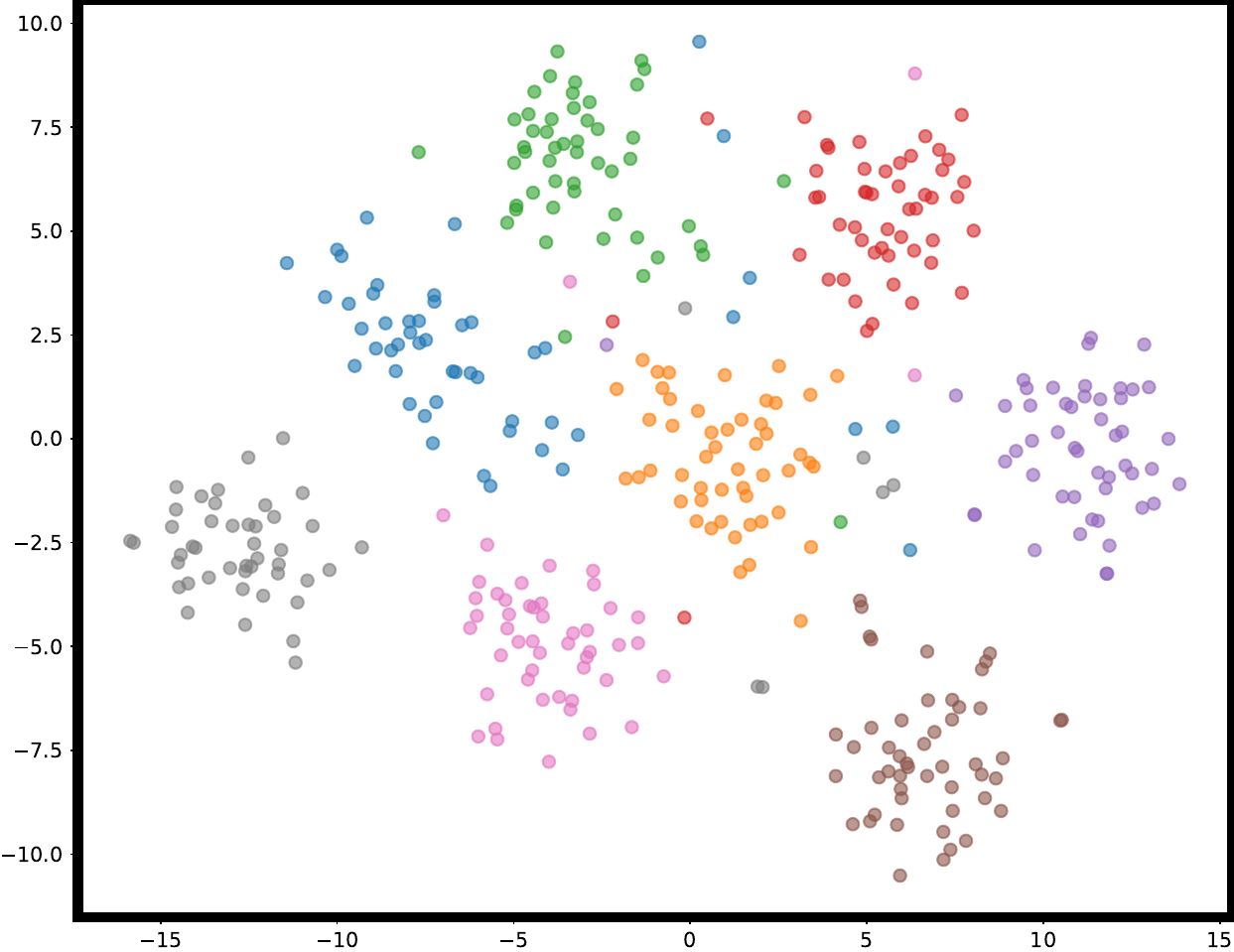}\label{fig:tsne_supervised}}
    \caption{Data distribution visualization based on 8 random classes of the test set of ImageNet in the feature space. (a) - (e) corresponds to the visualization results of the self-supervised method while (f) corresponds to the visualization results of the supervised method. We can observe that the \textbf{crowding problem} is present in SimCLR, BYOL, Barlow Twins, SwAV, and MAE. }
    \label{fig:tsne}
\end{figure*}

We further interpret the observations mentioned earlier from both empirical and theoretical perspectives. For empirical analysis, we find that SSL methods lack a specific component in their objective function that explores the relationships between different samples. In contrast, supervised methods leverage annotation information to  {promote intra-class compactness and inter-class dispersion within the learned representation space}.
{For theoretical analysis, we establish a novel theoretical framework explicitly connecting SSL objectives to the generalization error bound of supervised classification, demonstrating that %
{while}  minimizing intra-class variance {and at the same time} maximizing inter-class distance %
can effectively  {tighten} the upper bound on classification risk. Based on these findings, we conclude that the performance gap between SL and SSL methods may be attributed to the SSL methods' limited exploration of the dynamics between samples. Specifically, the objective function of SSL methods lacks a term that constrains points of similar semantics to be closer and points of dissimilar semantics to be further apart in the feature space. These insights provide valuable guidance for designing more effective regularization methods to close the performance gap between SSL and SL.}

 {Motivated by the above observations,} we propose  ``Dynamic Semantic Adjuster'' (DSA), {a learnable regulator} that can be seamlessly integrated into existing SSL methods . 
DSA  {is composed of two core components: an arranging module and a scoring module.}  
The arranging module aggregates similar samples while effectively distinguishing dissimilar ones in the feature space using a learnable regulator matrix.  
However, due to the poor performance of the feature extractor in the initial stages of training, the regulator matrix may be inaccurate, leading to incorrect aggregation and separation. To address this, we propose the scoring module, which aims to ensure that the regulator matrix preserves the local structure of the input sample space. This enhancement allows the matrix to more effectively explore similarity and dissimilarity in the early stages of training. Additionally, the scoring module ensures that the proposed method remains robust against outliers.

 {The rest of this paper is organized as follows.}  
Section~\ref{sec:relate}  {reviews} recent advances in SSL approaches.  
Section~\ref{sec:Preli} formalizes a unified framework for SSL and introduces the notation used throughout the paper.  
Section~\ref{sec:moti_anal} presents motivating experiments along with empirical and theoretical analyses.  
Section~\ref{sec:method} describes the proposed DSA  {in depth}.  
Section~\ref{sec:exp} reports comprehensive experiments on image and video benchmarks, including ablation studies of DSA.  
Finally, Section~\ref{sec:conclu} concludes the paper, and Section~\ref{sec:limit} discusses current limitations and outlines directions for future research. The main contributions are summarized as:
\begin{itemize}
    \item We show for the first time that existing SSL methods suffer from the crowding problem. We empirically analyze the cause of the crowding problem as the absence of a term in the objective function of SSL that can explicitly explore relationships between augmentations generated from samples of different ancestors.
    \item We provide a theoretical analysis that connects self-supervised learning, representation learning, and generalization bounds, demonstrating that the performance gap between SSL and SL primarily stems from the inability of SSL methods to effectively capture the aggregation and separation between augmentations that are generated by samples of different ancestors.
    \item We propose a novel method called ``Dynamic Semantic Adjuster" (DSA), which can make points of the same class cluster and points of different classes separate from each other. Notably, our method is insensitive to outliers and can be seamlessly integrated into existing SSL models, facilitating its practical applicability.
    \item We provide extensive empirical evaluations to substantiate the efficacy of the proposed DSA in improving the performance of various state-of-the-art SSL methods across diverse tasks and datasets.
\end{itemize}

\section{Related Work}
\label{sec:relate}
In this section, we briefly review the SSL methods, especially those that also aim to analyze the discriminability of SSL. We also review the SSL methods most relevant to our work and highlight the differences between these methods and ours.

{\textbf{Self-Supervised learning}. According to different learning paradigms, the SSL methods can be categorized into two classes, namely the augmentation-based methods and reconstruction-based methods \cite{gui2024survey}.}

\textit{Augmentation based}.
The core idea behind augmentation-based methods is to create different views of a single sample through data augmentation while simultaneously making the representations of these different views more similar \cite{gui2024survey}. One of the most successful augmentation-based methods is SimCLR \cite{chenSimpleFrameworkContrastive2020}, which utilizes the InfoNCE loss to group views from the same sample and repel views from different samples. One of the main drawbacks of SimCLR is the requirement for a large number of negative samples, i.e., views from different samples. Without sufficient negative samples, the problem of collapse can occur, where the shared representation ignores the input and remains constant. To address this limitation, negative-free methods are proposed. Recent methods such as BYOL \cite{grillBootstrapYourOwn2020}, SimSiam \cite{chenExploringSimpleSiamese2021}, and DINO \cite{caron2021emerging} have shown that collapse can be avoided within a knowledge distillation framework. Another line of negative-free methods, such as Barlow Twins \cite{zbontarBarlowTwinsSelfSupervised2021}, W-MSE \cite{ermolov2021whitening}, and VICReg \cite{Vicregl}, propose to prevent collapse by maximizing the information content of the embeddings.

\textit{Reconstruction based}.
Reconstruction-based SSL methods attempt to learn useful representations by reconstructing the original samples. Traditionally, an autoencoder takes the whole sample as input, feeds it into the encoder to obtain the latent representation, and then reconstructs the sample with a decoder.  Recently, inspired by the success of mask modeling in natural language processing, the reconstruction method with mask modeling has gained significant attention. Specifically, BEiT \cite{bao2021beit} and masked autoencoder (MAE) \cite{he2022masked} randomly mask the contents of visual samples and aim to reconstruct the whole sample using only the unmasked parts.

We demonstrate in Section \ref{sec:Preli} that current SSL methods share a unified objective and our proposed DSA is a plug-and-play module that can be integrated with any SSL methods.

\textbf{Discriminability Analysis}.
Discriminability refers to the easiness of separating feature representations from different categories. The concept of learning discriminative representations was first introduced by Linear Discriminant Analysis (LDA) \cite{balakrishnama1998linear}, which aims to find a projection into a lower-dimensional space that maximizes the separation between multiple classes while retaining as much class-discriminatory information as possible. Various studies have investigated discriminability in the context of SSL methods. For example, Saunshi et al. \cite{saunshiTheoreticalAnalysisContrastive2019} connect the SSL objective with supervised downstream error through instance discrimination. Chen et al. \cite{chen2022learning} proposed quantifying discriminability in SSL using the min-max distance ratio. In our research, we theoretically demonstrate that the downstream classification error risk is bounded by intra-class variance and inter-class difference, highlighting that the discriminability of feature representations significantly impacts downstream performance.

\textbf{SSL with Clustering}.
Clustering aims to group semantically similar samples. Traditionally, this unsupervised learning algorithm has been extensively used for small-scale, low-dimensional unlabeled data \cite{bishop2006pattern}. In recent years, several studies have integrated clustering algorithms with  {SSL \cite{caronUnsupervisedLearningVisual2020, liPrototypicalContrastiveLearning2021}}. {These methods typically use a K-means-like paradigm, treating the number of centroids as a hyperparameter.} Our research introduces a learnable regulator designed to aggregate similar samples while distinguishing dissimilar ones, thus positioning our approach as a form of clustering. .  \textbf{Since  {the true number of classes is unknown} during the training phase, samples  {from} different classes may be clustered together,  
and samples  {from the same class may be assigned to different class centroids}.  
Therefore,  {the crowding problem can also arise}.  
In contrast, our method does not require a predefined number of clusters,  {but instead groups} similar samples in an online manner.  
Furthermore, our DSA approach not only clusters similar samples {effectively} but also separates dissimilar ones.}

\section{Preliminaries}
\label{sec:Preli}
In this section, we first present a unified framework for SSL methods and briefly introduce prominent SSL approaches, including SimCLR \cite{chenSimpleFrameworkContrastive2020}, BYOL \cite{grillBootstrapYourOwn2020}, Barlow Twins \cite{zbontarBarlowTwinsSelfSupervised2021}, and Masked Autoencoder \cite{he2022masked}.

Formally,  {consider} a minibatch of  {training samples} denoted,  
where $x_i$  {denotes}  {a specific sample} and $N$  {is} the  {total size of the batch}.  
The  {augmentation-based approaches perform} random data augmentations (e.g., random crop)  
 {to produce a pair of augmented samples from a randomly chosen $x_i$}, namely $x^1_i$ and $x^2_i$.  
Similarly, the reconstruction-based methods  {mask $x_i$ randomly to form views $x_i^1$ and $x^2_i$}.  
The augmented dataset is denoted {as $X_{tr}^{aug} = \{x_i^l\}_{i=1,\dots,N}^{l=1,2}$}.  
The samples in ${X_{tr}}$ are considered {to be the original counterparts of those in $X_{tr}^{aug}$}.  
The augmented dataset $X_{tr}^{aug}$ is then fed  {into a feature encoder $f$ that produces representations $r_i^l = f(x_i^l)$}, where $i \in \{1,\dots,N\}$ and $l \in \{1,2\}$.  
A projection head $f_p$ is applied to $r_i^l$ to get the feature embedding $z_i^l$.  
For simplicity, we analyze the case with only two views, although the following analysis also applies to cases with more than two views.  

The SSL objective consists of two components:  {representation alignment and distribution constraint}.  
The alignment component  {is designed to enhance} the similarity between the feature embeddings of the two views that share the same ancestor sample. The constraint component introduces additional prior knowledge to the learning process, such as the distribution of the embedding space and the parameter update rules. Consequently, SSL methods can be unified under a common framework, as expressed below:
\begin{equation}
    \label{eq:SSL_unify}
    \min_{f,f_p} \mathcal{L}_{\text{align}}(X_{tr}^{aug},f,f_p) + \mathcal{L}_{\text{constrain}}(X_{tr}^{aug},f,f_p),
\end{equation}
where $\mathcal{L}_{\text{align}}$ and $\mathcal{L}_{\text{constrain}}$ denote the objectives of the alignment and constraint losses, which we detail in the following analysis.

SimCLR \cite{chenSimpleFrameworkContrastive2020} randomly selects an anchor sample $x_i^l$ from the augmented training set $X_{^{tr}}^{aug}$. The sample $x_i^{3-l}$ is considered as the positive sample related to $x_i^l$ for they share the same ancestor $x_i$. The remaining samples $X^-=X^{aug}_{tr}\setminus \{x_i^{l},x_i^{3-l}\}$ are considered as the negative samples related to $x_i^l$. 

The objective of SimCLR is defined as follows:
\begin{equation}\label{eq:NCE}
\scalebox{1}{
$
\mathcal{L}_{\mathrm{NCE}} = \sum\limits_{x_i^l\in X^{aug}_{tr}}-\log \frac{\exp(\mathrm{sim}(x_i^l,x_i^{3-l})/\tau)}{\sum\limits_{x_j^k\in X^-\cup \{x_i^{3-l}\}}\exp(\mathrm{sim}(x_i^l,x_j^k)/\tau)},
$
}
\end{equation}
where $\tau$ is the temperature hyperparameter. Denote $z_i^l=f_p(f(x_i^l))$ as the projected feature embedding after the projection head $f_p$ and feature extractor $f$, the similarity function $\mathrm{sim}(x_i^l,x_i^{3-l})={z_i^l}^T z_i^{3-l} / \Vert z_i^l \Vert_2 \Vert z_i^{3-l} \Vert_2 $ calculates the cosine similarity between projected feature embedding of samples. Equation \ref{eq:NCE} can be understood as aligning the embedding of $x_i^l$ and $x_i^{3-l}$ while constraining the feature embedding of all samples in $x_j^k\in X_{tr}^{aug}$ to satisfy a uniform distribution.

The main idea of BYOL \cite{grillBootstrapYourOwn2020} is to {maximize agreement between augmented sample pairs in $X_{^{tr}}^{aug}$ derived from a common origin, without relying on negative counterparts}. BYOL considers  {two distinct components, known as the online and target networks}. The two networks have the same feature extraction module $f$ and projection head module $f_p$. However, the online network has one more regression module $f_r$ than the target network. The objective of BYOL can be viewed as  {minimizing the MSE loss computed between the representations produced by the online and target networks}, which can be presented as:
\begin{equation}\label{eq:BYOL}
    {{\cal L}_{{\rm{BYOL}}}} = \sum\limits_{i = 1}^N {\sum\limits_{l = 1}^2 {\| {{f_r}( {\bar z_i^l} ) - \bar z_i^{3 - l}} \|^2} }
\end{equation}
where $z_i^l = {f_p}( {f( {x_i^l} )} )$ and $\bar z_i^l = z_i^l / \Vert z_i^l \Vert_2$. Note that a stop-gradient technique is applied to the target network in the gradient back-propagation stage.  {To simplify notation,} we  {define the} target network as $f_{target}$, and the part of the online network that is similar in structure to $f_{target}$ is  {referred to as} $f_{online}$. Then, \( f_{\text{target}} \)  {receives the online network's parameters through a moving average update, denoted by}:
\begin{equation}\label{eq:opt_BYOL}
    {f_{{{target}}}} \leftarrow \pi {f_{{{target}}}} + (1 - \pi ){f_{online}}
\end{equation}
where $\pi \in \left[0,1\right]$ represents a target decay rate. Equation \ref{eq:BYOL} can be considered as aligning the feature embedding of different views while Equation \ref{eq:opt_BYOL} constrains the update rule through stop-gradient and moving average of parameters. 

Barlow Twins~\cite{zbontarBarlowTwinsSelfSupervised2021} is an augmentation-based method that  {avoids relying on numerous negative samples or the use of a} stop-gradient technique or asymmetric networks. It first computes the cross-correlation matrix $C$ within a minibatch between $\left\{ {x_i^1} \right\}_{i = 1}^N$ and $\left\{ {x_i^2} \right\}_{i = 1}^N$ in the feature space. Then, we have:
\begin{equation}\label{eq:barlow_cov}
    C_{kj}=\frac{\sum_{i=1}^N z^1_{i,k} \cdot z^2_{i,j}}{\sqrt{\sum_{i=1}^N(z^1_{i,k})^2}\cdot \sqrt{\sum_{i=1}^N(z^2_{i,j})^2}}
\end{equation}
where $z_i^l = {f_p}( {f( {x_i^l} )} )$, $k,j\in\{1,\dots,D\}$, and $D$ is the dimension of $z_i^l$. Then, the objective of Barlow Twins is presented as follows:
\begin{equation}\label{eq:barlow_objective}
    \mathcal{L}_{\mathrm{BT}}=\sum_{k=1}^D (1-C_{kk})^2 +\lambda \sum_{k=1}^D\sum_{j=1,j\neq k}^D C_{kj}^2
\end{equation}
where $\lambda$ is a positive constant trading of the importance of the first and second terms of the objective. Minimizing the first term of Equation \ref{eq:barlow_objective} aligns the embedding of two views of $x_i$ while minimizing the second term of Equation \ref{eq:barlow_objective} constrains different dimensions of these feature embeddings to be decorrelated.

MAE~\cite{he2022masked}  {divides} the ancestor sample into $m$ patches, denoted as $x_i \in \mathbb{R}^{m\times s}$,  
where $s$ represents the patch size (e.g., $16\times 16$ for an image sample).  
A random binary mask $m \in \{0,1\}^m$  {is used on} sample $x_i$  {to obtain two separate views},  
where $x_i^1 = x_i[m]$ and $x_i^2 = x_i[1 - m]$.  
The MAE model  {consists of} an encoder $f$ and a decoder $g$.  
The encoder $f$ takes one view, $x_i^1$, as input and generates the feature representation $z_i^l$.  
The decoder $g$ then takes $z_i^l$ and the masked index $m$ as input and  {reconstructs} the other view, $x_i^2$.
 The objective of MAE can be described with the following mean squared error loss:
\begin{equation}
\label{eq:MAE}
\mathcal{L}_{\mathrm{MAE}} = \sum_{i=1}^N\sum_{l=1}^2 \Vert g(f(x_i^l)) - x_i^{3-l} \Vert_2^2.
\end{equation}
. It is noteworthy that if we treat the masking strategy in reconstruction-based methods as an augmentation technique, Equation \ref{eq:MAE} can be considered as implicitly aligning the feature representation of $x_i^l$ and $x_i^{3-l}$. Unlike augmentation-based methods, MAE enforces no explicit constraints, and this results in a dimensional collapse problem .

\section{Motivating example and analysis}
\label{sec:moti_anal}
In this section, we first describe the crowding problem with experimental observation. Then, we propose an empirical analysis to understand why SSL methods suffer from this problem and derive a way to solve this problem: minimizing the intra-class variance and maximizing the inter-class separation while performing SSL. Finally, we show through theoretical analysis that only by minimizing the SSL loss, intra-class variance, and maximizing inter-class separability at the same time can we better reduce the upper bound of cross-entropy loss.

\subsection{Motivating example} 
\label{sec:moti}
To provide a clear comparison between SSL methods and SL, we conducted a series of experiments. First, we trained feature extractors using different SSL methods: SimCLR, BYOL, Barlow Twins, SwAV, and MAE. Additionally, we trained a feature extractor using supervised cross-entropy loss. All feature extractors except MAE are implemented as the ResNet-50 , while MAE uses ViT-B/16 . They are pre-trained on the ImageNet dataset for 1000 epochs. Next, we evaluated the classification accuracies on the ImageNet validation set following the standard protocol by training a linear classifier on top of the feature extractor with the parameters of the feature extractor frozen. The accuracies are presented in Table \ref{tab:quantify}.
We observed that the accuracies for SSL are comparable but significantly lower than the accuracy of the supervised method. Specifically, MAE has a significantly lower linear-eval accuracy, which is also observed by a series of works. 
{To further analyze the learned encoders, we employed t-SNE \cite{van2008visualizing} to visualize the data distribution of the ImageNet test data. For each method, we selected the optimal perplexity value through a grid search and set the number of iterations to 1000 to ensure convergence. The detailed visualization parameters for all methods are listed in Table~\ref{tab:tsne_config} of \ref{app:config}. In addition, to verify the robustness of our observations, we further provide visualization results using 3D t-SNE, UMAP \cite{mcinnes2018umap}, and PCA \cite{jolliffe2002principal} in \ref{app:more_results}.}

In Figure \ref{fig:tsne}, each point represents an individual sample embedding, with different colors representing different categories. 
We observe that embeddings with the same label cluster tightly together, while those with different labels are clearly separated by distinct borders. This indicates that the supervised method achieves a small intra-class variance and a large inter-class distance. Conversely, the embeddings produced by SSL methods, shown in Figures \ref{fig:tsne_simclr} to \ref{fig:tsne_mae}, do not exhibit clear separation, and the intra-class variance is large. We refer to this phenomenon as \textit{the crowding problem}. Specifically, the cluster-based method SwAV (Figure \ref{fig:tsne_swav}) demonstrates smaller intra-class variance, but the embeddings from different classes still overlap. The reconstruction-based method MAE (Figure \ref{fig:tsne_mae}) shows very large intra-class variance, with embeddings from different classes failing to form distinct clusters. This corresponds to the poor linear evaluation performance seen in Table \ref{tab:quantify}. Other SSL methods, depicted in Figures \ref{fig:tsne_simclr}, \ref{fig:tsne_barlow}, and \ref{fig:tsne_byol}, exhibit some clustering characteristics, but with relatively large intra-class variance and overlapping class borders.

We provide a quantified analysis of the intra-class variance and inter-class distance for both SSL methods and the supervised method in Table \ref{tab:quantify}. In this table, the inter-class distance is calculated as the mean \(\ell_2\)-normalized Euclidean distance between embeddings of each class. The mean intra-class variance is calculated as the variance of the distances within embeddings of each class. Formally, for a test set \(D_{te}=\{(x_i,y_i)\}_{i=1}^N\) with \(N\) pairs of samples \(x\) and labels \(y\), the feature embedding of sample \(x\) is obtained using the feature extractor \(f\) as \(z=f(x)\). The feature embedding is then \(\ell_2\)-normalized as \(\bar{z} = z / \Vert z \Vert_2\). The mean embedding vector \(\mu_i\) of class \(i \in \{1,\dots,K\}\) is calculated with \(\mu_i = \frac{1}{N_i} \sum_{j=1}^{N_i} \bar{z}_{i,j}\), where \(N_i\) is the number of samples in the \(i\)-th class. The inter-class distance is calculated as:\( d_{\mathrm{inter}} = \frac{2}{K(K-1)}\sum_{i=1}^{K-1}\sum_{j=i+1}^{K}\Vert \mu_i - \mu_j \Vert_2 \).The intra-class variance of class \(i\) is calculated as:\( \mathrm{Var}_i = \frac{1}{N_i} \sum_{j=1}^{N_i} \Vert z_{i,j} - \mu_i \Vert_2^2 \).The mean intra-class variance is the average intra-class variance over all classes, calculated as: \( \mathrm{Var}_{\mathrm{intra}} = \frac{1}{K} \sum_{i=1}^K Var_i \)
\begin{table}[h]
    \centering
    \caption{The inter-class distance, the intra-class variance, and the linear evaluation accuracy of SSL methods.}
    \begin{tabular}{lccc}
    \toprule
    \bf   Method  &\bf Inter-class Dist. ($\uparrow$) &\bf Intra-class Var. ($\downarrow$) &\bf ACC (\%) \\
    \midrule
        SimCLR & 1.17 & 1.15 & 70.15 \\
        BYOL & 0.90 & 0.65 & 71.48 \\
        SwAV & 1.12 & 1.01 & 75.78 \\
        Barlow Twins & 1.06 & 1.11 & 73.97 \\
        MAE & 0.14 & 0.85 & 66.85 \\
        Supervised &\bf 1.32 &\bf 0.62 &\bf 79.24 \\
    \bottomrule
    \end{tabular}
    \label{tab:quantify}
\end{table}
As shown in Table \ref{tab:quantify}, the MAE has the smallest inter-class distance and shows no clustering characteristic in Figure \ref{fig:tsne_mae}, this corresponds to the poor performance in linear-eval ACC. In contrast, the supervised method exhibits the smallest intra-class variance and the largest inter-class distance and also has a good performance in linear-eval ACC. Other SSL methods either have smaller inter-class distances or larger intra-class variances, leading to sub-optimal linear-eval ACC performance. These quantified results underscore a significant connection between the crowding problem and downstream performance, aligning with the observed visualization patterns. Following this, we provide an empirical analysis of the SSL objectives introduced in Section \ref{sec:Preli}.

\subsection{Empirical analysis} \label{sec: prob_anal}
To elucidate the reasons behind the crowding problem observed in self-supervised methods, it is essential to gain an understanding of the underlying mechanisms employed by the self-supervised approaches described in Section \ref{sec:Preli}.

{M}inimizing Equation \ref{eq:NCE} aims to encourage the learned representations of positive pairs to exhibit similarity, while also constraining all samples in the feature space to be uniformly distributed on a unit hypersphere. In the case of BYOL, minimizing Equation \ref{eq:BYOL} focuses solely on promoting similarity between the learned representations of two samples with the same ancestor. Similarly, for Barlow-Twins, minimizing Equation \ref{eq:barlow_objective} aims to ensure that the learned representations of positive samples are similar to the anchor, while also encouraging the different dimensions of the learned representations to be independent of each other. Minimizing Equation \ref{eq:MAE} only encourages the alignment of the representation of two samples with the same ancestor with no additional constraints. This could lead to a collapse problem, where the inter-class distance is significantly lower than that of augmentation-based methods.

As we can see, these methods primarily focus on minimizing the distance between $x^1_i$ and $x^2_i$. None of these methods are \textbf{explicitly} designed to constrain the relationship between $x^l_i$ and $x^k_j$, {with indices $i$ and $j$ drawn from $\{1,2,\dots,N\}$ such that $i \neq j$},
and $l,k \in \left\{ {1,2} \right\}$. In contrast, SL leverages labeled data to explicitly enforce the aggregation of similar points and the separation of different types of points. Consequently, some pairs in $X^{aug}_{tr}$ are grouped together, while others are pushed apart.

In order to enhance the performance of self-supervised learning methods, it is necessary to adopt a criterion that enables the clustering of similar points and the separation of dissimilar points. This realization leads us to propose the following approach.

\subsection{Theoretical Analysis}
\label{subsec:theo}
To illustrate why we should minimize the intra-class variance and maximize the inter-class separation while performing SSL,  we first  {make} an assumption on the label consistency between pairs in $X^{aug}_{tr}$ with the same ancestor. In this section, we select Equation \ref{eq:NCE} as the SSL objective. 
\begin{assumption} \label{dadadadad}
    $\forall \left\{ {x_i^1,x_i^2} \right\} \in X_{tr}^{aug}$, assume the labels are deterministic (one-hot) and consistent: $p\left( {y\left| {x_i^1} \right.} \right) = p\left( {y\left| {x_i^2} \right.} \right)$.
\end{assumption}
Then, we aim to  {analyze} the generalization gap between unsupervised and supervised learning risks  {under the classification setting},  
which  {involves training a softmax classifier via expected cross-entropy minimization}:
{
\begin{align}
\mathcal{L}_{CE}^\mu(f) 
&= \mathbb{E}_{p(x,y)} 
\left[
    - \sum_{k=1}^{K} p_k \log q_k
\right] \\[6pt]
&= \mathbb{E}_{p(x,y)} 
\left[
    - \sum_{k=1}^{K} p_k 
    \log 
    \left(
        \frac{
            \exp(f(x)^{\top} \mu_k)
        }{
            \sum_{i=1}^{K} \exp(f(x)^{\top} \mu_i)
        }
    \right)
\right] \\[8pt]
&= \mathbb{E}_{p(x,y)} 
\left[
    - \log 
    \left(
        \frac{
            \exp(f(x)^{\top} \mu_y)
        }{
            \sum_{i=1}^{K} \exp(f(x)^{\top} \mu_i)
        }
    \right)
\right].
\end{align}
Here, $p(x,y)$ represents the joint distribution of the input sample and its corresponding label, and $K$ denotes the number of classes. The vector $p$ is a one-hot encoding, with $p_y = 1$ and $p_i = 0$ for $i \neq y$. The term $q_k$ denotes the model-predicted probability of class $k$, computed by the softmax function as $q_k = \exp(f(x)^{\top} \mu_k)/\sum_{i=1}^{K} \exp(f(x)^{\top} \mu_i)$. The vector $f(x)$ denotes the feature embedding of sample $x$. The term $\mu_i = \mathbb{E}_{p(x|y=i)} [ f(x) ]$ represents the cluster center (mean embedding) of class $i$ and can also be interpreted as the weight $w_i$ of a linear classifier $g$. The vector $\mu_y$ specifically denotes the cluster center corresponding to the true class label $y$ of sample $x$. The conditional distribution of samples given the class label is denoted as $p(x|y)$. Finally, the use of the dot product $f(x)^{\top} \mu_i$ comes from the standard formulation of linear classifiers, where it serves as a similarity measure. A higher dot product indicates stronger alignment with that class and leads to a higher softmax probability.
}
Then, we have:
\begin{theorem} \label{fghjk}
    If Assumption \ref{dadadadad} holds, then, for any $f \in \mathcal{F}$, $\mathcal{L}_{CE}^\mu \left( f \right)$ can be bounded by $\mathcal{L}_{\rm NCE}^\mu \left( f \right)$ as:
\begin{equation}
        \begin{array}{l}
\mathcal{L}_{CE}^\mu \left( f \right)  \le {\mathcal{L}_{{\rm{NCE}}}}\left( f \right) - const + \sum\limits_{i = 1,j = 1,i \ne j}^K {\mu _i^{\rm{T}}{\mu _j}} \\
 \quad\quad\quad \quad \quad+ \sqrt {{\rm{Var}}\left( {f\left( x \right)\left| y \right.} \right)}  + O\left( {cons{t^{{{ - 1} \mathord{\left/
 {\vphantom {{ - 1} 2}} \right.
 \kern-\nulldelimiterspace} 2}}}} \right)
\end{array}
    \end{equation}
where const is a constant that is related to the number of negative samples, ${\mu _i} = {\mathbbm{E}_{p\left( {x\left| i \right.} \right)}}\left[ {f\left( x \right)} \right]$, and ${\rm{Var}}\left( {f\left( x \right)\left| y \right.} \right) = {\mathbbm{E}_{p\left( y \right)}}\left[ {{\mathbbm{E}_{p\left( {x\left| y \right.} \right)}}{{\left\| {f\left( x \right) - {\mathbbm{E}_{p\left( {x\left| y \right.} \right)}}f\left( x \right)} \right\|}^2}} \right]$.
\end{theorem}
{
The detailed proof can be found in Appendix A. From the above empirical and theoretical analyses, we conclude that to effectively compress the upper bound of the classification error, one must simultaneously minimize three crucial factors: (i) the original SSL loss, (ii) intra-class variance, and (iii) inter-class similarity. Intuitively, reducing intra-class variance implies that samples sharing similar semantics should cluster more tightly in the feature space. Similarly, increasing inter-class distance requires samples with dissimilar semantics to be more widely separated in the feature space. However, in self-supervised learning, label information is not available. Therefore, achieving these two objectives necessitates designing an effective unsupervised mechanism.
}

\begin{figure*}[t]
    \centering
    \includegraphics[width=0.9\textwidth]{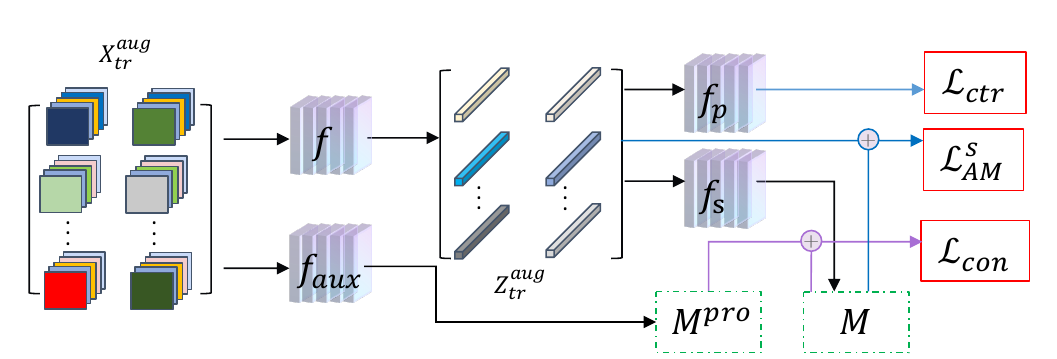}
    \caption{The pipeline of the proposed DSA. First, DSA obtains the feature representations by the $f_p$ and the $M^{pro}$ by the $f_p$. Then, DSA obtains the $M$ by the $f_s$. Finally, DSA simultaneously minimizes the ${\mathcal{L}_{ctr}}$, $ \mathcal{L}_{AM}^s$, and $ {\mathcal{L}_{con}}$ to learn the model. }
    \label{fig:good neighbor}
\end{figure*}

\section{Methodology}
\label{sec:method}
In this section, we describe our proposed method called {the} Dynamic Semantic Adjuster (DSA). DSA can be easily integrated into existing  SSL methods. The core concept of DSA is to attract samples with similar semantics while repelling others in an instance-based manner. Specifically, DSA first randomly selects a point from the training dataset as an anchor point, and then, based on the arranging module and the scoring module,  {aggregates points in the training dataset that are similar to the anchor point, and separates those that are dissimilar}.
Then DSA traverses all the points in the training dataset and treats them as anchor points, thus achieving the purpose of aggregating similar points and separating dissimilar points in the entire training dataset. Note that similar points are more likely to be of the same class, and dissimilar points are more likely to be of different classes. Therefore, DSA can induce a clustering structure and make points of the same class cluster together, thereby reducing intra-class variance,  {while pushing points from different classes apart to enhance inter-class separation}.

The overall pipeline of DSA is illustrated in Figure \ref{fig:good neighbor}, it consists of two main components. The first component is the arranging module, which groups samples with similar semantics together and pushes away those with dissimilar semantics. The second component is the scoring module, which is to further revise the semantic similarity between samples in the arranging module through the relative position structure between samples in the original space.

\subsection{Arranging module}
The arranging module (AM) is designed to learn a regulator $M$, which guides the arrangement of samples in the feature space. The goal of the regulator $M$ is to bring samples with high similarity closer to each other while separating samples with low similarity. Formally, the regulator $M$ can be interpreted as a similarity matrix with a size of $2N \times 2N$.

Specifically, we first project each sample in $X_{^{tr}}^{aug} = \left\{ {x_1^1,x_1^2,...,x_N^1,x_N^2} \right\}$ to the feature space via $f$ and obtain their feature representations, denoted as $Z_{^{tr}}^{aug} = \left\{ {z_1^1,z_1^2,...,z_N^1,z_N^2} \right\}$, where $z_i^l = f\left( {x_i^l} \right)$, $i \in \left\{ {1,...,N} \right\}$, and $l \in \left\{ {1,2} \right\}$. Regarding the $i$-th sample $z_i$ in $Z_{^{tr}}^{aug}$ as the anchor, then, for the $j$-th sample $z_j$ in $Z_{^{tr}}^{aug}$, we concatenate them and obtain a new representation $z_{i,j}$, e.g., ${z_{i,j}} = cat\left( {{z_i},{z_j}} \right)$. Subsequently, we feed ${z_{i,j}}$ into a similarity network $f_s$, which consists of a two-layer MLP with ReLU activation function. The selection of $f_s$ is evaluated and discussed through an ablation study in Section \ref{subsec:ablation}. The similarity network computes the similarity value between $z_i$ and $z_j$ and assigns it to the corresponding entry $M_{i,j}$ in the regulator matrix $M$, i.e., $M_{i,j} = f_s(z_{i,j})$. Then, based on the anchor $z_i$, we can obtain:
\begin{equation}
{M_i} = \left[ {{{\bar f}_s}\left( {{z_{i,1}}} \right),...,{{\bar f}_s}\left( {{z_{i,2N}}} \right)} \right]
\end{equation}
where ${{\bar f}_s}\left( {{z_{i,j}}} \right) = {{{f_s}\left( {{z_{i,j}}} \right)} \mathord{\left/
 {\vphantom {{{f_s}\left( {{z_{i,j}}} \right)} {\sum\nolimits_{k = 1}^{k = 2N} {{f_s}\left( {{z_{i,k}}} \right)} }}} \right.
 \kern-\nulldelimiterspace} {\sum\nolimits_{k = 1}^{k = 2N} {{f_s}\left( {{z_{i,k}}} \right)} }}$ and $j \in \left\{ {1,...,2N} \right\}$. We in turn treat the samples in $Z_{^{tr}}^{aug}$ as anchors and obtain $M = {\left[ {{M_1},....,{M_{2N}}} \right]^{\rm{T}}}$.

Once the regulator matrix $M$ is obtained, we propose to minimize the following loss function:
\begin{equation}\label{eq:L_lay}
\begin{array}{l}
{\mathcal{L}_{AM}}=
  \log ( {1 + \sum\limits_{i = 1}^{2N} {\sum\limits_{j = 1}^{2N} {\exp ( {\frac{{( {2{M_{i,j}} - \alpha } )\| {{z_i} - {z_j}} \|_2^2}}{\tau }} )} } } )
\end{array}
\end{equation}
where $\tau,\alpha > 0$ are two temperature hyperparameters. Note that $\mathcal{L}_{AM}$ captures not only the relationship between different augmentations of the same ancestor but also the relationship between different augmentations of different ancestors.

{Given $\alpha$ and $M_{i,j}$, when $2M_{i,j} - \alpha > 0$, we consider that the samples $z_i$ and $z_j$ in $Z_{^{tr}}^{aug}$ are semantically similar and should be brought closer together. From Equation (\ref{eq:L_lay}), we can see that when $2M_{i,j} - \alpha > 0$, minimizing $\| z_i - z_j \|_2^2$ will minimize $\mathcal{L}_{AM}$. Furthermore, to minimize $\mathcal{L}_{AM}$, the greater the value of $2M_{i,j} - \alpha$, the more we should pull $z_i$ and $z_j$ together. This means that minimizing $\mathcal{L}_{AM}$ brings pairs of samples closer together to varying degrees, depending on the value of $M_{i,j}$.}

{On the other hand, when $2M_{i,j} - \alpha < 0$, the scalar becomes negative, and minimizing $\mathcal{L}_{AM}$ drives the distance $\|z_i - z_j\|_2^2$ to increase, which leads to a repulsive force between dissimilar samples. Thus, while $\mathcal{L}_{AM}$ encourages attraction between similar samples, it simultaneously imposes a repulsive force between dissimilar ones, effectively controlling both attraction and repulsion based on the value of $M_{i,j}$.}

\subsection{Scoring module}
As illustrated in the above subsection, $M_{i,j}$ can control the dynamics  {of a sample, i.e., whether it moves closer to or farther from the anchor}.  
However, $M$  {is computed by feeding} $z_{i,j}$ into $f_s$. However, $M$ is obtained by inputting $z_{i,j}$ into $f_s$. Without proper constraints, $f_s$ may output undesirable values. For example,  {during early optimization,} $f_s$ {often performs suboptimally}, which leads to the learned regulator matrix being inaccurate, causing the sample points in the training dataset to be incorrectly aggregated and separated. To alleviate this problem, we propose a scoring module (SM) that can revise the output of $f_s$ based on the similarity prior of the samples in the original input space.

Specifically, we first input samples in $X^{aug}_{tr}$ into an auxiliary feature extractor $f_{aux}$ to obtain prior feature representations, denoted as $\bar Z_{tr}^{aug} = \{ {\bar z_1^1,\bar z_1^2,...,\bar z_{2N}^1,\bar z_{2N}^2} \}$, where $\bar z_i^l = {f_{aux}}( {x_i^l} )$, $i \in \{ {1,...,2N} \}$, and $l \in \{ {1,2} \}$. 
{In our implementation, $f_{aux}$ is a pre-trained and fixed feature extractor and does not involve any additional training during our method. The default choice of $f_{aux}$ is the image encoder of the pre-trained CLIP model, but we have also conducted ablation studies replacing $f_{aux}$ with other pre-trained models, including BEiT, SIFT, and a pre-trained ResNet-18. We observed that the choice of $f_{aux}$ has minimal impact on the final performance, as shown in Table \ref{tab:f_aux_fs}.}
Then, considering the $i$-th element ${{{\bar z}_i}}$ in $\bar Z_{tr}^{aug}$ as the anchor, we define:
\begin{equation}
M_{i,j}^{pro} = \exp \left( {{{\left\| {{{\bar z}_i} - {{\bar z}_j}} \right\|_2^2} \mathord{\left/
 {\vphantom {{\left\| {{{\bar z}_i} - {{\bar z}_j}} \right\|_2^2} \tau }} \right.
 \kern-\nulldelimiterspace} \tau }} \right)
\end{equation}
where ${{\bar z}_j}$ is the $j$-th sample in $\bar Z_{tr}^{aug}$. Then, we can obtain:
\begin{equation}
M_i^{pro} = \left[ {\bar M_{i,1}^{pro},...,\bar M_{i,2N}^{pro}} \right]
\end{equation}
where $\bar M_{i,k}^{pro} = {{M_{i,k}^{pro}} \mathord{\left/
 {\vphantom {{M_{i,k}^{pro}} {\sum\nolimits_{j = 1}^{j = 2N} {M_{i,j}^{pro}} }}} \right.
 \kern-\nulldelimiterspace} {\sum\nolimits_{j = 1}^{j = 2N} {M_{i,j}^{pro}} }}$ and $k \in \left\{ {1,...,2N} \right\}$. We in turn treat the samples in ${{{\bar z}_i}}$ as anchors and obtain ${M^{pro}} = {\left[ {M_1^{pro},...,M_{2N}^{pro}} \right]^{\rm{T}}}$. To this end, we give the prior constraint ${M^{pro}}$ of $M$ based on the Euclidean distance of different pairs of samples in the original input space. We constrain $M$ as follows:
\begin{equation}\label{dswd}
{\mathcal{L}_{con}} = \left\| {{M^{pro}} - M} \right\|_2^2
\end{equation}
By simultaneously minimizing Equation~(\ref{dswd}) and Equation~(\ref{eq:L_lay}),  
the points in the original input space that are closer to the anchor  {are pulled nearer} in the learned feature space,  while the points that are farther from the anchor  {are pushed further away in the feature space}.

{The key difference between $f_{aux}$ and the similarity network $f_s$ lies in their roles and training strategies. $f_{aux}$ provides a prior representation purely based on pre-trained visual features and remains fixed, serving as a stable reference to generate the proxy matrix $M^{pro}$. In contrast, $f_s$ is a learnable neural network that is trained jointly with the model to predict semantic similarities in the learned representation space. Therefore, $f_s$ is subject to instability during early training stages, which $f_{aux}$ helps mitigate.}

To assess the quality of each anchor, we propose a scoring mechanism based on the connectivity between samples. Given the $i$-th element $z_i$ in $Z_{tr}^{aug}$, we find its corresponding element $\bar{z}_i$ in $\bar Z_{tr}^{aug}$. {We denote the set of $\eta$ nearest neighbors of $\bar{z}_i$ as $N_\eta(\bar{z}_i)$. For each neighbor $\bar{z}_i^a$ in $N_\eta(\bar{z}_i)$, we further identify its own $\eta$ nearest neighbors $N_\eta(\bar{z}_i^a)$. For every point $\bar{z}_i^b$ in $N_\eta(\bar{z}_i^a)$, we check whether $\bar{z}_i$ itself is among the neighbors of $\bar{z}_i^b$. If so, it contributes a value of 1 to the score; otherwise, it contributes 0.} The connection score $s_{i}$ for anchor $z_i$ is formally defined as:
\begin{equation}\label{eq:s_ij}
    s_i\left( z_i \right) 
    = \sum_{\bar{z}_i^b \sim N_\eta(\bar{z}_i^a),\, \bar{z}_i^a \sim N_\eta(\bar{z}_i)} 
    \mathbbm{1}_{\left\{ \bar{z}_i \in N_\eta(\bar{z}_i^b) \right\}}.
\end{equation}
Then, we normalize the score to obtain the final connection score $sc(z_i)$:
\begin{equation}\label{eqasdasd}
    sc\left( z_i \right) = \frac{ s_i\left( z_i \right) }{ \eta }.
\end{equation}
The scoring mechanism in Equation~(\ref{eq:s_ij}) takes into account both direct connections between the anchor $\bar{z}_i$ and $N_\eta(\bar{z}_i)$, as well as latent connections induced by other samples along the connected path . {From Equation~(\ref{eqasdasd}), it follows that when $z_i$ is surrounded by samples sharing similar semantics, $s_i(z_i)$ will be close to $\eta$, making $sc(z_i)$ close to 1. Conversely, when $z_i$ is an outlier, $s_i(z_i)$ will be close to 0, resulting in a small value of $sc(z_i)$. This scoring mechanism is entirely unsupervised and relies solely on pairwise similarities among samples, without any use of ground-truth labels. This works because samples with similar semantics are naturally close in the feature space, so mutual neighborhood structures often reflect potential same-class clustering even without labels.
This score is then used as a weight in Equation~(\ref{eq:L_lay}) to reduce the influence of unreliable anchors, leading to the following formulation:}
\begin{equation}\label{eq:L_adadlay}
\mathcal{L}_{AM}^s = 
\log \left( 
    1 + \sum_{i=1}^{2N} 
    \left(
        sc(z_i) \cdot \sum_{j=1}^{2N}
        \exp 
        \left(
            \frac{
                \left(2M_{i,j} - \alpha\right)
                \| z_i - z_j \|_2^2
            }{\tau}
        \right)
    \right)
\right).
\end{equation}

\subsection{Overall objective}
 {Accordingly, we define the learning goal of the introduced} DSA {as follows:}
\begin{equation}\label{edfy}
\mathcal{L}_{DSA}^s = {\mathcal{L}_{ssl}} + \nu \mathcal{L}_{AM}^s + \upsilon {\mathcal{L}_{con}}
\end{equation}
where ${\mathcal{L}_{ssl}}$ is the loss in Equation \ref{eq:SSL_unify}, e.g., we can set ${\mathcal{L}_{ssl}}$ equals to  ${\mathcal{L}_{\rm NCE}}$, ${\mathcal{L}_{\rm BYOL}}$, ${\mathcal{L}_{\rm BT}}$, or $\mathcal{L}_{\rm MAE}$, and $\nu ,\upsilon  > 0$ represent the temperature hyperparameters. 

{
In summary, the proposed DSA explicitly addresses the theoretical motivations described in Section~\ref{subsec:theo}. Specifically, the second term of Equation (\ref{edfy}), $\mathcal{L}_{AM}^s$, encourages aggregation of semantically similar samples and separation of dissimilar ones, effectively reducing intra-class variance and enhancing inter-class distances. Meanwhile, the third term, $\mathcal{L}_{con}$, ensures the similarity estimates remain reliable and robust, further improving the accuracy of semantic aggregation and separation. Together, these two terms of the DSA objective directly target and tighten the theoretical upper bound of the classification error established in Theorem \ref{fghjk}, achieving the desired unsupervised minimization of intra-class variance and maximization of inter-class separation.
}

\subsection{{Complexity Analysis of DSA}}
{We analyze the time complexity of our proposed method based on the unified objective function given in Equation \ref{edfy}, which consists of three components: the base SSL loss $\mathcal{L}_{ssl}$, $\mathcal{L}_{AM}^s$, and $\mathcal{L}_{con}$. Let $N$ denote the number of original samples in a mini-batch, each augmented twice, yielding $2N$ samples in total. Let $D$ denote the feature dimension of each embedding, and $\eta$ the number of neighbors used in the scoring module.}

{We first consider the term $\mathcal{L}_{ssl}$. Using SimCLR as a representative example, the feature extraction step has complexity $O(ND)$, and pairwise cosine similarity is computed between all $2N$ embeddings. Each pairwise similarity takes $O(D)$ time, resulting in a total of $O(N^2D)$ operations. The InfoNCE loss then applies a softmax over each positive pair with respect to all $2N-1$ negatives, costing an additional $O(N^2)$. Thus, the overall complexity of $\mathcal{L}_{ssl}$ is $O(ND)+O(N^2D)+O(N^2)=O(N^2D)$.}

{Next, for the semantic arranging loss $\mathcal{L}_{AM}^s$, the main cost lies in computing the regulator matrix $M$, which evaluates a similarity function $f_s$ over all $(2N)^2$ pairs of embeddings. $f_s$ is a two-layer MLP, each forward pass costs $O(D)$, resulting in $O(N^2D)$ total cost. The loss $\mathcal{L}_{AM}^s$ itself involves computing weighted distances between all sample pairs, again requiring $O(N^2D)$ operations. Additionally, the computation of the connection score for each anchor involves identifying its $\eta$ nearest neighbors among $2N$ auxiliary features. Using an efficient approximate nearest neighbor search algorithm, the total neighbor search cost is $O(ND \log N)$, and the connectivity-based scoring adds another $O(N\eta)$. Therefore, the complexity of this term is also $O(N^2D)+O(N^2D)+O(ND \log N)+O(N\eta)=O(N^2D)$.}

{Finally, the constraint loss $\mathcal{L}_{con}$ computes the squared difference between the predicted similarity matrix $M$ and the prior matrix $M^{pro}$, each of size $2N \times 2N$. The prior matrix is computed via pairwise Euclidean distances between auxiliary features $\bar{z}_i \in \mathbb{R}^D$, costing $O(N^2D)$, while the loss term itself requires a simple matrix-wise difference and squared Frobenius norm, costing $O(N^2)$. Hence, the complexity for $\mathcal{L}_{con}$ is $O(N^2D)+O(N^2)=O(N^2D)$.}

{Combining all terms, we find that the overall time complexity of our method is $\mathcal{O}(N^2 D)$. This is consistent with the complexity of standard contrastive methods such as SimCLR. Furthermore, since the auxiliary extractor $f_{\text{aux}}$ in $\mathcal{L}_{\text{con}}$ only performs a forward pass without backpropagation, and all similarity computations can be efficiently parallelized on GPUs, the practical overhead remains modest despite the quadratic scaling.}

\section{Experiment}
\label{sec:exp}
  
{This section presents an extensive empirical evaluation of our proposed DSA.}  
First, we  {introduce} the benchmark datasets for evaluation.  
Then, we detail the implementation of the DSA pre-training process, and  {evaluate it} on standard benchmarks for self-supervised image and video representation learning.  
To  {further verify its generalization ability, we evaluate DSA on} various downstream tasks,  
including semi-supervised classification, object detection, few-shot learning, and semantic segmentation.  
Finally, we present a comprehensive ablation study on hyper-parameters, module design, and complexity, along with case studies for visualization.

\subsection{Benchmark dataset}
Our experiments involved  {a diverse set of} datasets across different tasks.  
The pre-training dataset for image data is ImageNet~\cite{dengImageNetLargescaleHierarchical2009}, which contains  {1.3M labeled images spanning 1000 categories}.  
The pre-training dataset for video data is the Kinetics-400 dataset~\cite{kay2017kinetics}, which contains 400 human action classes, with at least 400 video clips for each action.  
  
{Each video clip, sourced from YouTube, lasts approximately 10 seconds.}  
The downstream object detection and instance segmentation tasks for image data are evaluated on the MS-COCO~\cite{lin2014microsoft} and Pascal-VOC~\cite{everingham2010pascal} {datasets}.  
MS-COCO is a large-scale benchmark for object detection, segmentation, keypoint detection, and captioning, consisting of 328K images.  
The Pascal-VOC dataset includes 1464 training images and 1449 validation images, annotated with object segmentation, bounding boxes, and class labels.  
 {It contains 20 predefined object classes.}
The few-shot classification task is evaluated on  {three standard benchmarks:} FC100~\cite{fc100}, Caltech-UCSD Birds (CUB-200)~\cite{cub200}, and Plant Disease~\cite{plantdisease}.  
FC100 is a split dataset based on CIFAR-100~\cite{krizhevskyLearningMultipleLayers2009}, designed for few-shot learning via disjoint train/val/test splits to reduce information overlap.  
CUB-200 comprises 6033 bird images categorized into 200 species.  
Plant Disease is a public dataset consisting of 54,306 images of both diseased and healthy plant leaves;  {each sample is labeled with one of 14 crop types and 26 disease classes}.

\subsection{Pre-training Details}
 
{During optimization, we adopt a schedule that gradually increases the learning rate over the first 500 steps, followed by a decay of 0.2 occurring 50 and 25 epochs prior to training completion.} 
 
{The proposed approach is incorporated into standard SSL frameworks using ResNet-50 or ViT-B as the feature encoder.} The SSL pre-training for video data follows the standard protocol as introduced in \cite{feichtenhoferLargeScaleStudyUnsupervised2021}  {using the Kinetics-400 dataset as the source corpus} \cite{kay2017kinetics}.   {Our method is applied to four representative video self-supervised learning frameworks: v-SwAV, v-SimCLR, v-BYOL, and v-MoCo.} We evaluate the video SSL algorithm on two backbones, R3D-18  and R3D-50 . The default training epoch is 200 while the default batch size is 256.

For the default hyperparameters of DSA, we set the number of neighbors $\eta=20$, the hyperparameter $\alpha=1.0$, $\nu=0.1$, and $\upsilon=100$. We conduct a detailed ablation study of these hyper-parameters in Section \ref{subsec:ablation}.

\begin{table}[h!]
	\centering
	\caption{The Top-1 and Top-5 classification accuracies of a linear classifier on ImageNet with ResNet-50 as the feature extractor.}
	\label{tab:imagenet} 
	\begin{tabular}{lccc}
		\toprule
	    Method & \textbf{Backbone} & \textbf{Top-1} & \textbf{Top-5} \\
	    \midrule
	    SimCLR \cite{chenSimpleFrameworkContrastive2020} & ResNet-50 & 70.15 $\pm$ 0.16 & 89.75 $\pm$ 0.14  \\
		MoCo \cite{heMomentumContrastUnsupervised2020} & ResNet-50 & 72.80 $\pm$ 0.12 & 91.64 $\pm$ 0.11 \\ 
            BYOL \cite{grillBootstrapYourOwn2020} & ResNet-50 & 71.48 $\pm$ 0.15 & 92.32 $\pm$ 0.14 \\
            SimSiam \cite{chenExploringSimpleSiamese2021} & ResNet-50 & 73.01 $\pm$ 0.21 & 92.61 $\pm$ 0.27 \\ 
            Barlow Twins \cite{zbontarBarlowTwinsSelfSupervised2021} & ResNet-50 & 73.97 $\pm$ 0.23 & 92.91 $\pm$ 0.19 \\
		SwAV \cite{caronUnsupervisedLearningVisual2020} & ResNet-50 & 75.78 $\pm$ 0.16 & 92.86 $\pm$ 0.15 \\
            DINO \cite{caron2021emerging} & ResNet-50 & 75.43 $\pm$ 0.18 & 93.32 $\pm$ 0.19 \\
            W-MSE \cite{ermolov2021whitening} & ResNet-50 & 76.01 $\pm$ 0.27 & 93.12 $\pm$ 0.21 \\
            RELIC v2 \cite{RELIC-v2} & ResNet-50 & 75.88 $\pm$ 0.15 & 93.52 $\pm$ 0.13 \\
		LMCL \cite{LMLC} & ResNet-50 & 75.89 $\pm$ 0.19 & 92.89 $\pm$ 0.28 \\
        ReSSL \cite{ressl} & ResNet-50 & 75.77 $\pm$ 0.21 & 92.91 $\pm$ 0.27 \\
        SSL-HSIC \cite{ssl-hsic} & ResNet-50 & 74.99 $\pm$ 0.19 & 93.01 $\pm$ 0.20 \\
        CorInfoMax \cite{CorInfoMax} & ResNet-50 & 75.54 $\pm$ 0.20 & 92.23 $\pm$ 0.25 \\
        MEC \cite{MEC} & ResNet-50 & 75.38 $\pm$ 0.17 & 92.84 $\pm$ 0.20 \\
        VICRegL \cite{Vicregl} & ResNet-50 & 75.96 $\pm$ 0.19 & 92.97 $\pm$ 0.26 \\
    \midrule
    \rowcolor{orange!10}SimCLR + DSA & ResNet-50 &72.09 $\pm$ 0.16&91.39 $\pm$ 0.13\\
    \rowcolor{orange!10}MoCo + DSA & ResNet-50 &74.79 $\pm$ 0.69&93.65 $\pm$ 0.74\\
    \rowcolor{orange!10}SimSiam + DSA & ResNet-50 &74.09 $\pm$ 0.89&94.24 $\pm$ 0.82\\
    \rowcolor{orange!10}Barlow Twins + DSA & ResNet-50 &76.03 $\pm$ 0.49&93.95 $\pm$ 0.15\\
    \rowcolor{orange!10}SwAV + DSA & ResNet-50 &77.84 $\pm$ 0.32&94.05 $\pm$ 0.89\\
    \rowcolor{orange!10}DINO + DSA & ResNet-50 &76.50 $\pm$ 0.78&\bf 94.46 $\pm$ 0.62\\
    \rowcolor{orange!10}BYOL + DSA & ResNet-50 &74.95 $\pm$ 0.57&94.93 $\pm$ 0.16\\
    \rowcolor{orange!10}ReSSL + DSA & ResNet-50 &77.45 $\pm$ 0.33&94.23 $\pm$ 0.23\\
    \rowcolor{orange!10}VICRegL + DSA & ResNet-50 &\bf 78.15 $\pm$ 0.63&94.04 $\pm$ 0.20\\
    \midrule
    MAE \cite{he2022masked} & ViT-B & 66.85 $\pm$ 0.23 & 85.24 $\pm$ 0.43\\
    U-MAE \cite{zhang2022mask} & ViT-B & 70.46 $\pm$ 0.18 & 91.25 $\pm$ 0.28 \\
    MoCo-v3 \cite{chen2021empirical} & ViT-B & 76.47 $\pm$ 0.14 & 93.76 $\pm$ 0.46 \\
    DINO \cite{caron2021emerging} & ViT-B &  78.17 $\pm$ 0.57 & 96.14 $\pm$ 0.17 \\
    \midrule
    \rowcolor{orange!10}MAE + DSA  & ViT-B & 72.76 $\pm$ 0.41&89.36 $\pm$ 0.31\\
    \rowcolor{orange!10}MoCo-v3 + DSA  & ViT-B & 78.24 $\pm$ 0.72&95.75 $\pm$ 0.23\\
    \rowcolor{orange!10}DINO + DSA  & ViT-B &\bf 79.20 $\pm$ 0.42&\bf 97.16 $\pm$ 0.66\\
    \bottomrule
	\end{tabular}
\end{table}

\begin{table}[htb]
    \centering
    \caption{ {Results on Kinetics-400 under linear evaluation, with} $\varrho$, $L$, and $\delta$  {representing the count of positive pairs, temporal span, and frame stride}.}
    \begin{tabular}{lcccc}
    \toprule
    \bf Method & $\varrho$ & $L$ & $\delta$ &\bf Top-1 \\
    \midrule
    supervised & - & 8 & 8 & 74.7 \\
    \midrule
    v-MoCo \cite{feichtenhoferLargeScaleStudyUnsupervised2021} & 2 & 8 & 8 & 65.8 \\
    v-BYOL \cite{feichtenhoferLargeScaleStudyUnsupervised2021} & 2 & 8 & 8 & 65.8 \\
    v-SwAV \cite{feichtenhoferLargeScaleStudyUnsupervised2021} & 2 & 8 & 8 & 61.6 \\
    v-SimCLR \cite{feichtenhoferLargeScaleStudyUnsupervised2021} & 2 & 8 & 8 & 60.5 \\
    v-BYOL \cite{feichtenhoferLargeScaleStudyUnsupervised2021} & 2 & 16 & 4 & 67.6 \\
    v-MoCo \cite{feichtenhoferLargeScaleStudyUnsupervised2021} & 4 & 8 & 8 & 67.8 \\
    \midrule
    \rowcolor{orange!10} v-BYOL + DSA & 2 & 8 & 8 & 66.7   \\
    \rowcolor{orange!10} v-SimCLR + DSA & 2 & 8 & 8 & 62.5   \\
    \rowcolor{orange!10} v-MoCo + DSA & 2 & 8 & 8 & 67.3  \\
    \rowcolor{orange!10}v-MoCo  + DSA & 4 & 8 & 8 & 69.0 \\
    \rowcolor{orange!10} v-SwAV + DSA & 2 & 8 & 8 & 62.8   \\
    \rowcolor{orange!10}v-BYOL  + DSA & 2 & 16 & 4 &\bf 69.2 \\
    \bottomrule
    \end{tabular}
    \label{tab:k400}
\end{table}

\begin{table}[htb]
	\centering
	\caption{The semi-supervised learning accuracies (\%) on ImageNet using 1\% and 10\% training examples.  dataset with the ResNet-50 pre-trained on the Imagenet dataset.}
\label{tab:semi}
		\begin{tabular}{lcccc}
		\toprule
		\multirow{2.5}{*}{\bf Method} &\multicolumn{2}{c}{\bf 1\%} & \multicolumn{2}{c}{\bf 10\%} \\
	    \cmidrule(lr){2-3} \cmidrule(lr){4-5}
	    &\bf Top-1 &\bf Top-5 &\bf Top-1 &\bf Top-5 \\
        \midrule
        Supervised & 25.1 $\pm$ 1.3 & 48.6 $\pm$ 0.7 & 55.9 $\pm$ 0.7 & 81.2 $\pm$ 0.8 \\
        \midrule
        SimCLR \cite{chenSimpleFrameworkContrastive2020}   & 48.3 $\pm$ 0.2 & 75.5 $\pm$ 0.1 & 65.6 $\pm$ 0.1 & 87.8 $\pm$ 0.2\\
	MoCo \cite{heMomentumContrastUnsupervised2020}  &52.3 $\pm$ 0.1 & 77.9 $\pm$ 0.2 &68.4 $\pm$ 0.1 &88.0 $\pm$ 0.2\\
	BYOL \cite{grillBootstrapYourOwn2020}   & 56.3 $\pm$ 0.2 & 79.6 $\pm$ 0.2 & 69.7 $\pm$ 0.2& 89.3 $\pm$ 0.1\\
        SimSiam \cite{chenExploringSimpleSiamese2021}   & 54.9 $\pm$ 0.2 & 79.5 $\pm$ 0.2 & 68.0 $\pm$ 0.1 &89.0 $\pm$ 0.3 \\
        Barlow Twins \cite{zbontarBarlowTwinsSelfSupervised2021}   & 55.0 $\pm$ 0.1& 79.2 $\pm$ 0.1 & 67.7 $\pm$ 0.2 & 89.3 $\pm$ 0.2\\
	     RELIC v2 \cite{RELIC-v2}  & 55.2 $\pm$ 0.2 & 80.0 $\pm$ 0.1& 68.0 $\pm$ 0.2 & 88.9 $\pm$ 0.2\\
	     LMCL \cite{LMLC}   & 54.8 $\pm$ 0.2 & 79.4 $\pm$ 0.2 & 70.3 $\pm$ 0.1  & 89.9 $\pm$ 0.2\\
	     ReSSL \cite{ressl}   & 55.0 $\pm$ 0.1 & 79.6 $\pm$ 0.3 & 69.9 $\pm$ 0.1 & 89.7 $\pm$ 0.1\\
	     SSL-HSIC \cite{ssl-hsic}   & 55.1 $\pm$ 0.3 & 79.6 $\pm$ 0.2 & 70.4 $\pm$ 0.1 & 90.0 $\pm$ 0.1 \\
      CorInfoMax \cite{CorInfoMax}  & 55.0 $\pm$ 0.2 & 79.6 $\pm$ 0.3 & 70.3 $\pm$ 0.2 & 89.3 $\pm$ 0.2\\
      MEC \cite{MEC}  & 54.8 $\pm$ 0.1 & 79.4 $\pm$ 0.2&  70.0 $\pm$ 0.1 & 89.1 $\pm$ 0.1\\
      VICRegL \cite{Vicregl}  & 54.9 $\pm$ 0.1 & 79.6 $\pm$ 0.2 & 67.2 $\pm$ 0.1  & 89.4 $\pm$ 0.2\\
	   \midrule
        \rowcolor{orange!10} SimCLR + DSA   &49.9 $\pm$ 0.2&77.3 $\pm$ 0.1&66.6 $\pm$ 0.1&89.3 $\pm$ 0.2\\
\rowcolor{orange!10} MoCo + DSA   &54.1 $\pm$ 0.3&80.0 $\pm$ 0.3&69.9 $\pm$ 0.3&90.0 $\pm$ 0.2\\
\rowcolor{orange!10} BYOL + DSA   &\bf 57.5 $\pm$ 0.5&\bf 81.1 $\pm$ 0.1&\bf 71.7 $\pm$ 0.4&\bf 91.0 $\pm$ 0.3\\
\rowcolor{orange!10} Barlow Twins + DSA   &57.0 $\pm$ 0.1&80.9 $\pm$ 0.1&69.4 $\pm$ 0.5&90.7 $\pm$ 0.1\\
		\bottomrule
	\end{tabular}
\end{table}

\begin{table*}[htb]
    \centering
    \caption{ {Action recognition finetuning performance on UCF101 and HMDB51, reported as the average over three splits. Models are pre-trained in a self-supervised manner using the Kinetics-400 dataset.} $\varrho$ is the number of positive samples.}
    \label{tab:ucf}
    \resizebox{.9\linewidth}{!}{
    \begin{tabular}{lccccccc}
    \toprule
      \bf Method  & \bf Resolution & \bf Frames  & \bf Architecture & \bf Param. & \bf Epochs & \bf UCF101 & \bf HMDB51 \\
    \midrule
        VTHCL \cite{yangVideoRepresentationLearning2020} & 224$\times$224 & 8 & R3D-18 & 13.5M & 200   & 80.6 & 48.6 \\
        TCLR \cite{dave_tclr_2022} & 112$\times$112 & 16 & R3D-18 & 13.5M & 100   & 85.4 & 55.4 \\
        VideoMoCo \cite{videomoco} & 112$\times$112 & 16 & R3D-18 & 13.5M & 200   & 74.1 & 43.6 \\
        SLIC \cite{khorasgani2022slic} & 128$\times$128 & 32 & R3D-18 & 13.5M & 150  & 83.2 & 52.2 \\
        MACLR \cite{xiao2022maclr} & 112$\times$112 & 32 & R3D-18 & 13.5M & 600  & 91.3 & 62.1 \\
        v-BYOL$_{\varrho=4}$  & 112$\times$112 & 16 & R3D-18 & 13.5M & 200  & 88.3 & 69.3 \\
    \midrule
    \rowcolor{orange!10} v-BYOL$_{\varrho=4}$ + DSA & 112$\times$112 & 16 & R3D-18 & 13.5M & 200  &\bf 92.7   &\bf 70.5  \\

    \midrule
        CVRL \cite{qianSpatiotemporalContrastiveVideo2021} & 224$\times$224 & 32 & R3D-50 & 31.8M & 800  & 92.2 & 66.7 \\
        MACLR \cite{xiao2022maclr} & 224$\times$224 & 32 & R3D-50 & 31.8M & 600  & 94.0 & 67.4 \\
        v-SimCLR$_{\varrho=2}$ \cite{feichtenhoferLargeScaleStudyUnsupervised2021} & 224$\times$224 & 8 & R3D-50 & 31.8M & 200  & 88.9 & 67.2  \\ 
        v-SwAV$_{\varrho=2}$ \cite{feichtenhoferLargeScaleStudyUnsupervised2021} & 224$\times$224 & 8 & R3D-50 & 31.8M & 200  & 87.3 & 68.3  \\ 
        v-MoCo$_{\varrho=4}$ \cite{feichtenhoferLargeScaleStudyUnsupervised2021} & 224$\times$224 & 8 & R3D-50 & 31.8M & 200  & 93.5 & 71.6  \\
        v-BYOL$_{\varrho=4}$ \cite{feichtenhoferLargeScaleStudyUnsupervised2021} & 224$\times$224 & 8 & R3D-50 & 31.8M & 200  & 94.2 & 72.1 \\
    \midrule
    \rowcolor{orange!10}    v-SimCLR$_{\varrho=2}$ + DSA & 224$\times$224 & 8 & R3D-50 & 31.8M & 200  & 90.7   & 69.5  \\ 
    \rowcolor{orange!10}    v-SwAV$_{\varrho=2}$ + DSA & 224$\times$224 & 8 & R3D-50 & 31.8M & 200  & 91.1  & 69.7  \\ 
    \rowcolor{orange!10}    v-MoCo$_{\varrho=4}$ + DSA & 224$\times$224 & 8 & R3D-50 & 31.8M & 200  & 94.9   &\bf 73.7  \\
    \rowcolor{orange!10}    v-BYOL$_{\varrho=4}$ + DSA & 224$\times$224 & 8 & R3D-50 & 31.8M & 200  &\bf 95.4   & 73.6  \\ 
    \bottomrule
    \end{tabular}
    }
\end{table*}

\begin{table*}[htb]
  \caption{Few-shot learning accuracies (\%) with 95\% confidence intervals averaged over 2000 episodes on FC100, CUB200, and Plant Disease. $(N,K)$ denotes $N$-way $K$-shot tasks.}
  \label{tab:few-shot}
  \centering
\resizebox{\linewidth}{!}{
\begin{tabular}{lcccccc}
  \toprule
\multirow{2}{*}{\textbf{Method}} & \multicolumn{2}{c}{\textbf{FC100}} & \multicolumn{2}{c}{\textbf{CUB200}} & \multicolumn{2}{c}{\textbf{Plant Disease}}\\
  \cmidrule(r){2-3}
  \cmidrule(r){4-5}
  \cmidrule(r){6-7}
    & \textbf{(5,1)} & \textbf{(5,5)} & \textbf{(5,1)} & \textbf{(5,5)} & \textbf{(5,1)} & \textbf{(5,5)}\\
    \midrule
    Supervised &34.62 $\pm$ 0.88&47.15 $\pm$ 0.23&47.34 $\pm$ 0.66&64.87 $\pm$ 0.21&70.08 $\pm$ 0.24&89.83 $\pm$ 0.41\\
    \midrule
    SimCLR \cite{chenSimpleFrameworkContrastive2020} &39.89 $\pm$ 0.28&49.83 $\pm$ 0.70&44.47 $\pm$ 0.76&67.33 $\pm$ 0.65&76.94 $\pm$ 0.56&89.73 $\pm$ 0.67\\
    Barlow Twins \cite{zbontarBarlowTwinsSelfSupervised2021} &40.02 $\pm$ 0.73&52.69 $\pm$ 0.11&46.02 $\pm$ 0.21&65.06 $\pm$ 0.54&79.37 $\pm$ 0.11&89.53 $\pm$ 0.42\\
    BYOL \cite{zbontarBarlowTwinsSelfSupervised2021} &36.41 $\pm$ 0.43&51.36 $\pm$ 0.59&44.38 $\pm$ 0.33&62.75 $\pm$ 0.27&79.41 $\pm$ 0.24&91.05 $\pm$ 0.31\\
    W-MSE \cite{ermolov2021whitening} &40.97 $\pm$ 0.62&53.69 $\pm$ 0.56&49.15 $\pm$ 0.63&65.15 $\pm$ 0.74&75.52 $\pm$ 0.24&89.05 $\pm$ 0.64\\
    VICRegL \cite{Vicregl} &40.78 $\pm$ 0.48&54.28 $\pm$ 0.66&49.44 $\pm$ 0.82&67.18 $\pm$ 0.61 &79.36 $\pm$ 0.22&92.84 $\pm$ 0.24\\
    SimSiam \cite{chenExploringSimpleSiamese2021} &37.11 $\pm$ 0.86&51.64 $\pm$ 0.32&46.28 $\pm$ 0.68&63.25 $\pm$ 0.14&76.83 $\pm$ 0.16&90.45 $\pm$ 0.75\\
    ReSSL \cite{ressl} &37.92 $\pm$ 0.64&52.35 $\pm$ 0.43&46.42 $\pm$ 0.42&63.76 $\pm$ 0.45&77.25 $\pm$ 0.86&91.15 $\pm$ 0.76\\
    SwAV \cite{caronUnsupervisedLearningVisual2020}&39.64 $\pm$ 0.11&51.83 $\pm$ 0.55&47.34 $\pm$ 0.16&65.24 $\pm$ 0.75&79.41 $\pm$ 0.35&92.62 $\pm$ 0.61\\
    \midrule
   \rowcolor{orange!10}SimCLR + DSA &\bf 43.04 $\pm$ 0.37&52.38 $\pm$ 0.15&46.96 $\pm$ 0.12&\bf 70.52 $\pm$ 0.37&78.31 $\pm$ 0.71&91.23 $\pm$ 0.82\\
    \rowcolor{orange!10}Barlow Twins + DSA &41.42 $\pm$ 0.88&\bf 55.47 $\pm$ 0.67&47.36 $\pm$ 0.39&68.04 $\pm$ 0.32&82.18 $\pm$ 0.73&90.95 $\pm$ 0.18\\
    \rowcolor{orange!10}BYOL + DSA &39.85 $\pm$ 0.61&52.95 $\pm$ 0.46&46.13 $\pm$ 0.69&64.79 $\pm$ 0.76&\bf 82.89 $\pm$ 0.74&\bf 94.72 $\pm$ 0.89\\
    \rowcolor{orange!10}SwAV + DSA &41.03 $\pm$ 0.21&54.12 $\pm$ 0.56&\bf 49.36 $\pm$ 0.10&68.27 $\pm$ 0.55&81.34 $\pm$ 0.68&93.20 $\pm$ 0.26\\
    \bottomrule
  \end{tabular}
  }
\end{table*}

\begin{table*}[htb]
	\centering
	\caption{%
{Transfer learning results on object detection and instance segmentation using the C4-backbone.}  
``AP''  {refers to the average precision,}  
``$\text{AP}_{N}$'' %
{denotes the average precision at IoU threshold $N\%$.}%
}
\label{tab:voc_coco}
		\resizebox{\linewidth}{!}{
		
		\begin{tabular}{lcccccccccccc}
		\toprule
		\multirow{2.5}{*}{Method} &\multicolumn{3}{c}{VOC 07 detection} & \multicolumn{3}{c}{VOC 07+12 detection} &\multicolumn{3}{c}{COCO detection}&\multicolumn{3}{c}{COCO instance segmentation}\\
	    \cmidrule(lr){2-4} \cmidrule(lr){5-7} \cmidrule(lr){8-10} \cmidrule(lr){11-13} 
	    & \(\mathbf{AP_{50}}\)& \(\mathbf{AP}\) & \(\mathbf{AP_{75}}\)& \(\mathbf{AP_{50}}\)& \(\mathbf{AP}\) & \(\mathbf{AP_{75}}\)& \(\mathbf{AP_{50}}\)& \(\mathbf{AP}\) & \(\mathbf{AP_{75}}\)& \(\mathbf{AP^{mask}_{50}}\)& \(\mathbf{AP^{mask}}\) & \(\mathbf{AP^{mask}_{75}}\)\\
	       \midrule
	     Supervised & 74.4 & 42.4 & 42.7 & 81.3 & 53.5 & 58.8 & 58.2 & 38.2 & 41.2 & 54.7 & 33.3 & 35.2\\
	   \midrule
	     SimCLR \cite{chenSimpleFrameworkContrastive2020} & 75.9 & 46.8 & 50.1 & 81.8 & 55.5 & 61.4 & 57.7 & 37.9 & 40.9 & 54.6 & 33.3 & 35.3\\
	     MoCo \cite{heMomentumContrastUnsupervised2020} & 77.1 & 46.8 & 52.5 & 82.5 & 57.4 & 64.0 & 58.9 & 39.3 & 42.5 & 55.8 & 34.4 & 36.5\\
	     BYOL \cite{grillBootstrapYourOwn2020} & 77.1 & 47.0 & 49.9 & 81.4 & 55.3 & 61.1 & 57.8 & 37.9 & 40.9 & 54.3 & 33.2 & 35.0\\
	     SimSiam \cite{chenExploringSimpleSiamese2021} & 77.3 & 48.5 & 52.5 & 82.4 & 57.0 & 63.7 & 59.3 & 39.2 & 42.1 & 56.0 & 34.4 & 36.7\\
	     Barlow Twins \cite{zbontarBarlowTwinsSelfSupervised2021} & 75.7 & 47.2 & 50.3 & 82.6 & 56.8 & 63.4 & 59.0 & 39.2 & 42.5 & 56.0 & 34.3 & 36.5\\
        SwAV \cite{caronUnsupervisedLearningVisual2020} & 75.5 & 46.5 & 49.6 & 82.6 & 56.1 & 62.7 & 58.6 & 38.4 & 41.3 & 55.2 & 33.8 & 35.9\\
	     MEC \cite{MEC} & 77.4 & 48.3 & 52.3 & 82.8 & 57.5 & 64.5 & 59.8 & 39.8 & 43.2 & 56.3 & 34.7 & 36.8\\
	     RELIC v2 \cite{RELIC-v2} & 76.9 & 48.0 & 52.0 & 82.1 & 57.3 & 63.9 & 58.4 & 39.3 & 42.3 & 56.0 & 34.6 & 36.3\\
	 CorInfoMax \cite{CorInfoMax}& 76.8 & 47.6 & 52.2 & 82.4 & 57.0 & 63.4 & 58.8 & 39.6 & 42.5 & 56.2 & 34.8 & 36.5\\
      VICRegL \cite{Vicregl}& 75.9 & 47.4 & 52.3 & 82.6 & 56.4 & 62.9 & 59.2 & 39.8 & 42.1 & 56.5 & 35.1 & 36.8\\    
     \midrule
\rowcolor{orange!10}SimCLR + DSA &77.5&47.9&52.0&83.6&57.5&63.9&59.0&39.5&42.9&56.3&35.4&36.3\\
\rowcolor{orange!10}MoCo + DSA &\bf 79.4&48.2&54.1&84.4&\bf 58.6&\bf 66.4&61.0&40.8&\bf 44.3&57.1&36.2&\bf 38.9\\
\rowcolor{orange!10}BYOL + DSA &78.5&48.4&52.1&83.1&57.0&62.6&59.0&39.3&43.3&55.8&35.3&37.2\\
\rowcolor{orange!10}SimSiam + DSA &78.7&\bf50.1&54.2&\bf 84.7&58.3&65.0&60.4&41.0&43.2&\bf 58.1&36.9&\bf 38.9\\
\rowcolor{orange!10}SwAV + DSA &77.4&48.1&52.0&83.9&58.5&64.9&60.3&40.1&43.8&57.2&35.1&37.6\\
\rowcolor{orange!10}VICRegL + DSA &78.4&49.1&\bf 54.6&83.8&57.8&64.8&\bf 60.6&\bf 42.2&43.4&57.7&\bf 37.5&38.2\\
		\bottomrule
	\end{tabular}
	}
\end{table*}

\subsection{Standard Evaluation}
\label{sec:classification}
 {To assess how effectively our method improves the learned feature extractor,} we perform the standard evaluation protocol, i.e. linear evaluation on the obtained feature extractors.
This protocol involves freezing the feature extractor $f(\cdot)$ after pre-training and subsequently training a supervised linear classifier with softmax on top of it. We present the average top-1 and top-5 accuracy and the standard deviation of 5 runs for pre-trained image feature extractor in Table \ref{tab:imagenet}. The linear evaluation results for video data are shown in Table \ref{tab:k400}.

From Table \ref{tab:imagenet}, we can observe that the inclusion of DSA significantly improves the linear evaluation performance compared to the SSL baselines, with an increase in average top-1 accuracy by  2.04\%, and an increase in top-5 accuracy by 1.57\%. Furthermore, DSA is effective across different backbones, with an average improvement of top-1 accuracy of 1.88\% for ViT-B and 2.13\% for ResNet-50. Also, when using ResNet-50 as the backbone, VICRegL+DSA achieved the highest average top-1 accuracy of 78.15\%, surpassing the best previous result of W-MSE by 2.14\%, when using ViT-B as the backbone, DINO+DSA achieved the highest average top-1 accuracy of 79.20\%, exceeding that of DINO by 1.05\%. Moreover, despite the linear evaluation results of MAE being significantly lower compared to other methods, the introduction of DSA results in a notable improvement, with an increase in top-1 accuracy by 5.91\%. 

{Table~\ref{tab:k400} indicates that incorporating DSA yields a 1.4\% gain in top-1 classification accuracy on the Kinetics-400 validation split using SSL-based models.} 
This demonstrates that our proposed DSA is not only applicable to a specific type of data but is effective across various data modalities.

In summary, DSA significantly enhances the accuracy of SSL methods on the validation set during standard linear evaluation, aligning with the conclusions drawn from our theoretical analysis.

\subsection{Downstream Evaluation}
\label{sec:downstream}
To further explore the performance of our proposed DSA, we tested the pre-trained SSL feature extractor on a range of downstream tasks. These tasks include semi-supervised classification, action recognition, few-shot classification, object detection, instance segmentation, and action detection. Experimental results indicate that DSA achieves significant improvements across all these downstream tasks.

\paragraph{Semi-Supervised Classification}
We follow the setting of \cite{chenSimpleFrameworkContrastive2020} and sample 1\% and 10\% of the labeled data from ImageNet \cite{dengImageNetLargescaleHierarchical2009} in a class-balanced way, with each class containing around 12.8 to 128 images per class{, }respectively. We then fine-tune the pre-trained backbone with our method 
for 50 epochs with these sampled labeled data and report the average top-1 and top-5 accuracies of the test set.

The results for semi-supervised classification are shown in Table \ref{tab:semi}. These results demonstrate that when the amount of labeled data is limited, the fine-tuning performance of pre-trained models based on self-supervised learning methods surpasses that of supervised baselines. Additionally, the pre-trained models incorporating the DSA method exhibit superior performance in semi-supervised tasks compared to their self-supervised baseline counterparts.  
{When merely 1\% of the labels are provided, the mean top-1 accuracy rises by 1.65\%, accompanied by a 1.78\% gain in top-5 accuracy.} 
{For 10\% labeled data, the model achieves a 1.73\% improvement in both top-1 and top-5 accuracies.}

{In comparison to the strongest baseline under 1\% supervision, RELIC v2 combined with BYOL+DSA shows enhancements of 2.3\% in top-1 and 1.1\% in top-5 accuracy.}

{Against the best performing method using 10\% annotated samples, SSL-HSIC paired with BYOL+DSA delivers additional gains of 1.3\% and 1.0\% in top-1 and top-5 accuracy, respectively.}

In summary, the inclusion of DSA results in significant improvements in semi-supervised classification tasks. This indicates that enhancing model discriminability has a notable impact on the fine-tuning of downstream classification tasks, even when labeled data is limited.

\paragraph{Action Recognition}
To  {assess how well the SSL-pretrained models generalize to downstream tasks}, 
we follow the setting in~\cite{qianSpatiotemporalContrastiveVideo2021,feichtenhoferLargeScaleStudyUnsupervised2021} and  {conduct downstream adaptation using UCF-101 and HMDB-51}. 
We use the pre-trained backbone on Kinetics to initialize the network parameters and {additionally append} a classification layer. 
On each dataset, the models are trained for 50 epochs with a batch size of 128. 
The recognition results obtained from the fine-tuning stage are {reported in} Table~\ref{tab:ucf}. 
To enable a fair comparison with existing small-capacity video SSL methods, we also evaluate R3D-18-based models. 
From the table, {adding DSA leads to a 1.94\% boost in top-1 recognition performance} over standard video self-supervised learning. 
Furthermore, with 16-frame inputs and R3D-18 as backbone, v-BYOL + DSA {surpasses} the top result of MACLR by 1.4\% on UCF-101.

These results demonstrate that the inclusion of DSA significantly enhances accuracy in downstream action recognition tasks for video SSL. This further indicates that our method improves the discriminability of the learned representations.

\paragraph{Few-shot Classification}
We  {conduct evaluation of the few-shot classification} on FC100, CUB200, and Plant Disease following the standard cross-domain few-shot setting . 
We  {carry out this task} by performing logistic regression using the pre-trained ResNet-50 with the parameters frozen and report the results for both the $5$-way $1$-shot and $5$-way $5$-shots task {(see Table~\ref{tab:few-shot})}. 
From the results, we can observe that in few-shot classification tasks, SSL methods with DSA show an average accuracy improvement of 2.26\% compared to their baselines. 
Moreover, the inclusion of DSA achieves better results across different datasets. 
For instance,  {within the 5-way 1-shot task conducted using the FC100 dataset}, the best self-supervised learning result is 40.97\% achieved by W-MSE. 
In contrast, SimCLR + DSA exceeds this metric by 2.07\%. 
Similarly,  {under the setting of the 5-way 5-shot task on the Plant Disease dataset}, VICRegL achieves an average accuracy of 92.84\%, while BYOL + DSA surpasses this result by 1.88\%.

These findings indicate that DSA significantly  {enhances SSL model capability in few-example recognition settings}.

\paragraph{Detection and Segmentation}
The evaluation of downstream detection and segmentation tasks  {is performed} on Pascal-VOC and MS-COCO datasets.  
For object detection on Pascal-VOC, we fine-tune a Faster R-CNN~\cite{ren2015faster}  {using the C4-backbone}.  
Results are reported for   
{training on the VOC 2007 trainval set and the combined VOC 2007 trainval + 2012 train set}, both  {assessed on the VOC 2007 test split}.  
For object detection on COCO, we fine-tune a Mask R-CNN~\cite{he2017mask} (1$\times$ schedule)  {based on the C4-backbone}.  
Fine-tuning is performed on the COCO 2017 train set and  {tested on the COCO 2017 validation set}.  
The evaluation tasks on VOC and COCO  {adhere to the procedures described in} \cite{chenExploringSimpleSiamese2021}, where the pre-trained Resnet-50 are used for initialization.  
We  {present AP scores} at different intersections over union (IoU) ratios for results from VOC and COCO in Table~\ref{tab:voc_coco}.

\begin{figure*}[htb]
    \centering
    \subfigure[{SimCLR + DSA}]{\includegraphics[width=0.3\textwidth]{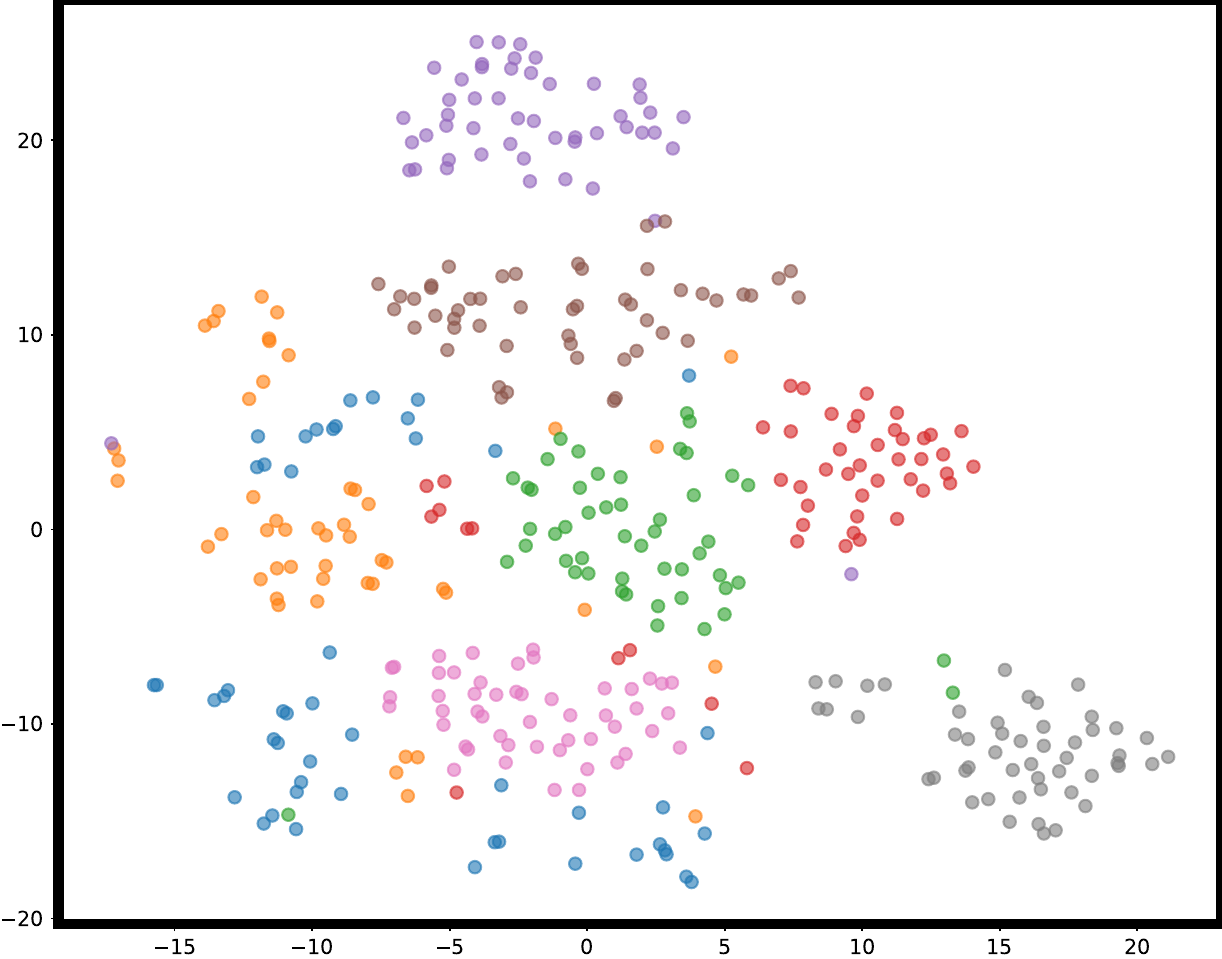}\label{fig:tsne_simclr_better}}
     \subfigure[{BYOL + DSA}]{\includegraphics[width=0.3\textwidth]{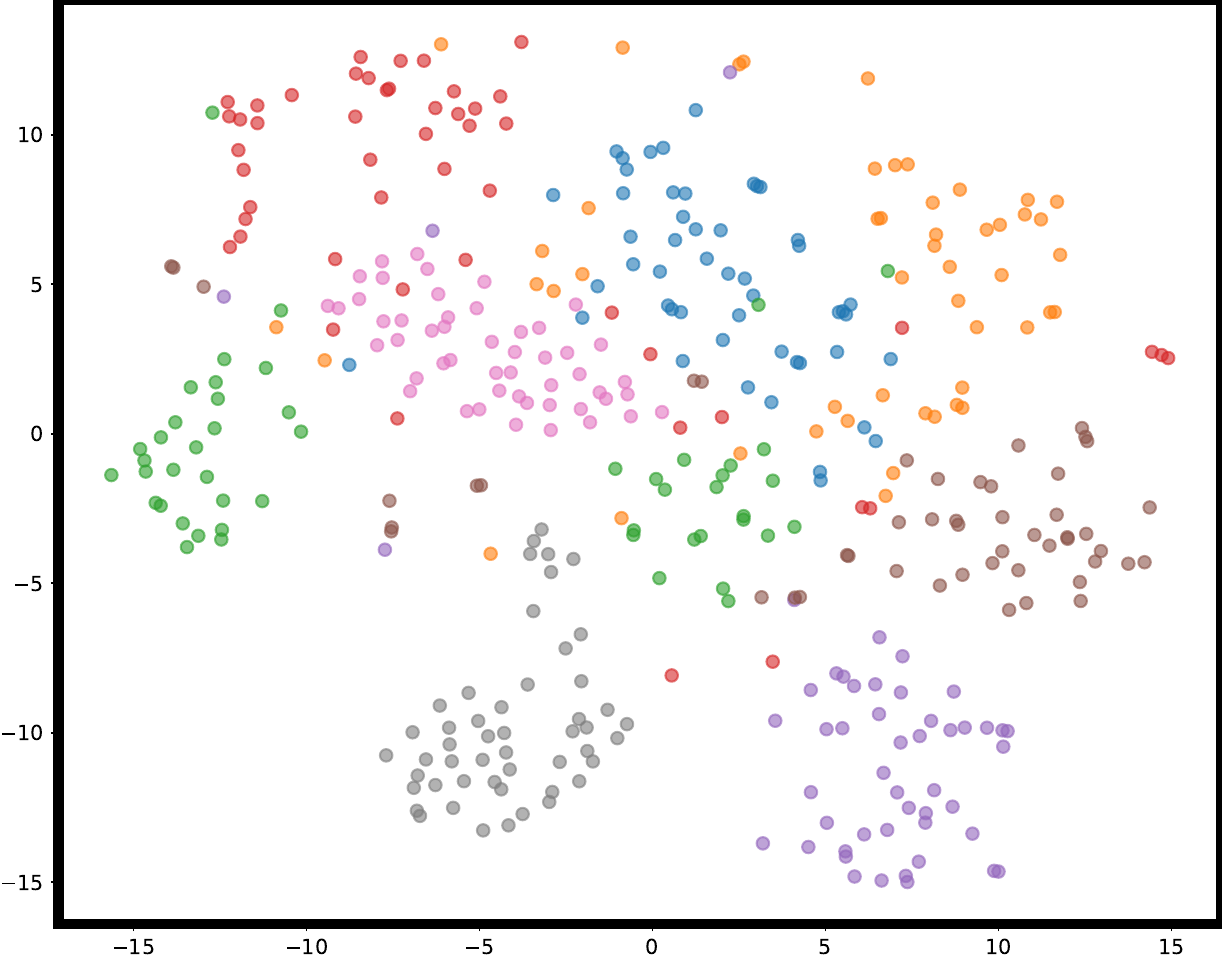}\label{fig:tsne_byol_better}}
     \subfigure[{Barlow Twins + DSA}]{\includegraphics[width=0.3\textwidth]{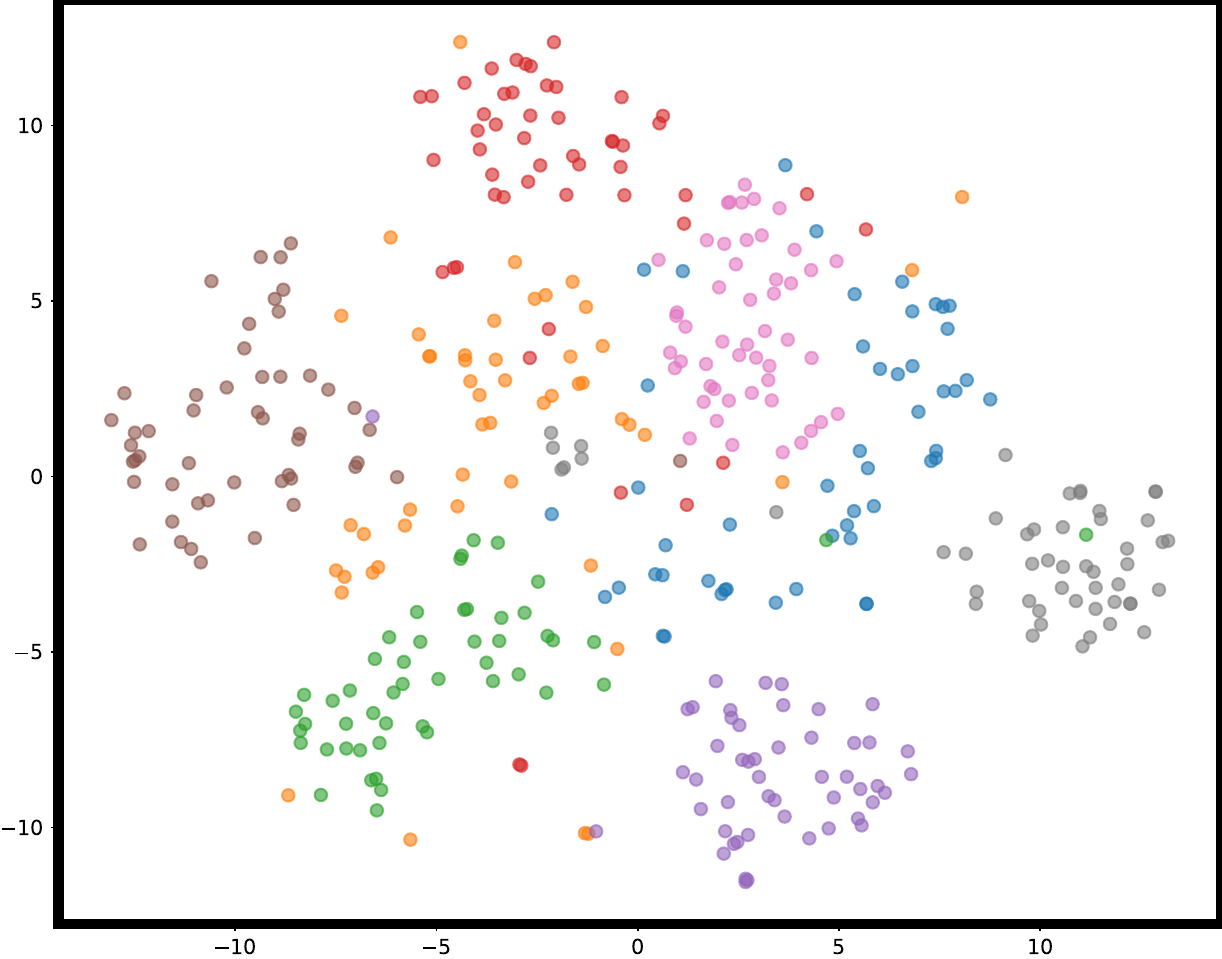}\label{fig:tsne_barlow_better}}
     \subfigure[{SwAV + DSA}]{\includegraphics[width=0.3\textwidth]{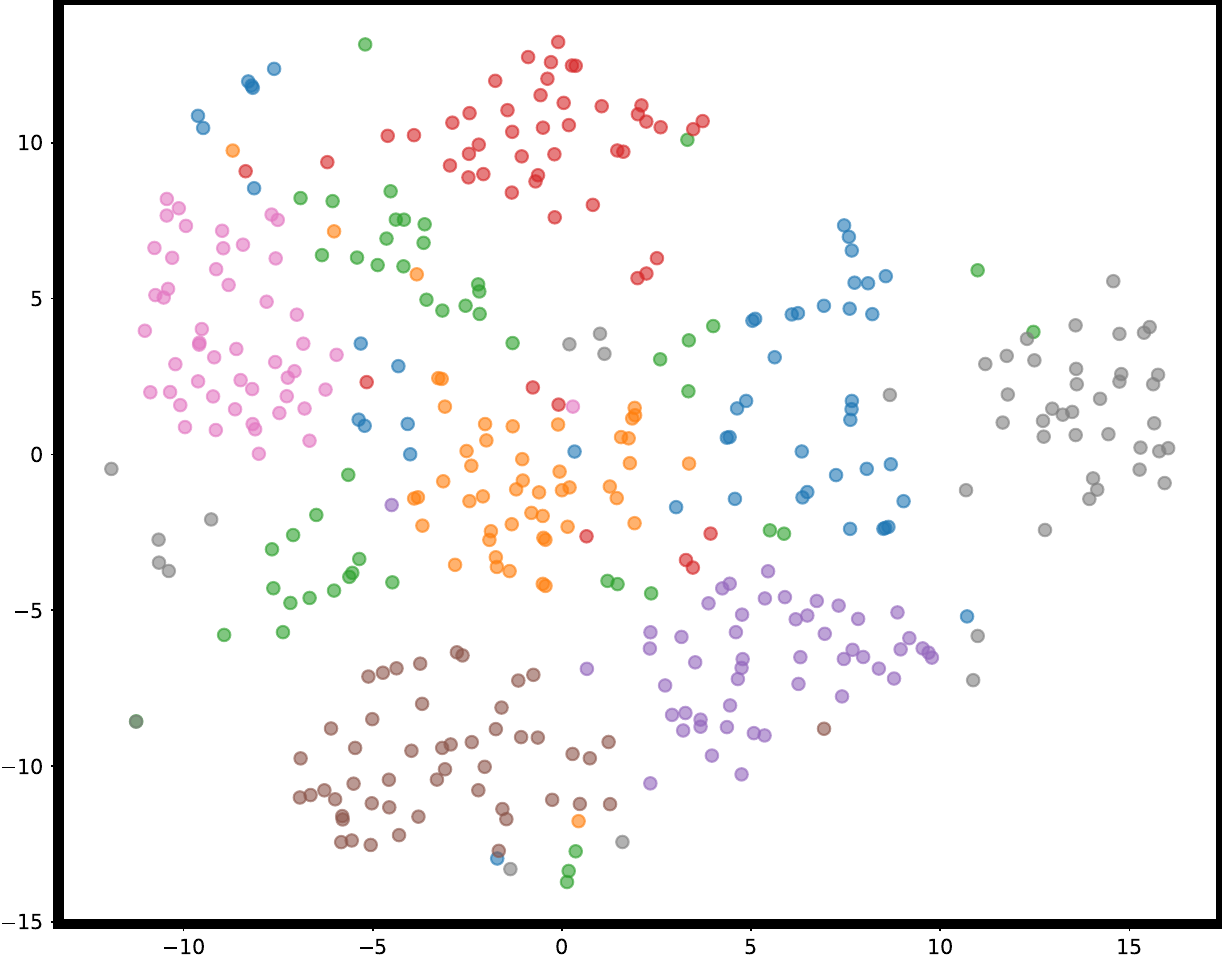}\label{fig:tsne_swav_better}}
     \subfigure[{MAE + DSA}]{\includegraphics[width=0.3\textwidth]{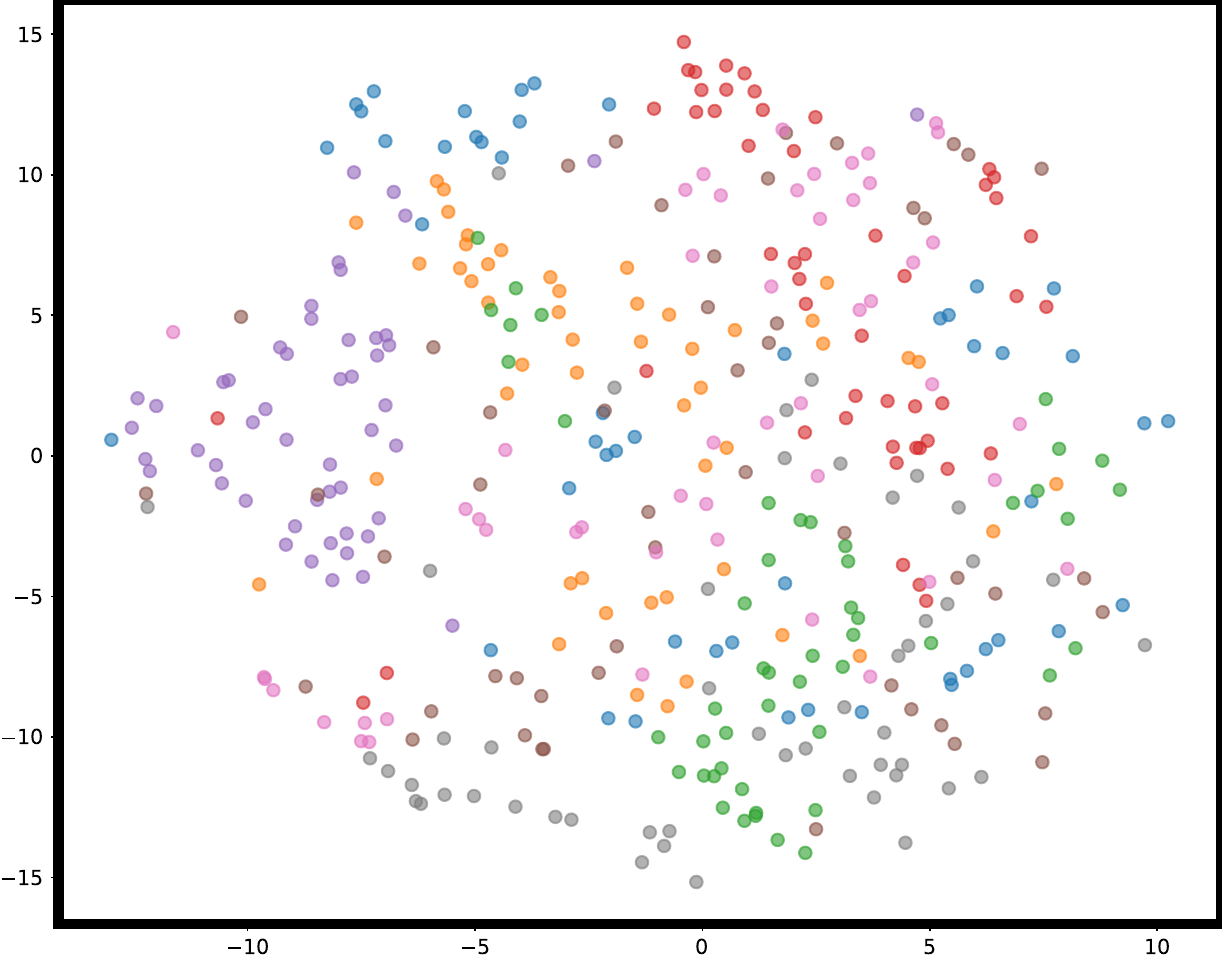}\label{fig:tsne_mae_better}}
    \caption{Data distribution visualization based on the same set of classes as Figure \ref{fig:tsne} of the test set of ImageNet in the feature space. (a) - (e) corresponds to the visualization results of the self-supervised methods integrated with DSA. We can observe that the border between classes does not overlap (except for MAE), and the inter-class features are clustered together. }
    \label{fig:tsne_better}
\end{figure*}

\begin{figure*}[htb]
    \centering
    \subfigure[$\nu$]{\includegraphics[width=0.3\textwidth]{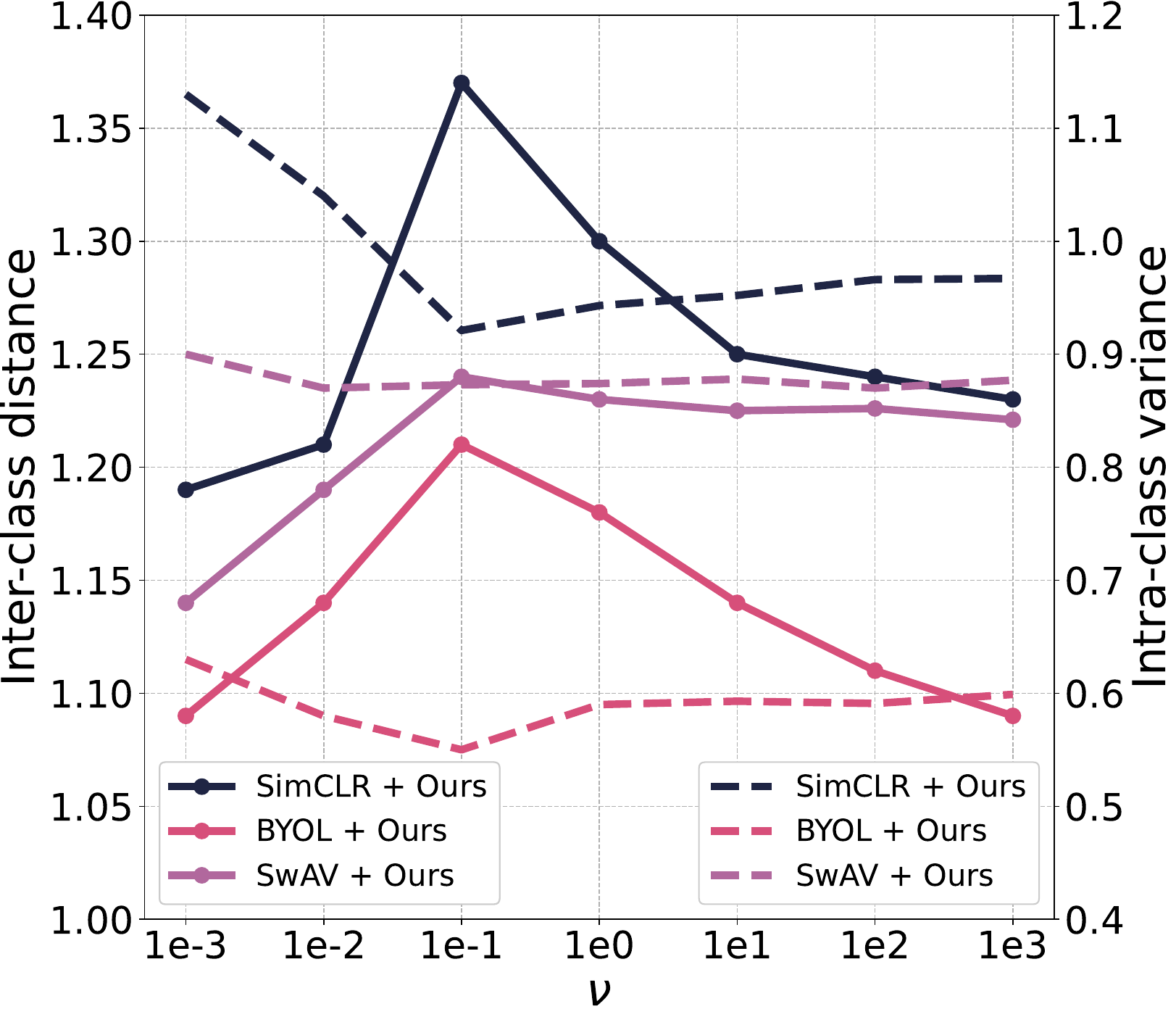}\label{fig:nu_acc}}
    \subfigure[$\upsilon$]{\includegraphics[width=0.3\textwidth]{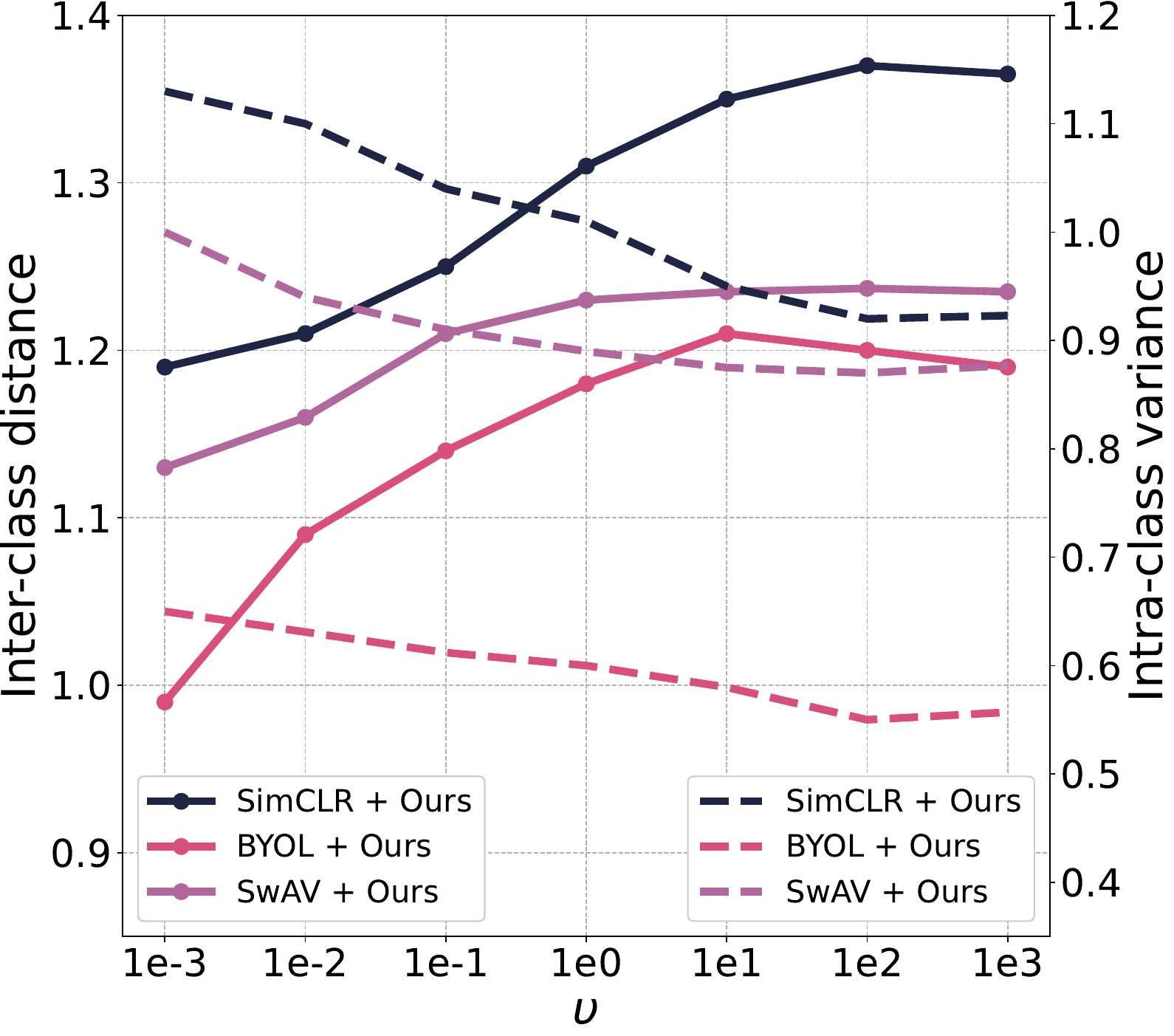}\label{fig:upsilon_acc}}
    \subfigure[$\alpha$]{\includegraphics[width=0.3\textwidth]{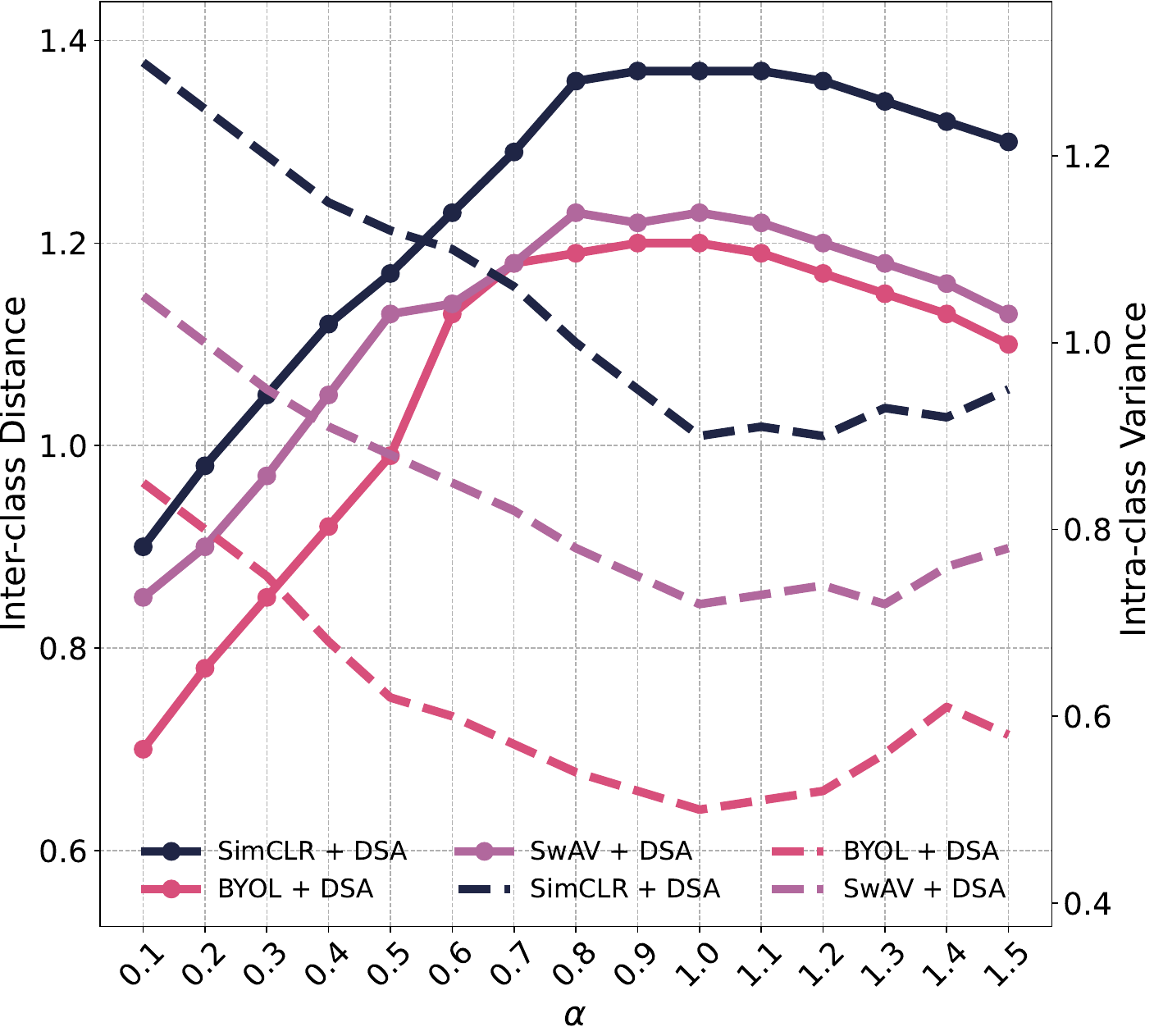}\label{fig:alpha_acc}}
    \subfigure[$\eta$]{\includegraphics[width=0.3\textwidth]{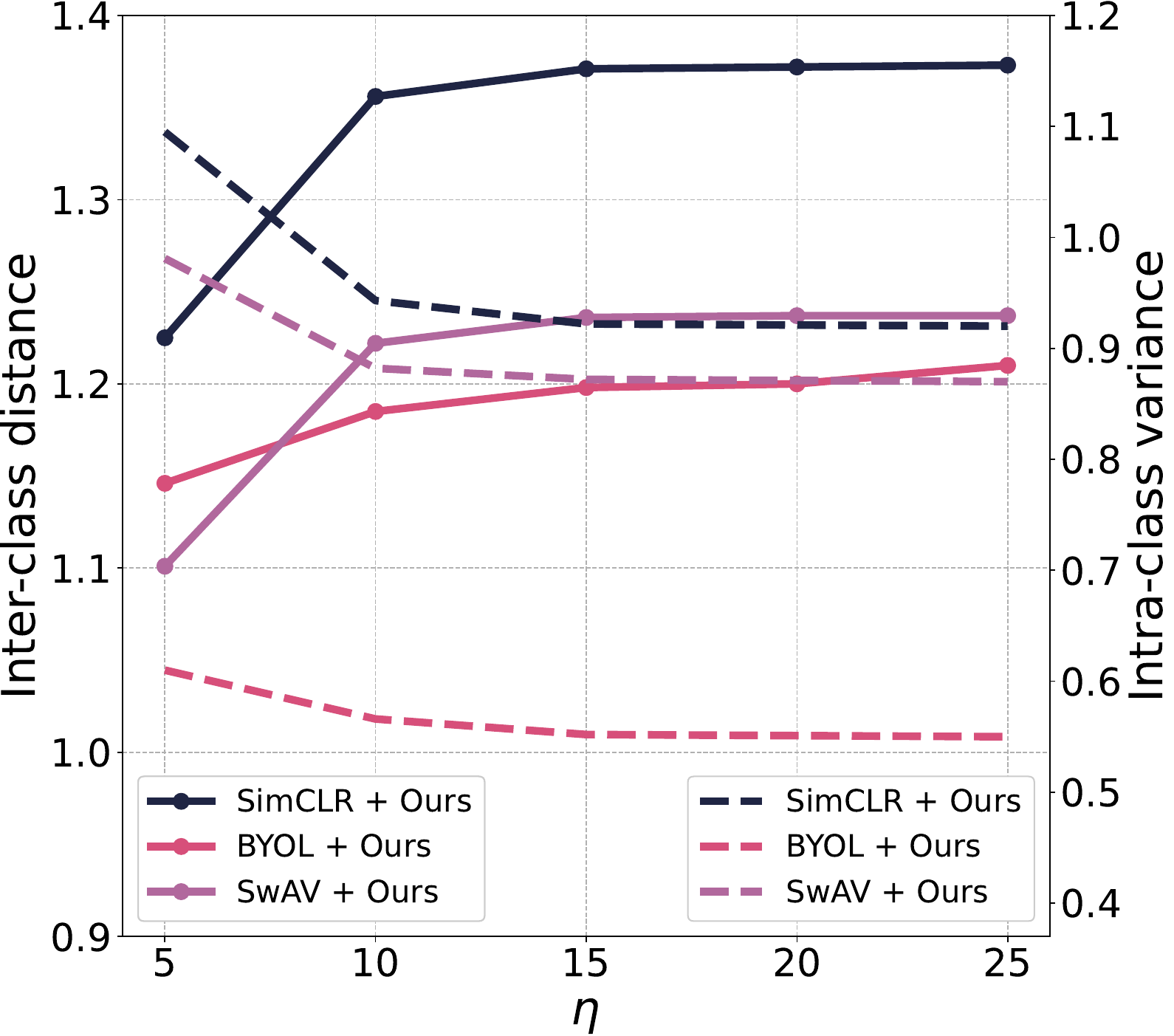}\label{fig:eta_acc}}
    \subfigure[$\tau$]{\includegraphics[width=0.3\textwidth]{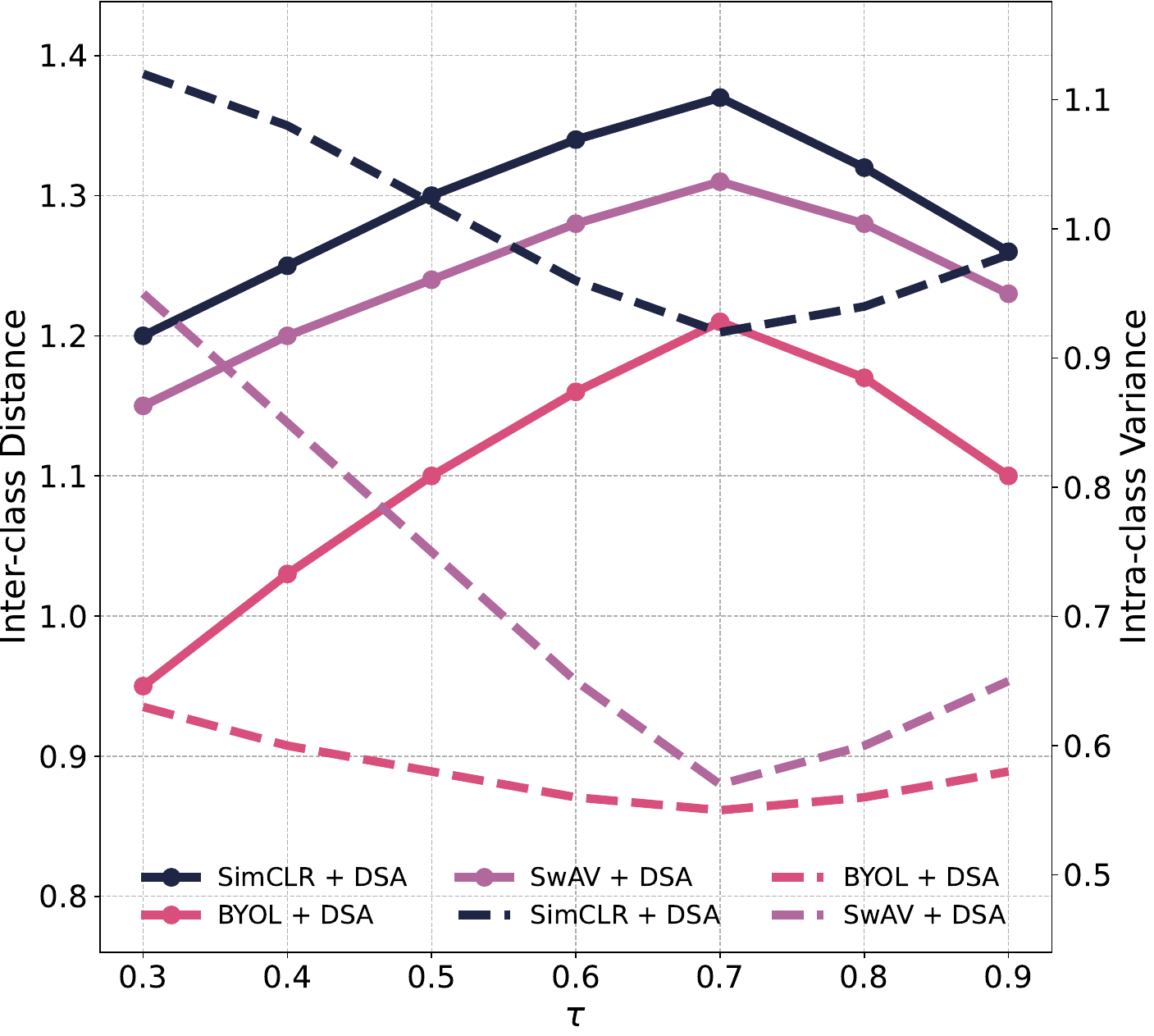}\label{fig:tau_acc}}

    \caption{{The inter-class distance and intra-class variance correspond to different values of hyper-parameters $\nu$, $\upsilon$, $\alpha$, $\eta$, and $\tau$. The solid lines in figures (a) - (e) represent the inter-class distance, while the dashed lines represent the intra-class variance.} }
    \label{fig:hyperparam}
\end{figure*}

\begin{figure*}[htb]
    \centering
    \subfigure[$\nu$]{\includegraphics[width=0.3\textwidth]{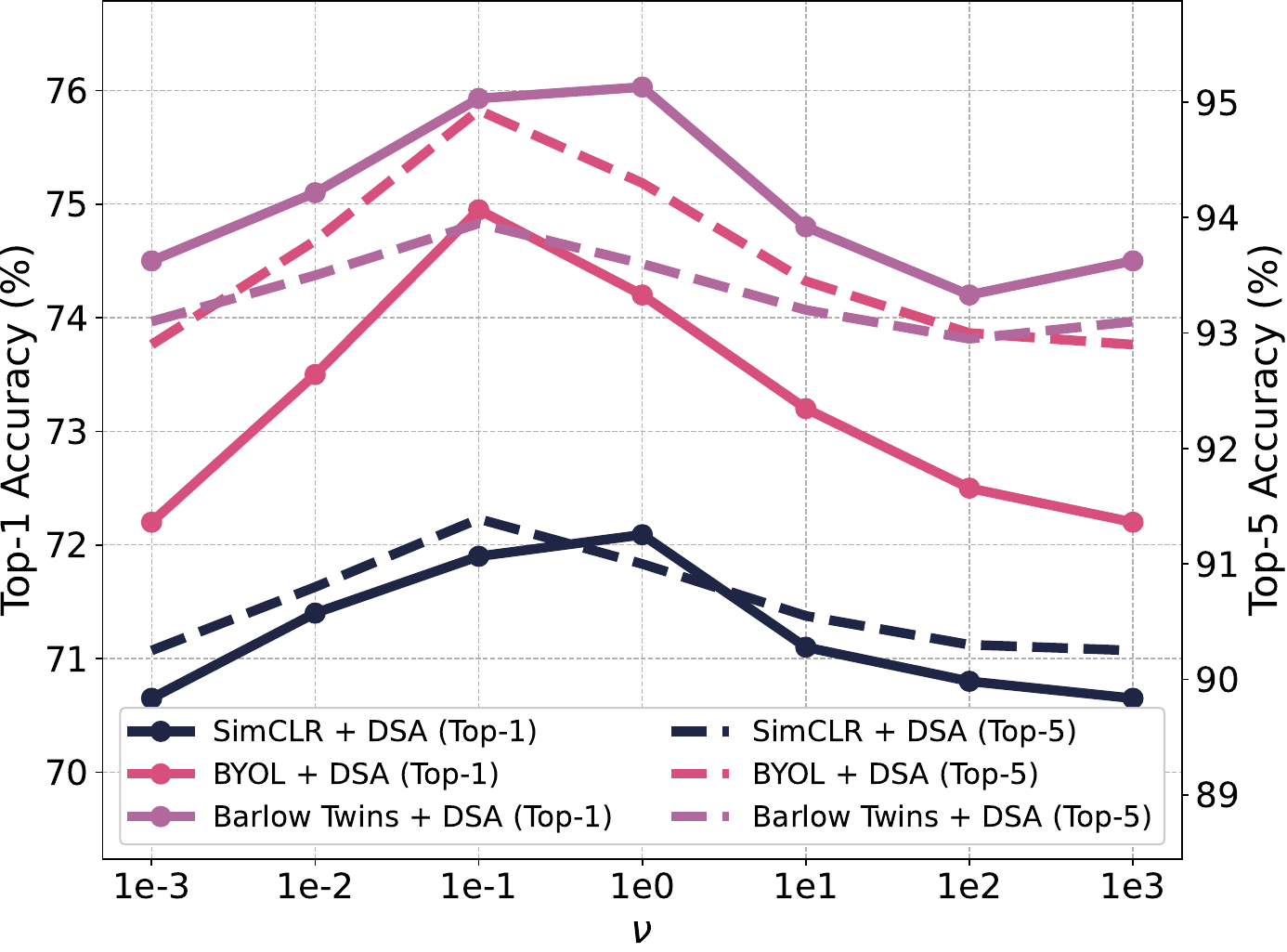}\label{fig:nu_imagenet}}
    \subfigure[$\upsilon$]{\includegraphics[width=0.3\textwidth]{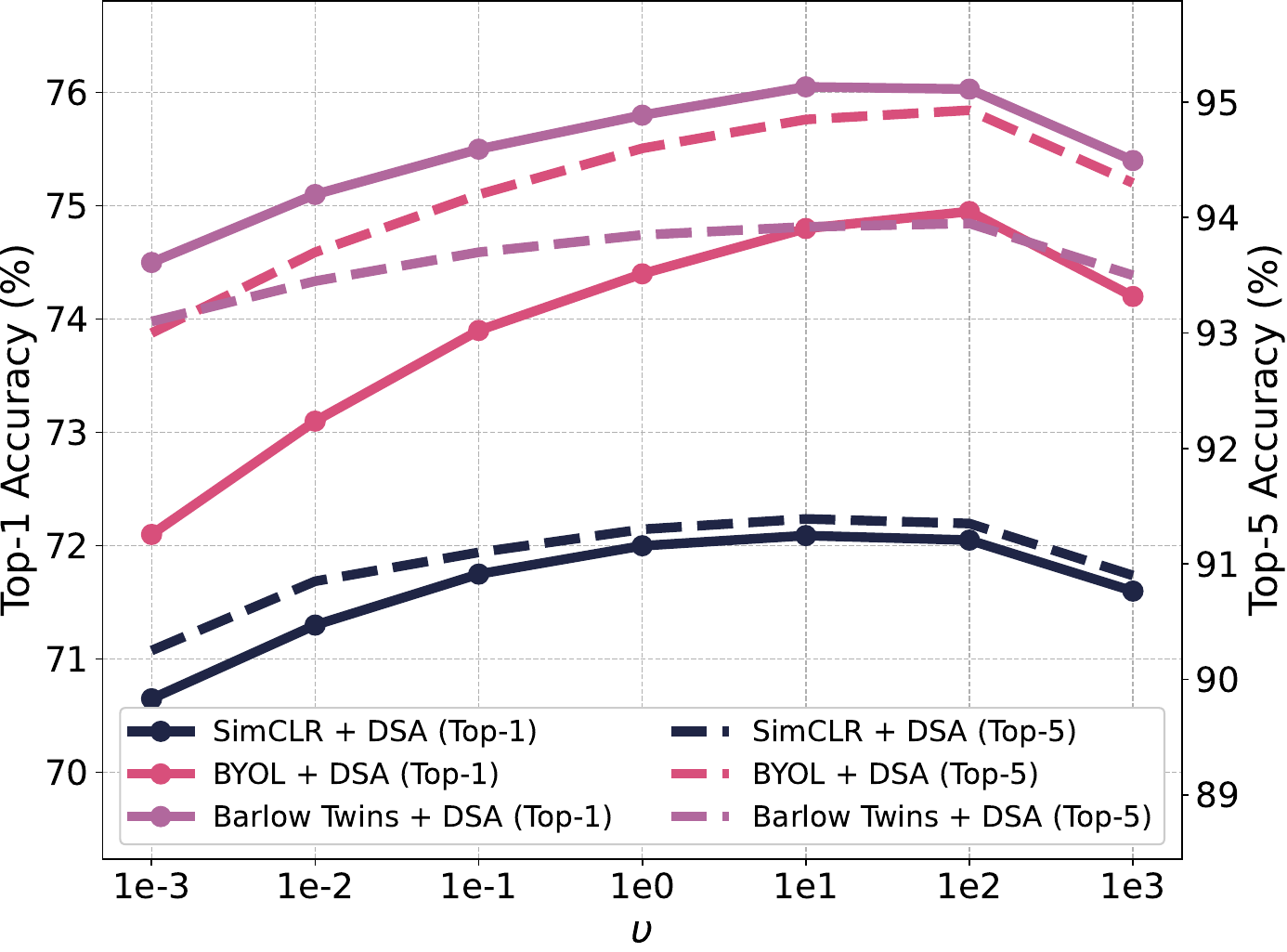}\label{fig:upsilon_imagenet}}
    \subfigure[$\alpha$]{\includegraphics[width=0.3\textwidth]{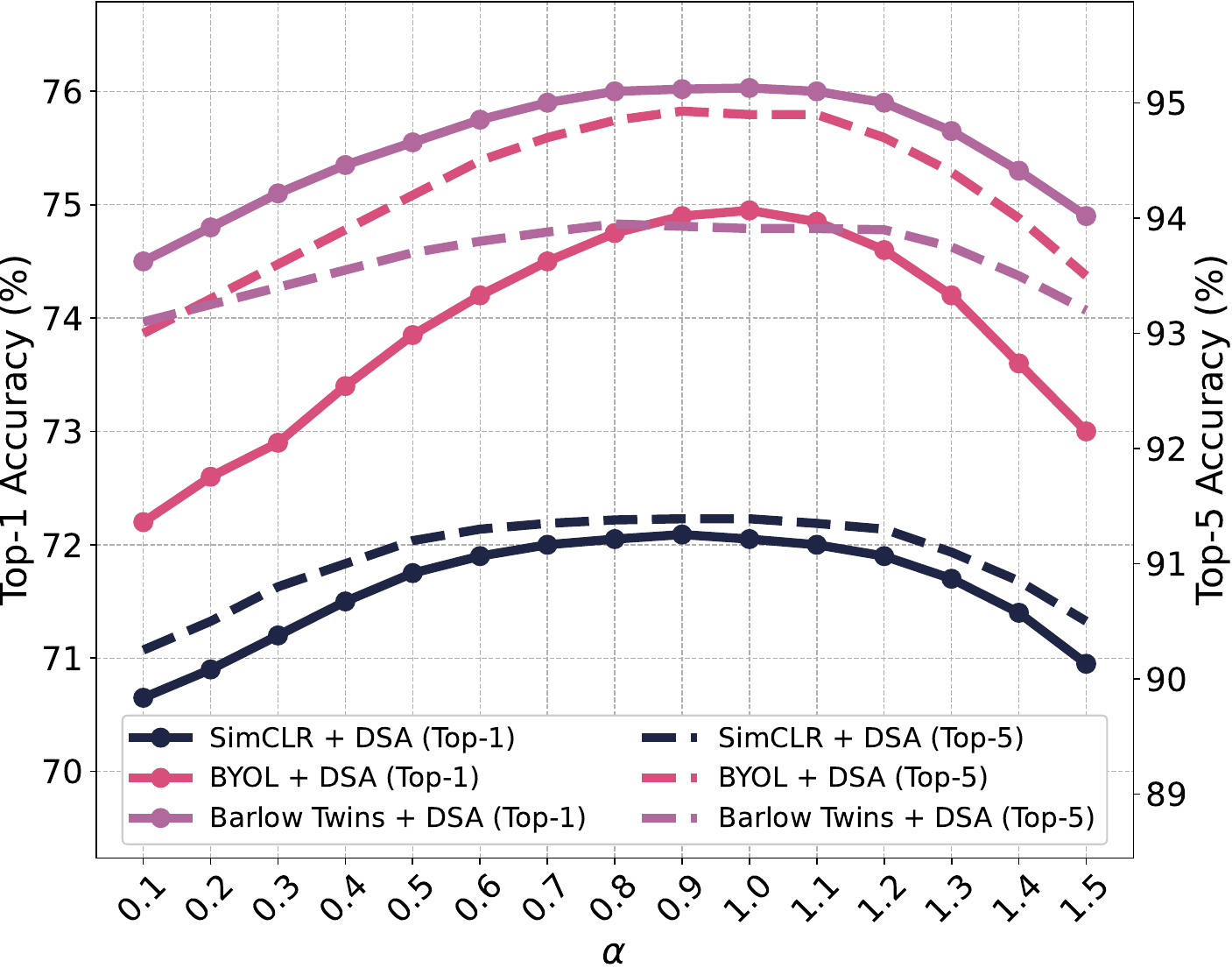}\label{fig:alpha_imagenet}}
    \subfigure[$\eta$]{\includegraphics[width=0.3\textwidth]{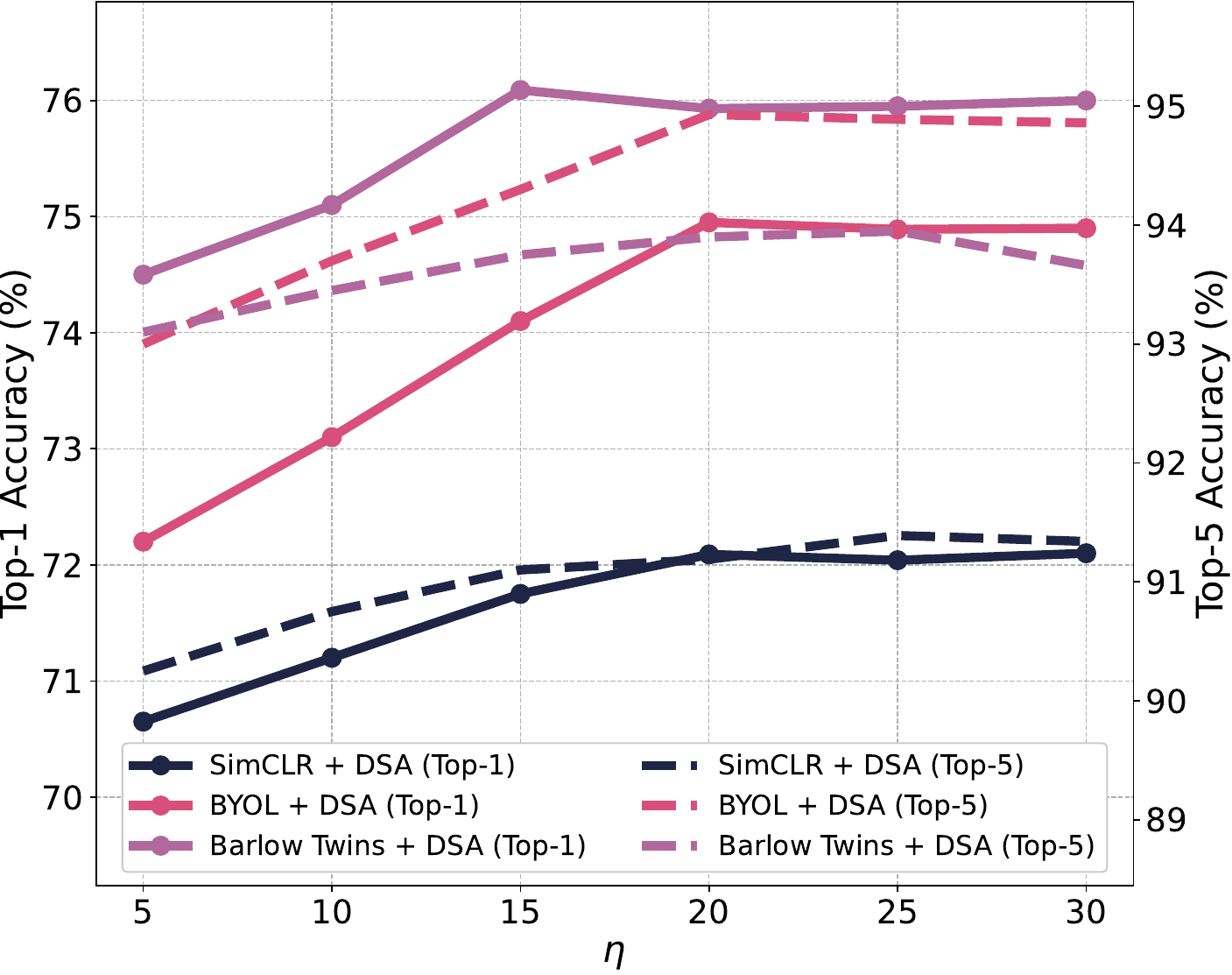}\label{fig:eta_imagenet}}
    \subfigure[$\tau$]{\includegraphics[width=0.3\textwidth]{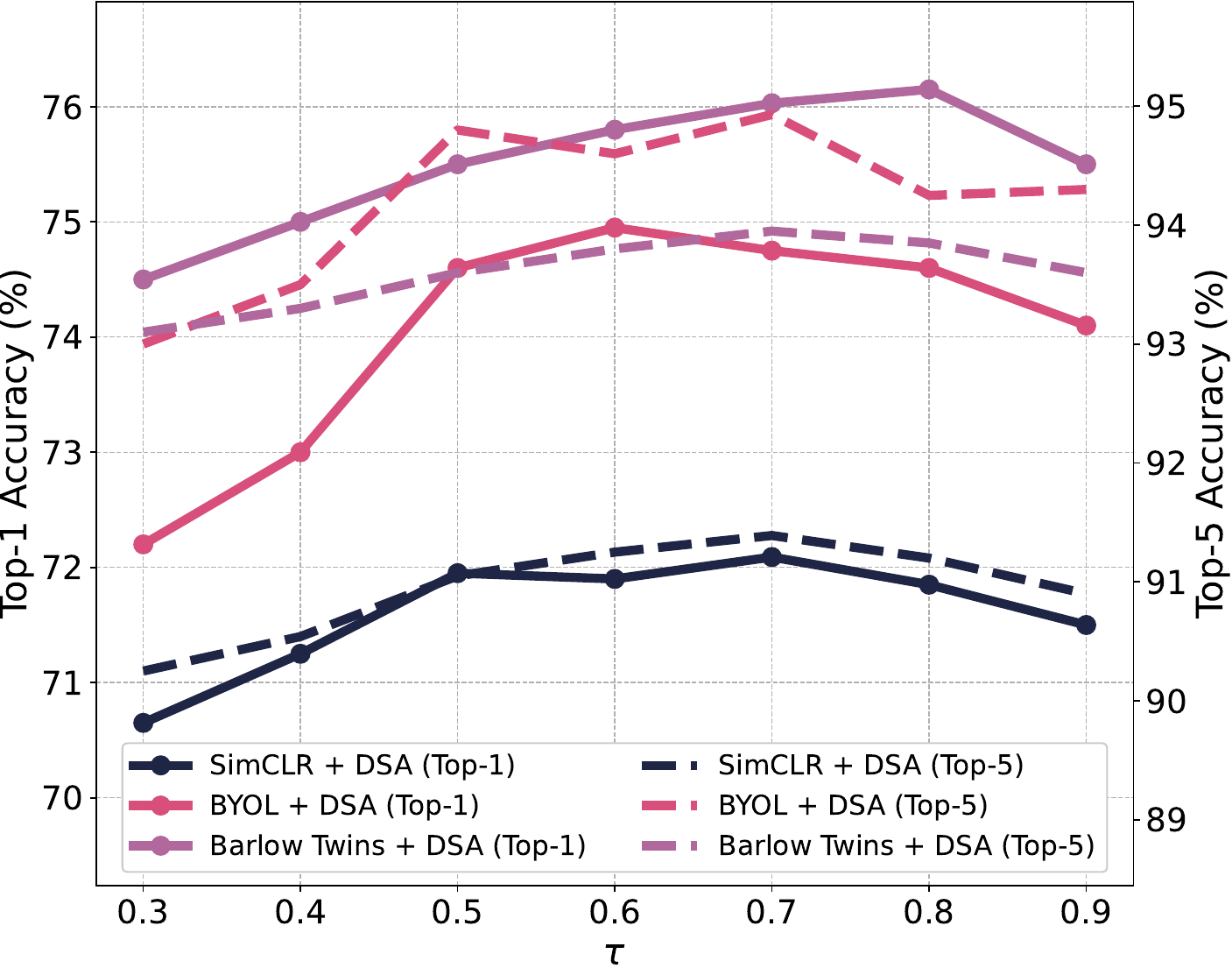}\label{fig:tau_imagenet}}

    \caption{ {ImageNet Top-1 and Top-5 linear accuracies under} different values of hyper-parameters $\nu$, $\upsilon$, $\alpha$, $\eta$, and $\tau$. The solid lines in figures (a)–(e) represent the Top-1 accuracies, while the dashed lines represent the Top-5 accuracies.}

    \label{fig:hyperparam_2}
\end{figure*}

\begin{table}[htb]
    \centering
    \caption{The inter-class distance, the intra-class variance, and the linear evaluation accuracy of SSL methods.}
    \begin{tabular}{lccc}
    \toprule
    \bf   Method  &\bf Inter-class Dist. ($\uparrow$) &\bf Intra-class Var. ($\downarrow$) \\
    \midrule
    SimCLR & 1.17 & 1.15  \\
    BYOL & 0.9 & 0.65  \\
    SwAV & 1.12 & 1.01  \\
    Barlow Twins & 1.06 & 1.11  \\
    MAE & 0.14 & 0.85  \\
    Supervised & 1.32 & 0.62 \\
    \midrule
    \rowcolor{orange!10} SimCLR + DSA &\bf 1.37 & 0.92  \\
    \rowcolor{orange!10} BYOL + DSA & 1.21 &\bf 0.55  \\
    \rowcolor{orange!10} SwAV + DSA & 1.31 & 0.57  \\
    \rowcolor{orange!10} Barlow Twins + DSA & 1.15 & 0.97  \\
    \rowcolor{orange!10} MAE + DSA & 0.64 & 0.65  \\
    \bottomrule
    \end{tabular}
    \label{tab:quantify_compare}
\end{table}

From the results in Table \ref{tab:voc_coco},  {we observe that} the inclusion of DSA significantly improves the AP in object detection and instance segmentation tasks. Specifically, in the VOC 07 detection task, the method with DSA shows an average AP improvement of 1.42\% compared to the SSL baseline. In the VOC 07+12 detection task, the average AP increases by 2.11\%. In the COCO detection task, the average AP increases by 1.65\%, and in the COCO instance segmentation task, the average AP increases by 1.78\%.

 {These results suggest that} incorporating DSA in the self-supervised pre-training process can significantly enhance its effectiveness in downstream detection and instance segmentation tasks.

\begin{table}[htb]
    \centering
    \caption{The inter-class distance, the intra-class variance, and the linear evaluation accuracy of SimCLR + DSA without individual components.}
    \resizebox{\linewidth}{!}{
    \begin{tabular}{lccc}
    \toprule
    \bf   Method  &\bf Inter-class Dist. ($\uparrow$) &\bf Intra-class Var. ($\downarrow$) &\bf ACC (\%) \\
    \midrule
    SimCLR & 1.17 & 1.15 & 70.15 \\
    ~ + DSA &\bf 1.37 &\bf 0.92 &\bf 72.09 \\
    ~ + DSA without SM & 1.24 & 1.06 & 71.06 \\
    ~ + DSA without $sc(\cdot)$ & 1.32 & 0.97 & 71.67 \\
    \midrule
    BYOL & 0.90 & 0.65 & 71.48 \\
    ~ + DSA &\bf 1.21 &\bf 0.55 &\bf 74.95 \\
    ~ + DSA without SM & 0.96 & 0.63 & 72.13 \\
    ~ + DSA without $sc(\cdot)$ & 1.16 & 0.52 & 73.74 \\
    \midrule
    SwAV & 1.12 & 1.01 & 75.78 \\
    ~ + DSA &\bf 1.31 &\bf 0.57 &\bf 77.84 \\
    ~ + DSA without SM & 1.17 & 0.92 & 76.06 \\
    ~ + DSA without $sc(\cdot)$ & 1.26 & 0.67 & 77.23 \\
    \bottomrule
    \end{tabular}
    }
    \label{tab:components}
\end{table}
\subsection{Discriminant analysis}

\paragraph{Visualization of SSL Features} 
To validate the effectiveness of DSA, we used the same experimental setup as in Section \ref{sec:moti}. We visualized the representations obtained from self-supervised learning with DSA on the ImageNet test set using t-SNE. The feature visualization is shown in Figure \ref{fig:tsne_better}. 

From the figure, it can be observed that after incorporating DSA, there are clear boundaries between clusters of different categories, and the distribution within the same category is more concentrated. Specifically, the t-SNE map of MAE + DSA in Figure \ref{fig:tsne_mae_better} shows clustering characteristics compared to MAE in Figure \ref{fig:tsne_mae}.  This indicates that the DSA algorithm effectively reduces intra-class variance and increases inter-class distance. In the next section, we further demonstrate this through quantitative experiments. 

\paragraph{Quantified Evidence of Discrimination}
We further conducted a quantitative analysis of DSA's intra-class variance and inter-class distance using the same methods as in Section \ref{sec:moti}. The results are presented in Table \ref{tab:quantify_compare}.
From the table, it can be observed that for all SSL methods, the inclusion of DSA results in a reduction in intra-class variance and an increase in inter-class distance. This demonstrates the effectiveness of DSA. Additionally, as evidenced by the performance across various downstream tasks discussed in Section \ref{sec:downstream}, reducing intra-class variance and increasing inter-class distance indeed enhance generalization. This validates the correctness of our theoretical analysis.

\subsection{Ablation Study}
\label{subsec:ablation}

\paragraph{Analysis of Hyperparameters} 

{We explore the effects of hyperparameters $\nu$, $\upsilon$, $\eta$, $\alpha$, and $\tau$ by analyzing how they influence the mean inter-class distance and intra-class variance. Specifically, $\nu$ controls the weight of the semantic arranging loss $\mathcal{L}_{AM}^s$, while $\upsilon$ weights the constraint loss $\mathcal{L}_{con}$. Both $\nu$ and $\upsilon$ do not have fixed natural ranges, so we adopt logarithmic scaling to test scenarios ranging from nearly disabling these components to allowing them to dominate the optimization process. The parameter $\alpha$ in $\mathcal{L}_{AM}$ determines whether the term $2M_{i,j} - \alpha$ is positive (causing attractive forces) or negative (causing repulsive forces), with a meaningful range strictly between 0 and 2, since $M_{i,j}$ is normalized between 0 and 1. We choose the tested range of $\alpha$ between [0.1,1.5]. The hyperparameter $\eta$ represents the number of nearest neighbors considered when computing the connectivity score $sc(z_i)$, balancing the sensitivity of local structure estimation and the risk of mistakenly treating outliers as normal samples, and is chosen in the range of 5 to 25 for computational efficiency and stability. Finally, $\tau$ in Equation (11) serves as a temperature parameter that scales the distance term $\Vert z_i - z_j \Vert_2^2$ within the loss function, where lower values amplify both the attractive force between similar samples and the repulsive force between dissimilar samples, while $\tau=1$ applies no scaling; based on observed stability, we constrain $\tau$ to the range [0.3, 1].}

The results are illustrated in Figure \ref{fig:hyperparam}. For results in Figures \ref{fig:nu_acc} to \ref{fig:eta_acc}, the results are obtained by fixing other hyperparameters while changing the corresponding hyperparameter. 

From Figures \ref{fig:nu_acc} and \ref{fig:upsilon_acc}, we can observe that in the objective function in Equation \ref{edfy}, the optimal weights for $\mathcal{L}_{AM}^s$ and $\mathcal{L}_{con}$ are achieved at $\nu=1\times 10^{-1}$ and $\upsilon=1\times 10^2$, respectively.
{From Figure~\ref{fig:alpha_acc}, we observe that as $\alpha$ increases, the inter-class distance gradually rises and stabilizes around $\alpha=1$, while the intra-class variance continuously decreases and also plateaus within this range. However, further increases in $\alpha$ cause the inter-class distance to decline and the intra-class variance to rise again. These results indicate that the threshold $\alpha$ achieves optimal performance when it effectively balances both attractive and repulsive forces.}
From Figure \ref{fig:eta_acc}, we observe that as $\eta$ increases, both the inter-class distance and intra-class variance experience significant growth from $\eta=5$ to $\eta=10$ and then level off. This suggests that when the number of nearest neighbor samples is small, the score $sc(z_i)$ obtained is not accurate. However, as $\eta$ increases, the score $sc(z_i)$ becomes more accurate but stabilizes once $\eta$ reaches a certain value. The default value for $\eta$ is set to 20. {From Figure \ref{fig:tau_acc}, we observe that as $\tau$ increases, the inter-class distance first increases, reaches an optimum around $\tau=0.7$, and then decreases. Meanwhile, the intra-class variance first decreases, reaches its minimum around $\tau=0.7$, and then increases again. Therefore, the default value of $\tau$ is set to $0.7$.}

{
To further evaluate the robustness of DSA across different hyperparameter settings, we perform ImageNet linear evaluation under varying hyperparameter values, as shown in Figure~\ref{fig:hyperparam_2}.
}

{Specifically, in Figure~\ref{fig:nu_imagenet}, although the optimal Top-5 accuracy for Barlow Twins + DSA and the optimal Top-1 accuracy for SimCLR + DSA are both achieved when $\nu=1$, we observed that, on average, all methods achieve relatively higher Top-1 and Top-5 accuracies when $\nu$ is set to 0.1. Furthermore, across the entire range of $\nu$, the DSA-enhanced models consistently outperform the baseline SSL methods.}
{In Figure~\ref{fig:upsilon_imagenet}, the best performance on ImageNet is achieved when $\upsilon$ ranged between 10 and 100.}
{As shown in Figure~\ref{fig:alpha_imagenet}, we test $\alpha$ values in the range from 0.1 to 1.5 and observe that performance declines when $\alpha$ is either too large (favoring only attractive forces) or too small (favoring only repulsive forces). The best results are obtained when $\alpha$ is in the range of 0.8 to 1.1, where attractive and repulsive forces are balanced.}
{In Figure~\ref{fig:eta_imagenet}, the performance trends on ImageNet are generally consistent with the observation in Figure~\ref{fig:eta_acc}.}
{For Figure~\ref{fig:tau_imagenet}, we observe that lower $\tau$ values lead to poorer performance, while optimal results are achieved when $\tau$ is set between 0.6 and 0.8.}

{
These results collectively indicate that although the exact optimal hyperparameter configurations vary across different SSL models, model performance remains relatively stable within certain hyperparameter ranges. Therefore, for fair comparisons across models, we selected a fixed set of hyperparameters that falls within these stable regions. Additionally, we consistently found that under various hyperparameter configurations, models with DSA outperformed their baseline counterparts. We have also provided results on Kinetics-400 in Figure~\ref{fig:hyperparam_3}, FC100 (5-way 1-shot and 5-way 5-shot) in Figure~\ref{fig:hyperparam_4}, as well as HMDB-51 and UCF-101 in Figure~\ref{fig:hyperparam_5} under different hyperparameter settings in \ref{app:ablation}. The results show similar trends, further demonstrating the robustness and generalizability of our approach.
}

\paragraph{Ablation Study on Components of DSA} 
We present an ablation study comparing different components of DSA, i.e. the SM and the scoring mechanism $sc(\cdot)$, in Table \ref{tab:components}. 

Table \ref{tab:components} shows inter-class distance, intra-class variance, and top-1 test accuracy on ImageNet for DSA without each individual component.
From Table \ref{tab:components}, it can be seen that without incorporating SM, the inter-class distance is slightly higher than that of the original SSL method, and the final accuracy is also marginally higher than the original SSL method. This indicates that SM  {plays a crucial role in enhancing the performance of} DSA. Moreover, even without $sc(\cdot)$, the performance is already significantly better than the original SSL method, though still lower than the full DSA method. This demonstrates that the weights obtained by $sc$ play an important auxiliary role in the DSA method.

\begin{table}[ht]
\centering
\caption{
Top-1 accuracies on ImageNet test set for different selection of $f_{aux}$ and $f_s$}
\resizebox{\linewidth}{!}{
\begin{tabular}{lcccc}
\toprule
\multirow{2}{*}{\textbf{Network}} & \multirow{2}{*}{\textbf{Selection}} & \multicolumn{3}{c}{\textbf{Top-1 ACC (\%)}} \\
\cmidrule{3-5}
& & \textbf{SimCLR + DSA} & \textbf{BYOL + DSA} & \textbf{Barlow Twins + DSA} \\
\midrule
\multirow{4}{*}{$f_{aux}$} & CLIP & 72.09 & 74.95 & 76.03 \\ 
 & BEiT & 72.04 & 74.87 & 75.94 \\ 
 & SiFT & 71.98 & 74.82 & 75.91 \\ 
 & ResNet-18 & 72.06 & 74.49 & 76.04 \\ 
\midrule
\multirow{2}{*}{$f_s$} & MLP & 72.09 & 74.95 & 76.03 \\ 
 & Cosine Similarity & 70.96 & 73.39 & 75.88 \\ 
\bottomrule
\end{tabular}
}
\label{tab:f_aux_fs}
\end{table}

\paragraph{The Selection of $f_{aux}$ and $f_s$} 
To evaluate the impact of different choices of $f_{aux}$ and $f_s$, we conduct ablation studies on the ImageNet test set. The default feature extractor is the pre-trained CLIP image encoder \cite{radford2021learning}. We evaluate the selection of $f_{aux}$ with three other different feature extractors, namely open-source pre-trained ResNet-18, Scale-invariant feature transform (SIFT) \cite{lowe2004distinctive}, and BEiT \cite{bao2021beit}. The default setting for $f_s$ is a two-layer MLP with ReLU activation function, we conduct experiments by substituting the network $f_s$ with cosine similarity, which yields $M_{i,j} = \text{sim}(z_i,z_j) = \frac{{z_i \cdot z_j^T}}{{\Vert z_i \Vert \Vert z_j \Vert}}$. All results are depicted in Table \ref{tab:f_aux_fs}. From the results of $f_{aux}$, we can conclude that the performance of DSA is not impacted by the choice of pre-trained feature extractor (i.e., the architecture, pre-training method, or data), while the use of a pre-trained feature extractor is essential. The results of $f_s$ demonstrate that the performance of cosine similarity deteriorates significantly compared to learnable MLP, highlighting the efficacy of $f_s$.

{
\begin{table}[htb]
    \centering
    \caption{{Component-wise runtime breakdown per training iteration during ImageNet pretraining. Experiments were conducted on NVIDIA A100 GPUs with a per-GPU batch size of 512 and image size 224$\times$224. Each value represents the average runtime measured over 50 iterations.}}
    {
    \begin{tabular}{l r}
    \toprule
      Components & Runtime per iteration  \\
    \midrule
    \multicolumn{2}{c}{Forward Pass} \\
    \midrule
    Compute $Z_{tr}^{aug}$ & 498 ms \\
    Compute $M$ & 23 ms \\
    Compute $M^{pro}$ & 69 ms \\
    Compute $sc$ & 48 ms \\
    Compute $\mathcal{L}_{ssl}$ & 15 ms \\
    Compute $\mathcal{L}_{AM}^s$ & 7 ms \\
    Compute $\mathcal{L}_{con}$ & 8 ms \\
    \midrule
    \multicolumn{2}{c}{Backward Pass} \\
    \midrule
    $\mathcal{L}_{ssl}$ & 359 ms \\
    $\mathcal{L}_{DSA}^s$ & 387 ms \\
    \midrule
    \multicolumn{2}{c}{Total time Comparison} \\
    \midrule
    SimCLR & 872 ms \\
    SimCLR + DSA & 1054 ms \\
    \bottomrule
    \end{tabular}
    }
    \label{tab:component_runtime_updated}
\end{table}
}

\begin{figure}
    \centering
    \includegraphics[width=0.7\linewidth]{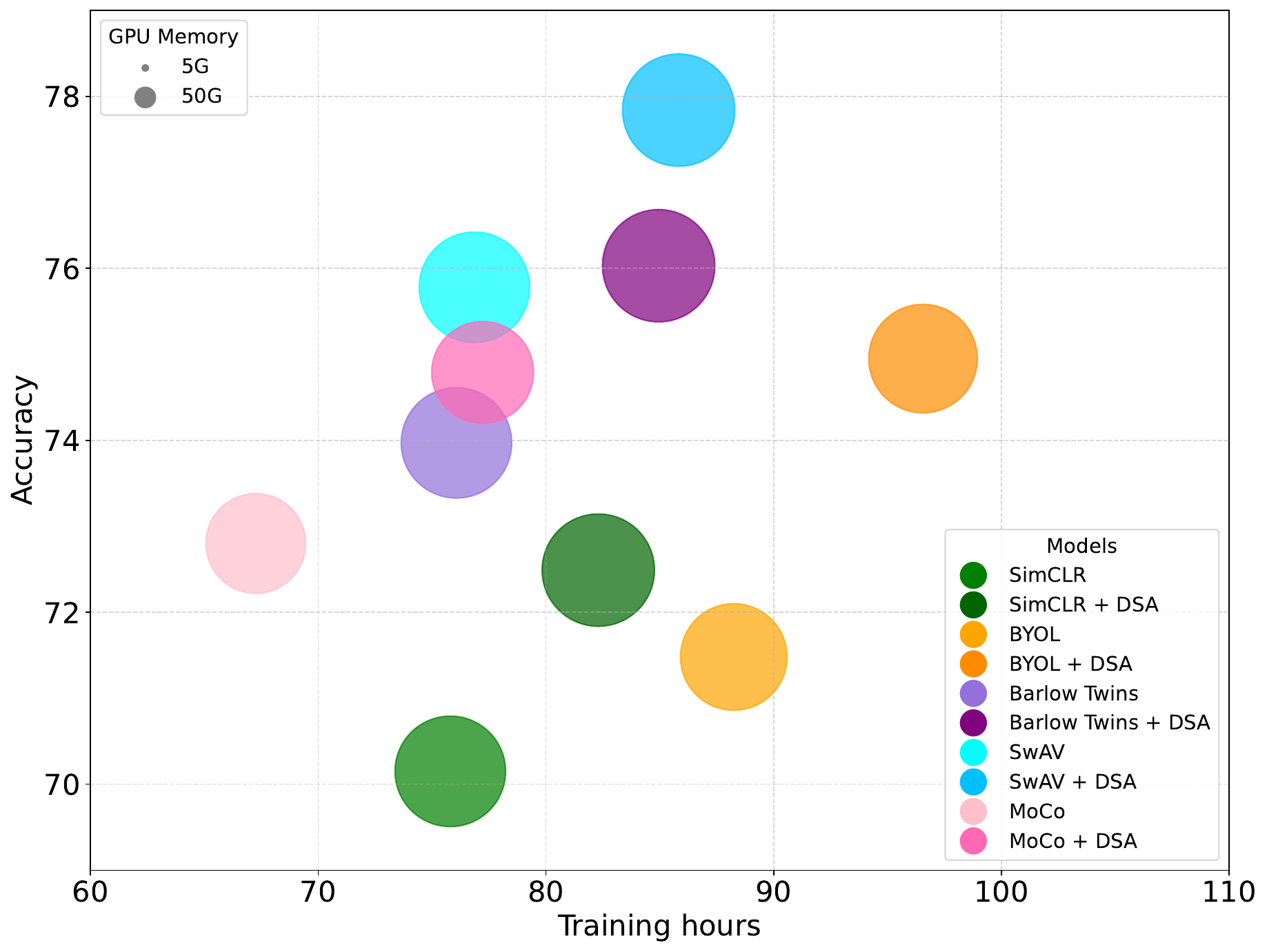}
    \caption{Relationship between training time, accuracy, and memory footprint across various SSL methods and their DSA-enhanced versions.}
    \label{fig:ablation_resource}
\end{figure}

\paragraph{{Computation Complexity Analysis}} 

{
Table~\ref{tab:component_runtime_updated} reports the detailed runtime per training iteration for each component of the proposed DSA framework during ImageNet pretraining. All experiments were conducted on NVIDIA A100 GPUs with a per-GPU batch size of 512 and an image size of 224×224, with runtimes averaged over 50 iterations.
In the forward pass, most of the computation time is spent generating the augmented features $Z_{tr}^{aug}$, which takes 498 ms per iteration. Additional operations introduced by DSA, such as computing the similarity matrices $M$ and $M^{pro}$ and the connection score $sc$, contribute further overhead but remain moderate, with $M^{pro}$ and $sc$ requiring 69 ms and 48 ms per iteration, respectively. Calculating the losses $\mathcal{L}_{ssl}$, $\mathcal{L}_{AM}^s$, and $\mathcal{L}_{con}$ in the forward pass collectively adds another 30 ms.
In the backward pass, the baseline SimCLR loss $\mathcal{L}_{ssl}$ requires 359 ms per iteration. The backward time for SimCLR + DSA, captured as $\mathcal{L}_{DSA}^s$, is 387 ms, indicating only 28 ms of additional computation over the baseline. Overall, compared to SimCLR's runtime of 872 ms per iteration, incorporating DSA results in a total runtime of 1054 ms per iteration.}

{Furthermore, Figure~\ref{fig:ablation_resource} visualizes the trade-offs among computational cost, memory consumption, and performance across various SSL methods and their DSA-enhanced versions. The figure demonstrates that although integrating DSA leads to longer total training times and slightly higher memory usage, it consistently improves accuracy across all baseline methods.}

{These results confirm that while DSA introduces some additional computational and memory overhead, the increase remains modest and manageable for large-scale pretraining scenarios, offering favorable trade-offs for the observed performance gains.}

\section{Conclusion}
\label{sec:conclu}
In this paper, we present empirical evidence of the existence of the crowding problem in self-supervised learning and provide theoretical proof that minimizing intra-class variance and maximizing inter-class separation effectively compresses the upper bound of classification error. Based on these findings, we propose a novel approach called ``Dynamic Semantic Adjuster" (DSA), which aims to attract samples with similar semantics while repelling others. We evaluate the effectiveness of our proposed methods through experiments in comprehensive downstream tasks. We also provide visualization and quantified evidence to analyze the discriminability of the learned representation. The experimental results and ablation study demonstrate the effectiveness of our proposed DSA.

\section{{Limitation and Future Work}}
\label{sec:limit}
{A notable limitation of our proposed DSA lies in its sensitivity to hyperparameter choices. Although our ablation study reveals broad stable regions where DSA consistently improves performance, suboptimal configurations may still lead to degraded semantic separation. Furthermore, in our current experiments, we used a unified set of DSA hyperparameters across all SSL models to ensure fair comparisons and to demonstrate the general applicability of the method. We acknowledge that model-specific tuning could further optimize performance for individual SSL backbones. Future work will explore automatic strategies for balancing the arrangement and scoring objectives, as well as adaptive hyperparameter tuning for different SSL architectures, potentially via meta-gradient methods or adaptive scheduling.}

{Our approach holds potential for extension to broader application domains beyond the current setting.}  
The structure-aware regularization inherent in our approach is well-suited to capture geometric consistency .  
Moreover, its ability to model discriminative features may also benefit the alignment of  {complex patterns in diverse data types}.  
In future work, we will explore these  {potential directions} to further validate the generality of our method.

\newpage
\bibliographystyle{elsarticle-num} 
\bibliography{main}

\begin{thebibliography}{10}
\expandafter\ifx\csname url\endcsname\relax
  \def\url#1{\texttt{#1}}\fi
\expandafter\ifx\csname urlprefix\endcsname\relax\def\urlprefix{URL }\fi
\expandafter\ifx\csname href\endcsname\relax
  \def\href#1#2{#2} \def\path#1{#1}\fi

\bibitem{IS_SSL1}
X.~Chen, X.~Zheng, K.~Sun, W.~Liu, Y.~Zhang, Self-supervised vision transformer-based few-shot learning for facial expression recognition, Information Sciences 634 (2023) 206--226.

\bibitem{IS_SSL2}
Y.~Cai, Z.~Zhang, P.~Ghamisi, B.~Rasti, X.~Liu, Z.~Cai, Transformer-based contrastive prototypical clustering for multimodal remote sensing data, Information Sciences 649 (2023) 119655.

\bibitem{grillBootstrapYourOwn2020}
J.-B. Grill, F.~Strub, F.~Altch{\'e}, C.~Tallec, P.~Richemond, E.~Buchatskaya, C.~Doersch, B.~Avila~Pires, Z.~Guo, M.~Gheshlaghi~Azar, B.~Piot, k.~{kavukcuoglu}, R.~Munos, M.~Valko, Bootstrap {{Your Own Latent}} - {{A New Approach}} to {{Self-Supervised Learning}}, in: Advances in {{Neural Information Processing Systems}}, Vol.~33, {Curran Associates, Inc.}, 2020, pp. 21271--21284.

\bibitem{chenExploringSimpleSiamese2021}
X.~Chen, K.~He, Exploring {{Simple Siamese Representation Learning}}, in: Proceedings of the {{IEEE}}/{{CVF Conference}} on {{Computer Vision}} and {{Pattern Recognition}}, 2021, pp. 15750--15758.

\bibitem{qiang2021robust}
W.~Qiang, J.~Li, C.~Zheng, B.~Su, H.~Xiong, Robust local preserving and global aligning network for adversarial domain adaptation, IEEE Transactions on Knowledge and Data Engineering (2021).

\bibitem{gui2024survey}
J.~Gui, T.~Chen, J.~Zhang, Q.~Cao, Z.~Sun, H.~Luo, D.~Tao, A survey on self-supervised learning: Algorithms, applications, and future trends, IEEE Transactions on Pattern Analysis and Machine Intelligence (2024).

\bibitem{chenSimpleFrameworkContrastive2020}
T.~Chen, S.~Kornblith, M.~Norouzi, G.~Hinton, A {{Simple Framework}} for {{Contrastive Learning}} of {{Visual Representations}}, in: Proceedings of the 37th {{International Conference}} on {{Machine Learning}}, {PMLR}, 2020, pp. 1597--1607.

\bibitem{zbontarBarlowTwinsSelfSupervised2021}
J.~Zbontar, L.~Jing, I.~Misra, Y.~LeCun, S.~Deny, Barlow {{Twins}}: {{Self-Supervised Learning}} via {{Redundancy Reduction}}, in: Proceedings of the 38th {{International Conference}} on {{Machine Learning}}, {PMLR}, 2021, pp. 12310--12320.

\bibitem{caronUnsupervisedLearningVisual2020}
M.~Caron, I.~Misra, J.~Mairal, P.~Goyal, P.~Bojanowski, A.~Joulin, Unsupervised learning of visual features by contrasting cluster assignments, Advances in Neural Information Processing Systems 33 (2020) 9912--9924.

\bibitem{he2022masked}
K.~He, X.~Chen, S.~Xie, Y.~Li, P.~Doll{\'a}r, R.~Girshick, Masked autoencoders are scalable vision learners, in: Proceedings of the IEEE/CVF conference on computer vision and pattern recognition, 2022, pp. 16000--16009.

\bibitem{dengImageNetLargescaleHierarchical2009}
J.~Deng, W.~Dong, R.~Socher, L.-J. Li, K.~Li, L.~{Fei-Fei}, {{ImageNet}}: {{A}} large-scale hierarchical image database, in: 2009 {{IEEE Conference}} on {{Computer Vision}} and {{Pattern Recognition}}, 2009, pp. 248--255.
\newblock \href {https://doi.org/10.1109/CVPR.2009.5206848} {\path{doi:10.1109/CVPR.2009.5206848}}.

\bibitem{caron2021emerging}
M.~Caron, H.~Touvron, I.~Misra, H.~J{\'e}gou, J.~Mairal, P.~Bojanowski, A.~Joulin, Emerging properties in self-supervised vision transformers, in: Proceedings of the IEEE/CVF international conference on computer vision, 2021, pp. 9650--9660.

\bibitem{ermolov2021whitening}
A.~Ermolov, A.~Siarohin, E.~Sangineto, N.~Sebe, Whitening for self-supervised representation learning, in: International Conference on Machine Learning, PMLR, 2021, pp. 3015--3024.

\bibitem{Vicregl}
A.~Bardes, J.~Ponce, Y.~LeCun, Vicregl: Self-supervised learning of local visual features, arXiv preprint arXiv:2210.01571 (2022).

\bibitem{bao2021beit}
H.~Bao, L.~Dong, S.~Piao, F.~Wei, Beit: Bert pre-training of image transformers, arXiv preprint arXiv:2106.08254 (2021).

\bibitem{balakrishnama1998linear}
S.~Balakrishnama, A.~Ganapathiraju, Linear discriminant analysis-a brief tutorial, Institute for Signal and information Processing 18~(1998) (1998) 1--8.

\bibitem{saunshiTheoreticalAnalysisContrastive2019}
N.~Saunshi, O.~Plevrakis, S.~Arora, M.~Khodak, H.~Khandeparkar, A {{Theoretical Analysis}} of {{Contrastive Unsupervised Representation Learning}}, in: Proceedings of the 36th {{International Conference}} on {{Machine Learning}}, {PMLR}, 2019, pp. 5628--5637.

\bibitem{chen2022learning}
S.~Chen, C.~Gong, J.~Li, J.~Yang, G.~Niu, M.~Sugiyama, Learning contrastive embedding in low-dimensional space, Advances in Neural Information Processing Systems 35 (2022) 6345--6357.

\bibitem{bishop2006pattern}
C.~M. Bishop, Pattern recognition and machine learning, Springer google schola 2 (2006) 1122--1128.

\bibitem{liPrototypicalContrastiveLearning2021}
J.~Li, P.~Zhou, C.~Xiong, S.~C.~H. Hoi, Prototypical {{Contrastive Learning}} of {{Unsupervised Representations}} (Mar. 2021).
\newblock \href {http://arxiv.org/abs/2005.04966} {\path{arXiv:2005.04966}}, \href {https://doi.org/10.48550/arXiv.2005.04966} {\path{doi:10.48550/arXiv.2005.04966}}.

\bibitem{van2008visualizing}
L.~Van~der Maaten, G.~Hinton, Visualizing data using t-sne., Journal of machine learning research 9~(11) (2008).

\bibitem{mcinnes2018umap}
L.~McInnes, J.~Healy, J.~Melville, Umap: Uniform manifold approximation and projection for dimension reduction, arXiv preprint arXiv:1802.03426 (2018).

\bibitem{jolliffe2002principal}
I.~T. Jolliffe, Principal component analysis for special types of data, Springer, 2002.

\bibitem{kay2017kinetics}
W.~Kay, J.~Carreira, K.~Simonyan, B.~Zhang, C.~Hillier, S.~Vijayanarasimhan, F.~Viola, T.~Green, T.~Back, P.~Natsev, et~al., The kinetics human action video dataset, arXiv preprint arXiv:1705.06950 (2017).

\bibitem{lin2014microsoft}
T.-Y. Lin, M.~Maire, S.~Belongie, J.~Hays, P.~Perona, D.~Ramanan, P.~Doll{\'a}r, C.~L. Zitnick, Microsoft coco: Common objects in context, in: European conference on computer vision, Springer, 2014, pp. 740--755.

\bibitem{everingham2010pascal}
M.~Everingham, L.~Van~Gool, C.~K. Williams, J.~Winn, A.~Zisserman, The pascal visual object classes (voc) challenge, International journal of computer vision 88~(2) (2010) 303--338.

\bibitem{fc100}
B.~Oreshkin, P.~Rodr{\'\i}guez~L{\'o}pez, A.~Lacoste, Tadam: Task dependent adaptive metric for improved few-shot learning, Advances in neural information processing systems 31 (2018).

\bibitem{cub200}
C.~Wah, S.~Branson, P.~Welinder, P.~Perona, S.~Belongie, The caltech-ucsd birds-200-2011 dataset (2011).

\bibitem{plantdisease}
S.~P. Mohanty, D.~P. Hughes, M.~Salath{\'e}, Using deep learning for image-based plant disease detection, Frontiers in plant science 7 (2016) 1419.

\bibitem{krizhevskyLearningMultipleLayers2009}
A.~Krizhevsky, G.~Hinton, Learning multiple layers of features from tiny images (2009).

\bibitem{feichtenhoferLargeScaleStudyUnsupervised2021}
C.~Feichtenhofer, H.~Fan, B.~Xiong, R.~Girshick, K.~He, A {{Large-Scale Study}} on {{Unsupervised Spatiotemporal Representation Learning}} (Apr. 2021).
\newblock \href {http://arxiv.org/abs/2104.14558} {\path{arXiv:2104.14558}}, \href {https://doi.org/10.48550/arXiv.2104.14558} {\path{doi:10.48550/arXiv.2104.14558}}.

\bibitem{heMomentumContrastUnsupervised2020}
K.~He, H.~Fan, Y.~Wu, S.~Xie, R.~Girshick, Momentum {{Contrast}} for {{Unsupervised Visual Representation Learning}}, in: Proceedings of the {{IEEE}}/{{CVF Conference}} on {{Computer Vision}} and {{Pattern Recognition}}, 2020, pp. 9729--9738.

\bibitem{RELIC-v2}
N.~Tomasev, I.~Bica, B.~McWilliams, L.~Buesing, R.~Pascanu, C.~Blundell, J.~Mitrovic, Pushing the limits of self-supervised resnets: Can we outperform supervised learning without labels on imagenet?, arXiv preprint arXiv:2201.05119 (2022).

\bibitem{LMLC}
S.~Chen, G.~Niu, C.~Gong, J.~Li, J.~Yang, M.~Sugiyama, Large-margin contrastive learning with distance polarization regularizer, in: International Conference on Machine Learning, PMLR, 2021, pp. 1673--1683.

\bibitem{ressl}
M.~Zheng, S.~You, F.~Wang, C.~Qian, C.~Zhang, X.~Wang, C.~Xu, Ressl: Relational self-supervised learning with weak augmentation, Advances in Neural Information Processing Systems 34 (2021) 2543--2555.

\bibitem{ssl-hsic}
Y.~Li, R.~Pogodin, D.~J. Sutherland, A.~Gretton, Self-supervised learning with kernel dependence maximization, Advances in Neural Information Processing Systems 34 (2021) 15543--15556.

\bibitem{CorInfoMax}
S.~Ozsoy, S.~Hamdan, S.~Arik, D.~Yuret, A.~Erdogan, Self-supervised learning with an information maximization criterion, Advances in Neural Information Processing Systems 35 (2022) 35240--35253.

\bibitem{MEC}
X.~Liu, Z.~Wang, Y.-L. Li, S.~Wang, Self-supervised learning via maximum entropy coding, Advances in Neural Information Processing Systems 35 (2022) 34091--34105.

\bibitem{zhang2022mask}
Q.~Zhang, Y.~Wang, Y.~Wang, How mask matters: Towards theoretical understandings of masked autoencoders, Advances in Neural Information Processing Systems 35 (2022) 27127--27139.

\bibitem{chen2021empirical}
X.~Chen, S.~Xie, K.~He, An empirical study of training self-supervised vision transformers, in: Proceedings of the IEEE/CVF international conference on computer vision, 2021, pp. 9640--9649.

\bibitem{yangVideoRepresentationLearning2020}
C.~Yang, Y.~Xu, B.~Dai, B.~Zhou, Video {{Representation Learning}} with {{Visual Tempo Consistency}} (Dec. 2020).
\newblock \href {http://arxiv.org/abs/2006.15489} {\path{arXiv:2006.15489}}.

\bibitem{dave_tclr_2022}
I.~Dave, R.~Gupta, M.~N. Rizve, M.~Shah, {TCLR}: {Temporal} {Contrastive} {Learning} for {Video} {Representation}, Computer Vision and Image Understanding 219 (2022) 103406, arXiv:2101.07974 [cs].
\newblock \href {https://doi.org/10.1016/j.cviu.2022.103406} {\path{doi:10.1016/j.cviu.2022.103406}}.

\bibitem{videomoco}
T.~Pan, Y.~Song, T.~Yang, W.~Jiang, W.~Liu, Videomoco: Contrastive video representation learning with temporally adversarial examples, in: Proceedings of the IEEE/CVF conference on computer vision and pattern recognition, 2021, pp. 11205--11214.

\bibitem{khorasgani2022slic}
S.~H. Khorasgani, Y.~Chen, F.~Shkurti, Slic: Self-supervised learning with iterative clustering for human action videos, in: Proceedings of the IEEE/CVF Conference on Computer Vision and Pattern Recognition, 2022, pp. 16091--16101.

\bibitem{xiao2022maclr}
F.~Xiao, J.~Tighe, D.~Modolo, Maclr: Motion-aware contrastive learning of representations for videos, in: European Conference on Computer Vision, Springer, 2022, pp. 353--370.

\bibitem{qianSpatiotemporalContrastiveVideo2021}
R.~Qian, T.~Meng, B.~Gong, M.-H. Yang, H.~Wang, S.~Belongie, Y.~Cui, Spatiotemporal {{Contrastive Video Representation Learning}} (Apr. 2021).
\newblock \href {http://arxiv.org/abs/2008.03800} {\path{arXiv:2008.03800}}, \href {https://doi.org/10.48550/arXiv.2008.03800} {\path{doi:10.48550/arXiv.2008.03800}}.

\bibitem{ren2015faster}
S.~Ren, K.~He, R.~Girshick, J.~Sun, Faster r-cnn: Towards real-time object detection with region proposal networks, Advances in neural information processing systems 28 (2015).

\bibitem{he2017mask}
K.~He, G.~Gkioxari, P.~Doll{\'a}r, R.~Girshick, Mask r-cnn, in: Proceedings of the IEEE international conference on computer vision, 2017, pp. 2961--2969.

\bibitem{radford2021learning}
A.~Radford, J.~W. Kim, C.~Hallacy, A.~Ramesh, G.~Goh, S.~Agarwal, G.~Sastry, A.~Askell, P.~Mishkin, J.~Clark, et~al., Learning transferable visual models from natural language supervision, in: International Conference on Machine Learning, PMLR, 2021, pp. 8748--8763.

\bibitem{lowe2004distinctive}
D.~G. Lowe, Distinctive image features from scale-invariant keypoints, International journal of computer vision 60 (2004) 91--110.

\end{thebibliography}

\newpage
\appendix
\renewcommand{\thetheorem}{}
\section{Proof of Theorem \ref{fghjk}}
\label{app:theorem1}
\begin{theorem}
{\rm{\textbf{IV.1.}}}
    If Assumption \textbf{IV.1.} holds, then, for any $f \in \mathcal{F}$, $\mathcal{L}_{CE}^\mu \left( f \right)$ can be bounded by $\mathcal{L}_{\rm NCE}^\mu \left( f \right)$ as:
\begin{equation}
        \begin{array}{l}
\mathcal{L}_{CE}^\mu \left( f \right)  \le {\mathcal{L}_{{\rm{NCE}}}}\left( f \right) - const + \sum\limits_{i = 1,j = 1i \ne j}^K {\mu _i^{\rm{T}}{\mu _j}} \\
 \quad\quad\quad \quad \quad+ \sqrt {{\rm{Var}}\left( {f\left( x \right)\left| y \right.} \right)}  + O\left( {cons{t^{{{ - 1} \mathord{\left/
 {\vphantom {{ - 1} 2}} \right.
 \kern-\nulldelimiterspace} 2}}}} \right)
\end{array}
    \end{equation}
where const is a constant that is related to the number of negative samples, ${\mu _i} = {\mathbbm{E}_{p\left( {x\left| i \right.} \right)}}\left[ {f\left( x \right)} \right]$, and ${\rm{Var}}\left( {f\left( x \right)\left| y \right.} \right) = {\mathbbm{E}_{p\left( y \right)}}\left[ {{\mathbbm{E}_{p\left( {x\left| y \right.} \right)}}{{\left\| {f\left( x \right) - {\mathbbm{E}_{p\left( {x\left| y \right.} \right)}}f\left( x \right)} \right\|}^2}} \right]$.
\end{theorem}
\begin{proof} \label{fghasdk}
Suppose that $Q$ represents the number of negative samples. Then, we have:
\begin{align}
    &\quad \mathbbm{E}_{p(x_i,x_j)}[ \log\frac{1}{Q}\sum_{i=1}^Q\exp(f^{\rm{T}}(x_i)f(x_j))  \nonumber\\
    &\quad\quad\quad\quad\quad\quad\quad\quad\quad - \log\mathbbm{E}_{p(x_i)}\exp(f^{\rm{T}}(x_i)f(x_j)] \nonumber\\
    &\leq e\mathbbm{E}_{p(x_i,x_j)}[\frac{1}{Q}\sum_{i=1}^Q\exp(f^{\rm{T}}(x_i)f(x_j)) \nonumber\\
    &\quad\quad\quad\quad\quad\quad\quad\quad\quad -\mathbbm{E}_{p(x_i)}\exp(f^{\rm{T}}(x_i)f(x_j))] \nonumber\\
    &= \mathcal{O}\left( {{Q^{ - \frac{1}{2}}}} \right)
\end{align}
where the first inequality follows the Intermediate Value Theorem and $e$ (the natural number) is the upper bound of the absolute derivative of log between two points when $\left| {{f^{\rm{T}}}\left( {{x_i}} \right)f\left( {{x_j}} \right)} \right| \le 1$. The second inequality follows the Berry-Esseen Theorem given the bounded support of ${\exp \left( {{f^{\rm{T}}}\left( {{x_i}} \right)f\left( {{x_j}} \right)} \right)}$ as following: for i.i.d random variables ${\psi _i}$ with bounded support $supp\left( \psi  \right) \subset \left[ { - a,a} \right]$, zero mean and bounded variance $\sigma _\psi ^2 < {a^2}$, we have:
\begin{align}
    \mathbbm{E}\left[\frac{1}{Q}\sum_{i = 1}^Q \psi_i \right] &= \frac{\sigma_\psi}{\sqrt{Q}}\int_0^{\frac{a\sqrt{Q}}{\sigma_\psi}}p\left[\frac{1}{\sigma_\psi\sqrt{Q}}\sum_{i=1}^Q\psi_i > x\right]dx \nonumber\\
    &\leq \frac{\sigma_\psi}{\sqrt{Q}}\int_0^{\frac{a\sqrt{Q}}{\sigma_\psi}}p[\vert N(0,1)\vert>x]dx + \frac{C_a}{\sqrt{Q}}dx \nonumber\\
    &\leq \frac{C_a}{\sqrt{Q}} + \frac{a}{\sqrt{Q}}\mathbbm{E}[N(0,1)] \nonumber\\
    &= \mathcal{O}\left( {{Q^{ - \frac{1}{2}}}} \right)
\end{align}
where ${{C_a}}$ is the constant that only depends on $a$ and ${\psi _i} = \exp \left( {{f^{\rm{T}}}\left( {{x_i}} \right)f\left( {{x_j}} \right)} \right)$.

Then, we suppose that the classification task consists of $K$ categories and denote that ${u_y}$ as the center of features of class $y, y \in \left\{ {1,...,K} \right\}$. We denote ${p\left( {x,{x^ + },y} \right)}$ as the joint distribution of the positive pairs ${x,{x^ + },y}$, denote ${x^ + }$ as the random variable of positive sample, denote ${x^ +_i }$ as a negative sample, denote ${x^-}$ as the random variable of negative sample, and denote $y^-$ as the negative class. We have:
\begin{align}
    \mathcal{L}_{\rm{NCE}} &= -\mathbbm{E}_{p(x,x^+)}f(x)^{\rm{T}} f(x^+) \nonumber\\
                 &\quad + \mathbbm{E}_{p(x)}\mathbbm{E}_{p(x_i^-)}\log\sum_{i=1}^Q \exp(f(x)^{\rm{T}}f(x_i^-))  + \log Q\nonumber\\
                 &\geq -\mathbbm{E}_{p(x,x^+)}f(x)^{\rm{T}} f(x^+) \nonumber\\
                 &\quad + \mathbbm{E}_{p(x)}\log \frac{1}{Q}\mathbbm{E}_{p(x_i^-)}\exp(f(x)^{\rm{T}}f(x_i^-))   + \log Q\nonumber\\
                 &\geq -\mathbbm{E}_{p(x,x^+,y)}f(x)^{\rm{T}} f(x^+) - const(Q) + \log Q \nonumber\\
                 &\quad + \mathbbm{E}_{p(x)}\log \frac{1}{Q}\mathbbm{E}_{p(y^-)}\mathbbm{E}_{p(x_i^-|y^-)}\exp(f(x)^{\rm{T}}f(x_i^-)) \nonumber\\ 
                 &\geq -\mathbbm{E}_{p(x,x^+,y)}[f(x)^{\rm{T}} u_y + \Vert f(x^+) - u_y \Vert] \nonumber\\
                 &\quad + \mathbbm{E}_{p(x)}\log\mathbbm{E}_{p(y^-)}\exp(f(x)^{\rm{T}}u_{y^-}) \nonumber\\ 
                 &\quad - const(Q) + \log Q- \sum\limits_{i = 1,j = 1,i \ne j}^K {\mu _i^{\rm{T}}{\mu _j}}\nonumber\\
                 &\geq -\mathbbm{E}_{p(x,y)}f(x)^{\rm{T}} u_y - \sqrt{\mathbbm{E}_{p(x,y)}\Vert f(x) - u_y \Vert^2} \nonumber\\
                 &\quad + \mathbbm{E}_{p(x)}\log\mathbbm{E}_{p(y^-)}\exp(f(x)^{\rm{T}}u_{y^-}) \nonumber\\
                 &\quad - const(Q) + \log Q- \sum\limits_{i = 1,j = 1,i \ne j}^K {\mu _i^{\rm{T}}{\mu _j}} \nonumber\\
                 &= \mathbbm{E}_{p(x,y)}[-f(x)^{\rm{T}} u_y + \log\sum_{k=1}^K\exp(f(x)^{\rm{T}}u_k)] + \log M \nonumber\\
                 & -\sqrt{{\rm{Var}} (f(x)|y)}  - const(Q) - \sum\limits_{i = 1,j = 1,i \ne j}^K {\mu _i^{\rm{T}}{\mu _j}}\nonumber\\
                 &= \mathcal{L}_{CE}^{\mu}(f) -\sqrt{{\rm{Var}} (f(x)|y)}  - const(Q) \nonumber\\
                 &\quad + \log \frac{Q}{K}- \sum\limits_{i = 1,j = 1,i \ne j}^K {\mu _i^{\rm{T}}{\mu _j}}\nonumber\\
                 &= \mathcal{L}_{CE}^{\mu}(f) -\sqrt{{\rm{Var}} (f(x)|y)} - \sum\limits_{i = 1,j = 1,i \ne j}^K {\mu _i^{\rm{T}}{\mu _j}} \nonumber\\
                 &\quad + \mathcal{O}(const(Q)^{-\frac{1}{2}})
\end{align}
This ends the proof.

\end{proof}

\section{{Additional Results}}
\label{app:ablation}

\subsection{{Configuration of T-SNE Visualization}}
\label{app:config}
{To ensure fair and reliable visualization comparisons, we carefully tuned the hyperparameters of the t-SNE algorithm for each method. Specifically, for every SSL method and the supervised baseline, we conducted a grid search over perplexity values $\{10, 20, 30, 40, 50, 60, 70, 80, 90, 100\}$.}

{For each perplexity value, we ran t-SNE five times with different random seeds and computed two quantitative metrics:
\begin{itemize}
    \item \textbf{KL divergence}, which measures how well the low-dimensional embedding preserves the pairwise similarities of the original high-dimensional data.
    \item \textbf{Trustworthiness at 12 neighbors (T@12)}, a widely-used metric that quantifies the local neighborhood preservation in the lower-dimensional space.
\end{itemize}}

{We selected the perplexity value that achieved the highest T@12 score among those configurations whose KL divergence remained below $1.0$, ensuring both local neighborhood preservation and stable embedding quality. Table~\ref{tab:tsne_config} summarizes the selected perplexity values, as well as the corresponding averaged KL divergence and trustworthiness scores over five runs.}

\begin{table}[htb]
    \centering
    \caption{{Best configuration (perplexity) and corresponding KL divergence and trustworthiness (for 12 neighbors) in the t-SNE 2D visualization of each method, averaged over 5 runs.}}
    {\begin{tabular}{lccc}
    \toprule
    Method    &  Perplexity  &  KL-Divergence & T@12  \\
    \midrule
    SimCLR    &  50 & 0.94$\pm$0.03 & 0.91$\pm$0.016 \\
    BYOL      &  40 & 0.87$\pm$0.07 & 0.90$\pm$0.008 \\
    Barlow Twins & 30 & 0.78$\pm$0.08 & 0.91$\pm$0.012 \\
    SwAV      &  70 & 0.71$\pm$0.10 & 0.92$\pm$0.014 \\
    MAE       &  80 & 0.68$\pm$0.13 & 0.93$\pm$0.007 \\ 
    Supervised & 60 & 0.58$\pm$0.02 & 0.94$\pm$0.008 \\
    \midrule
    SimCLR + DSA   &  40 & 0.90$\pm$0.24 & 0.93$\pm$0.023 \\
    BYOL + DSA     &  50 & 0.83$\pm$0.11 & 0.92$\pm$0.012 \\
    Barlow Twins + DSA & 70 & 0.74$\pm$0.05 & 0.92$\pm$0.008 \\
    SwAV + DSA      &  60 & 0.76$\pm$0.12 & 0.91$\pm$0.011 \\
    MAE + DSA      &  60 & 0.65$\pm$0.06 & 0.94$\pm$0.012 \\ 
    \bottomrule
    \end{tabular}}
    \label{tab:tsne_config}
\end{table}

\subsection{{Additional Visualization Results}}
\label{app:more_results}
{To further verify the effectiveness and generality of our method, we conducted additional visualization experiments using 3D t-SNE, PCA, and UMAP, complementing the 2D t-SNE results reported in Figure~\ref{fig:tsne}.}

{Specifically, we visualized feature embeddings from all methods (both with and without DSA) on the same subset of eight randomly selected classes from the ImageNet test set. All features were $\ell_2$-normalized prior to dimensionality reduction to ensure consistency across methods. } 

{For the 3D t-SNE visualizations, we performed a grid search over perplexity values in $\{10, 20, 30, 40, 50, 60, 70, 80, 90, 100\}$ for each method, and selected the value yielding the lowest KL divergence as the best configuration. We also repeated each t-SNE run five times and reported the average KL divergence and the trustworthiness score at 12 neighbors ($T@12$) to quantify the embedding stability, as summarized in Table~\ref{tab:tsne_config_3d}.}  

{In addition, we generated PCA visualizations by projecting the features onto the first two principal components, and we visualized the same features with UMAP using $n_\text{neighbors}=15$ and a minimum distance of 0.1. The results for PCA and UMAP are shown in Figures~\ref{fig:pca}--\ref{fig:pca_better} and Figures~\ref{fig:umap_simclr}--\ref{fig:umap_mae_better}, respectively.  }

{Across all three visualization techniques, we consistently observed that the feature distributions of SSL methods without DSA exhibit larger intra-class variance and greater overlap between different classes, which reflects the crowding problem. In contrast, the SSL methods integrated with DSA demonstrate significantly more compact intra-class clusters and clearer inter-class separation in the embedding space. These qualitative observations align well with our quantitative measurements of intra-class variance and inter-class distance presented earlier in Table~\ref{tab:quantify} and Table~\ref{tab:quantify_compare}.}

\begin{table}[htb]
    \centering
    \caption{{Best configuration (perplexity) and corresponding KL divergence and trustworthiness (for 12 neighbors) in the t-SNE 3D visualization of each method, averaged over 5 runs.}}
    {\begin{tabular}{lccc}
    \toprule
    Method    &  Perplexity  &  KL-Divergence & T@12  \\
    \midrule
    SimCLR    &  30 & 0.81$\pm$0.23 & 0.95$\pm$0.015 \\
    BYOL      &  50 & 0.91$\pm$0.13 & 0.91$\pm$0.008 \\
    Barlow Twins & 40 & 0.95$\pm$0.08 & 0.90$\pm$0.005 \\
    SwAV      &  60 & 0.83$\pm$0.12 & 0.92$\pm$0.015 \\
    MAE       &  80 & 0.92$\pm$0.05 & 0.91$\pm$0.011 \\ 
    Supervised & 50 & 0.69$\pm$0.02 & 0.90$\pm$0.008 \\
    \midrule
    SimCLR + DSA   &  40 & 0.91$\pm$0.23 & 0.91$\pm$0.015 \\
    BYOL + DSA     &  20 & 0.89$\pm$0.13 & 0.90$\pm$0.008 \\
    Barlow Twins + DSA & 20 & 0.85$\pm$0.08 & 0.92$\pm$0.005 \\
    SwAV + DSA      &  20 & 0.93$\pm$0.12 & 0.93$\pm$0.015 \\
    MAE + DSA      &  80 & 0.82$\pm$0.05 & 0.93$\pm$0.011 \\ 
    \bottomrule
    \end{tabular}}
    \label{tab:tsne_config_3d}
\end{table}

\begin{figure*}[t]
     \centering
     \subfigure[SimCLR]{\includegraphics[width=0.3\textwidth]{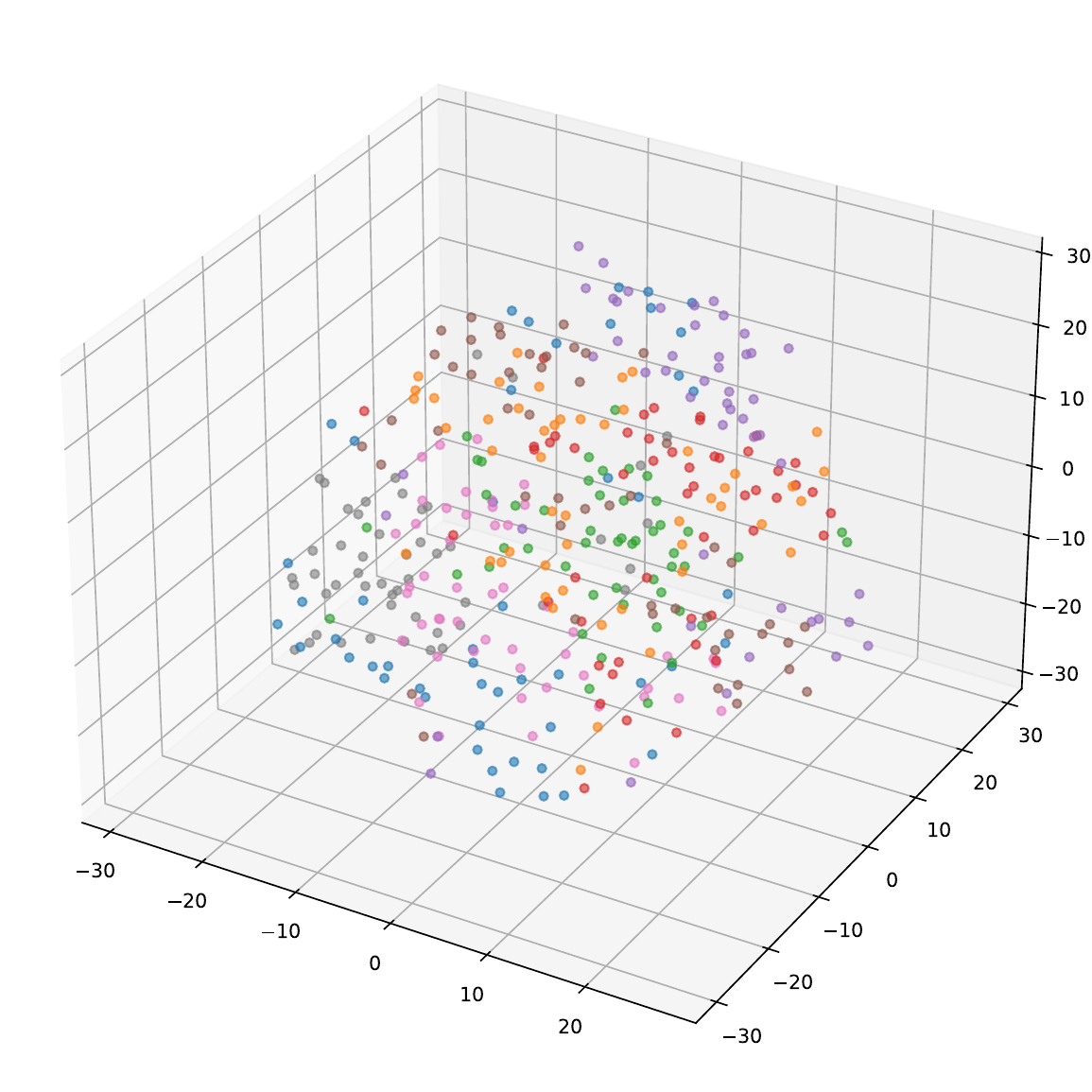}\label{fig:tsne_3d_simclr}}
     \subfigure[BYOL]{\includegraphics[width=0.3\textwidth]{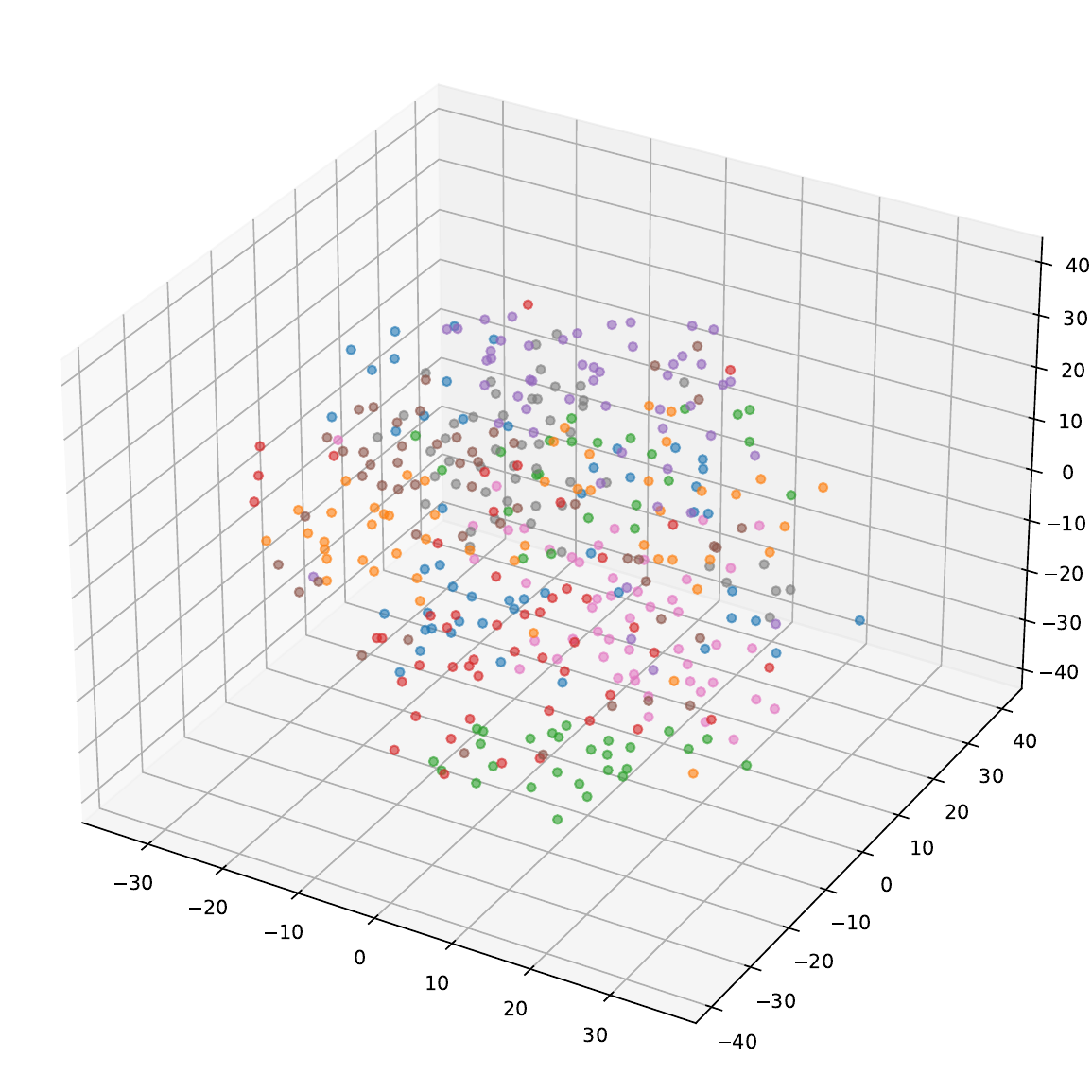}\label{fig:tsne_3d_byol}}
     \subfigure[Barlow Twins]{\includegraphics[width=0.3\textwidth]{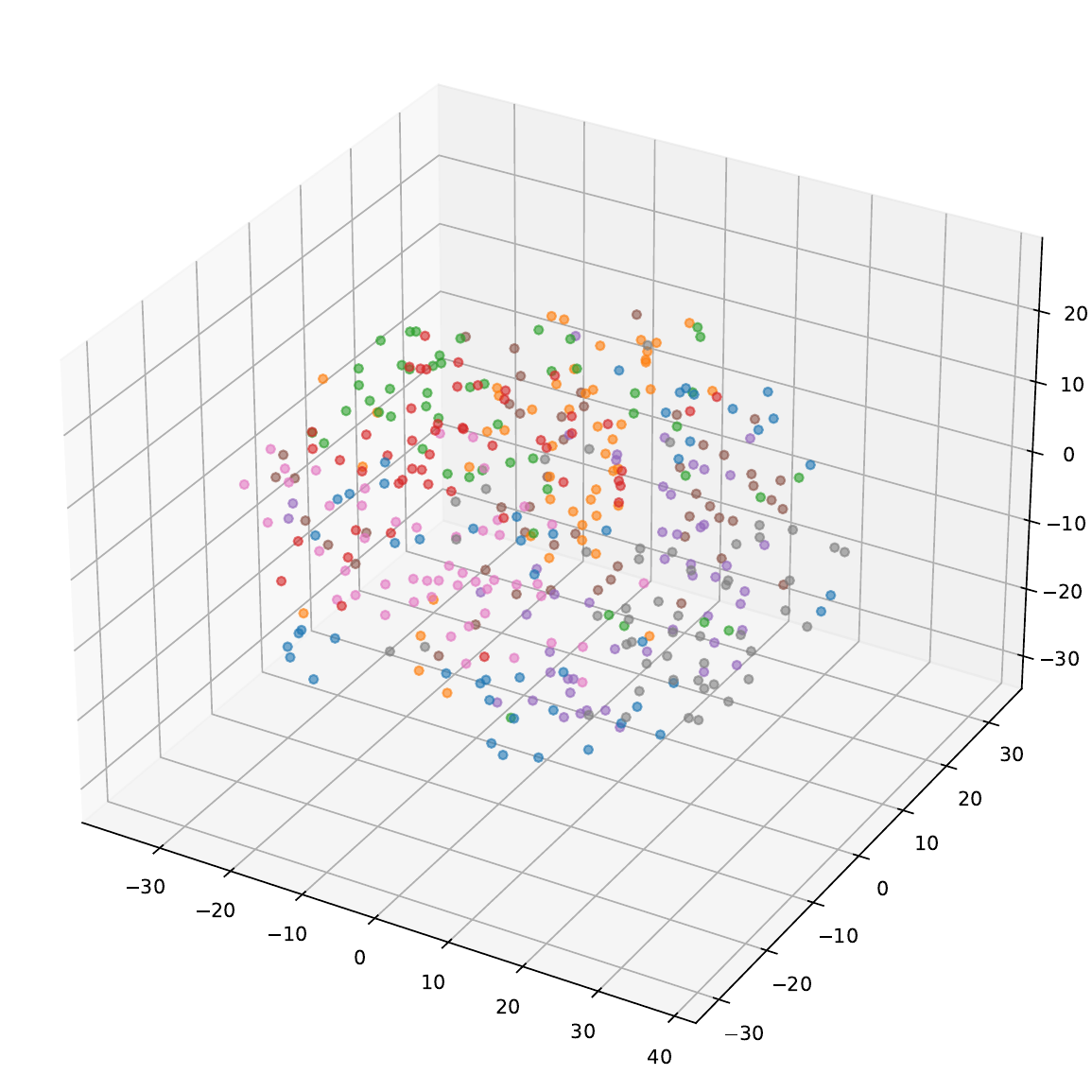}\label{fig:tsne_3d_barlow}}
     \subfigure[SwAV]{\includegraphics[width=0.3\textwidth]{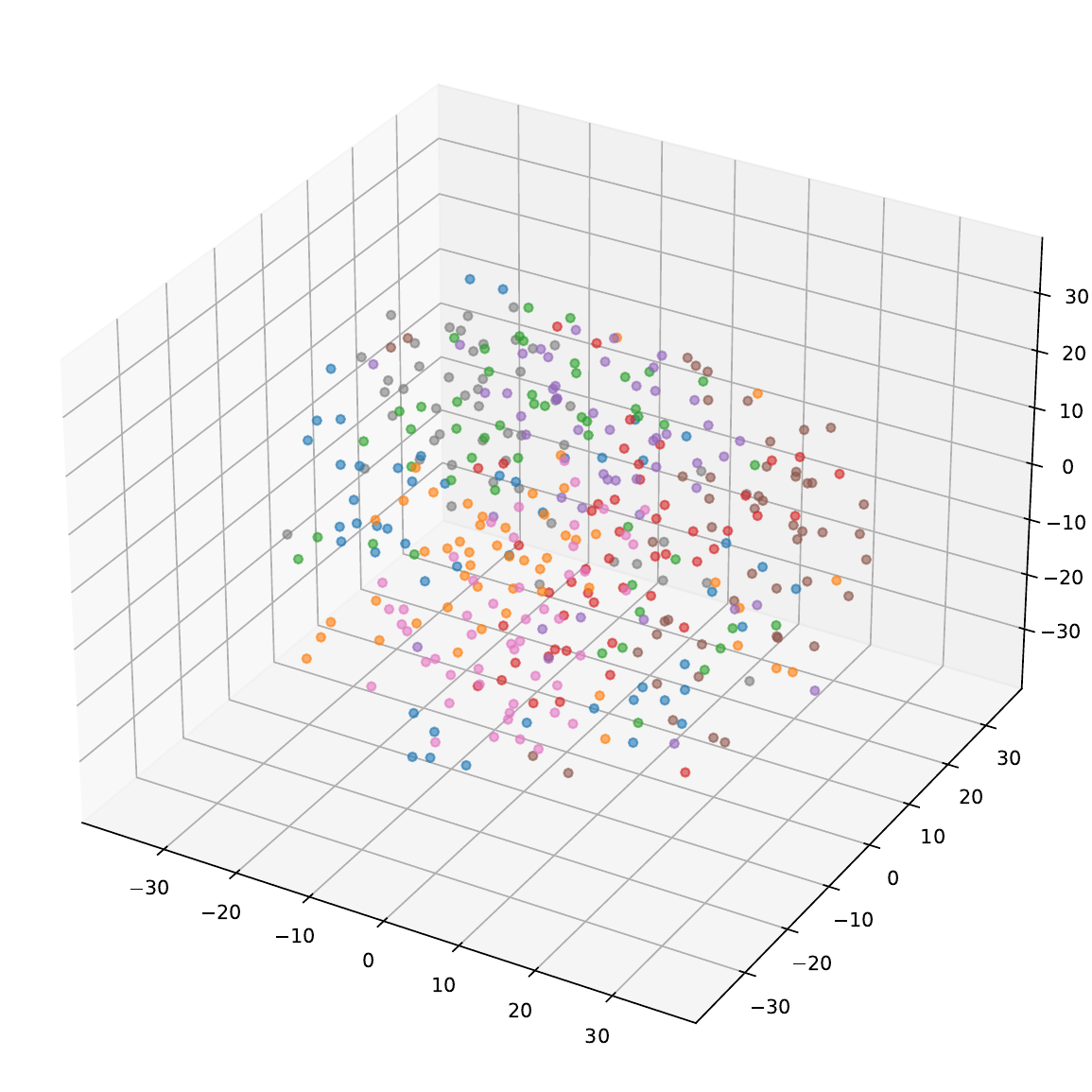}\label{fig:tsne_3d_swav}}
     \subfigure[MAE]{\includegraphics[width=0.3\textwidth]{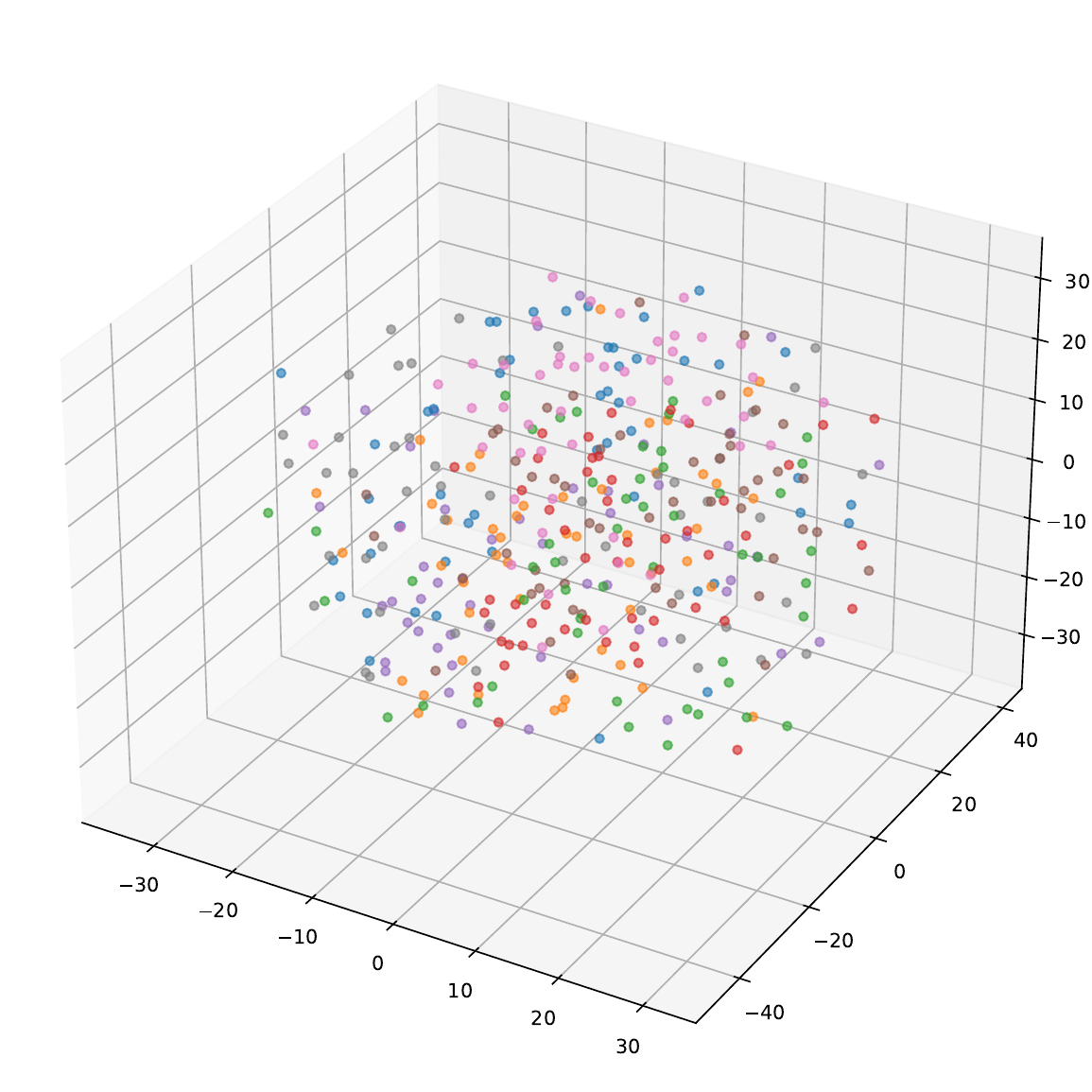}\label{fig:tsne_3d_mae}}
     \subfigure[Supervised]{\includegraphics[width=0.3\textwidth]{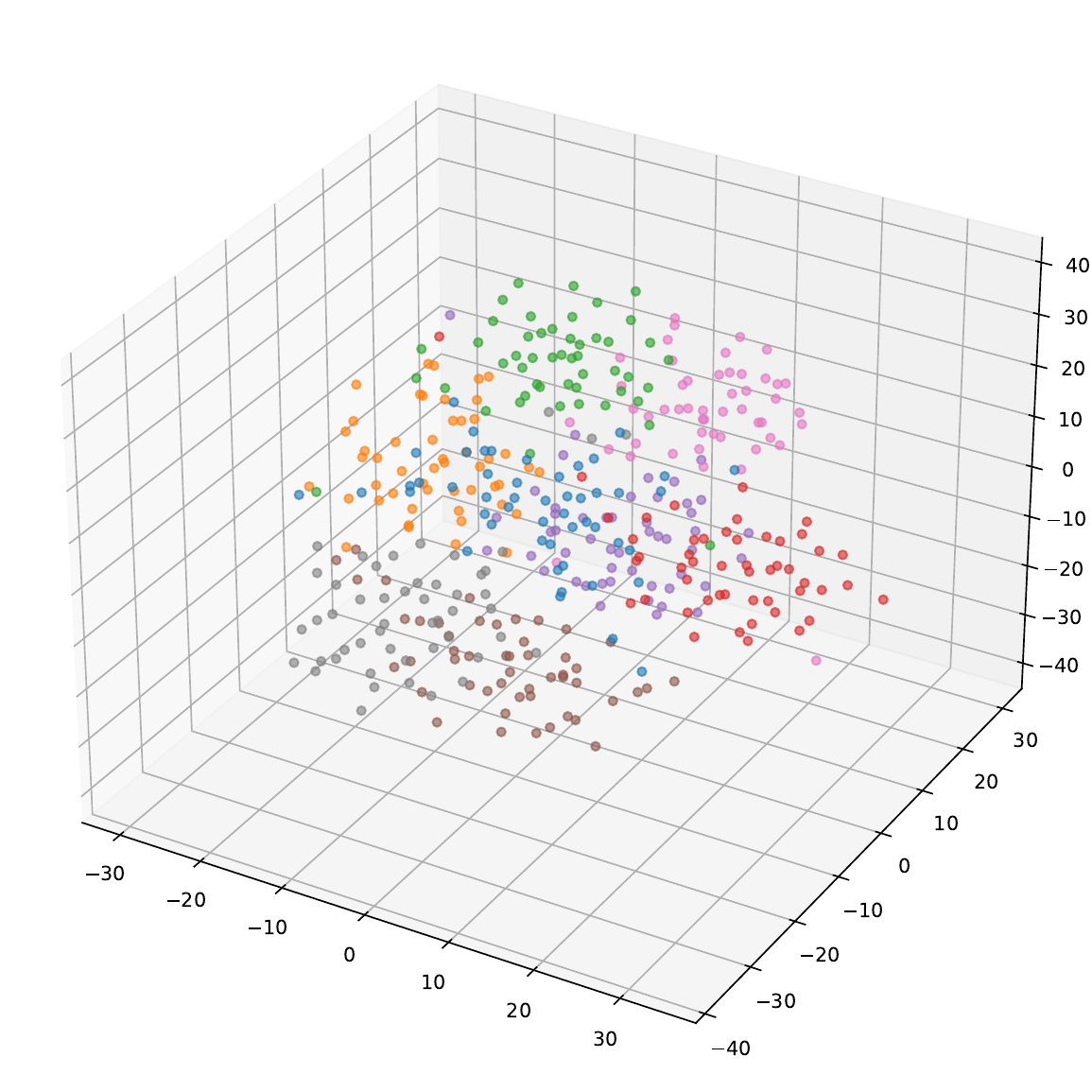}\label{fig:tsne_3d_supervised}}
    \caption{{Data distribution visualization with 3D t-SNE based on 8 random classes of the test set of ImageNet in the feature space. (a) - (e) corresponds to the visualization results of the self-supervised method, while (f) corresponds to the visualization results of the supervised method. }}
    \label{fig:tsne_3d}
\end{figure*}

\begin{figure*}[htb]
    \centering
    \subfigure[SimCLR + DSA]{\includegraphics[width=0.3\textwidth]{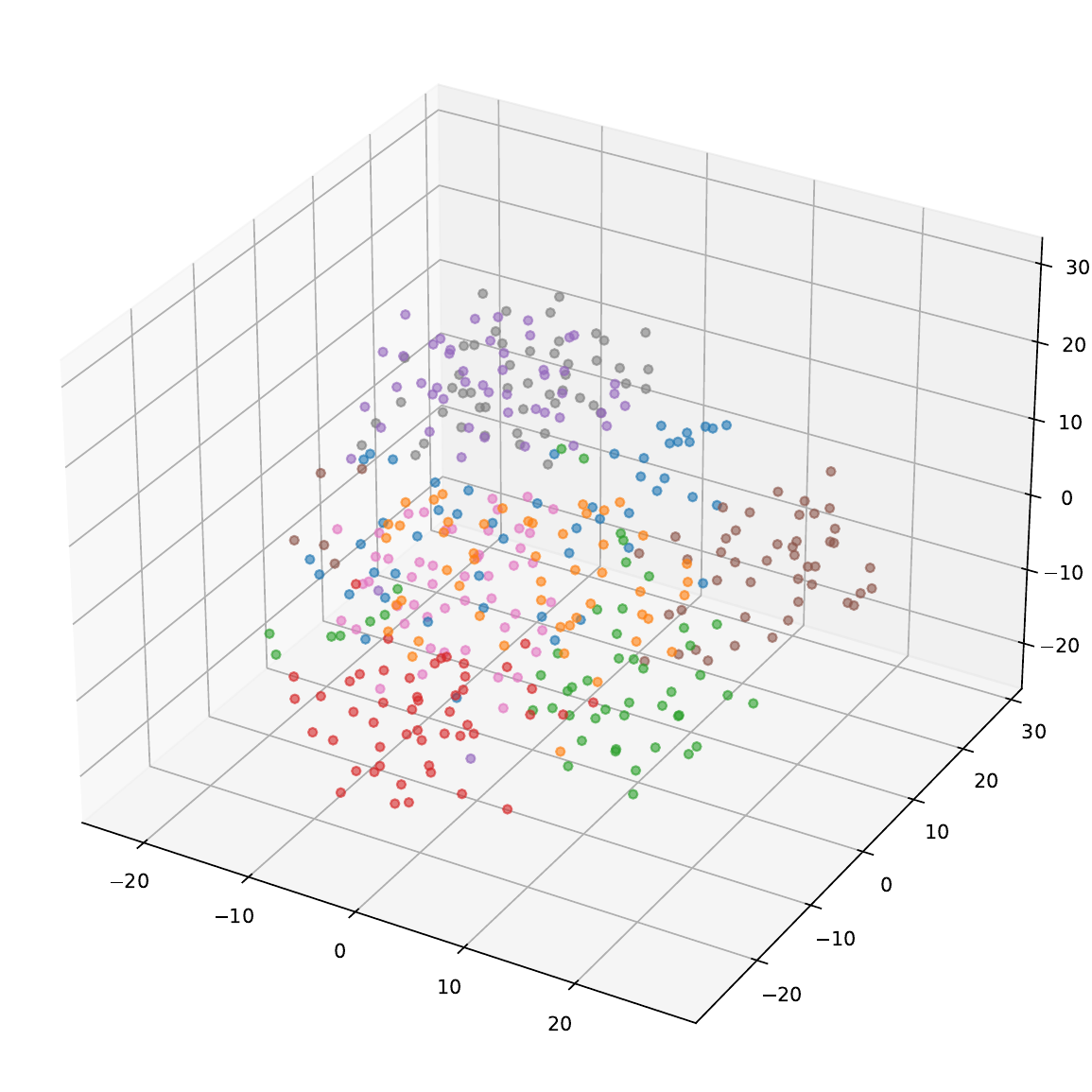}\label{fig:tsne_3d_simclr_better}}
     \subfigure[BYOL + DSA]{\includegraphics[width=0.3\textwidth]{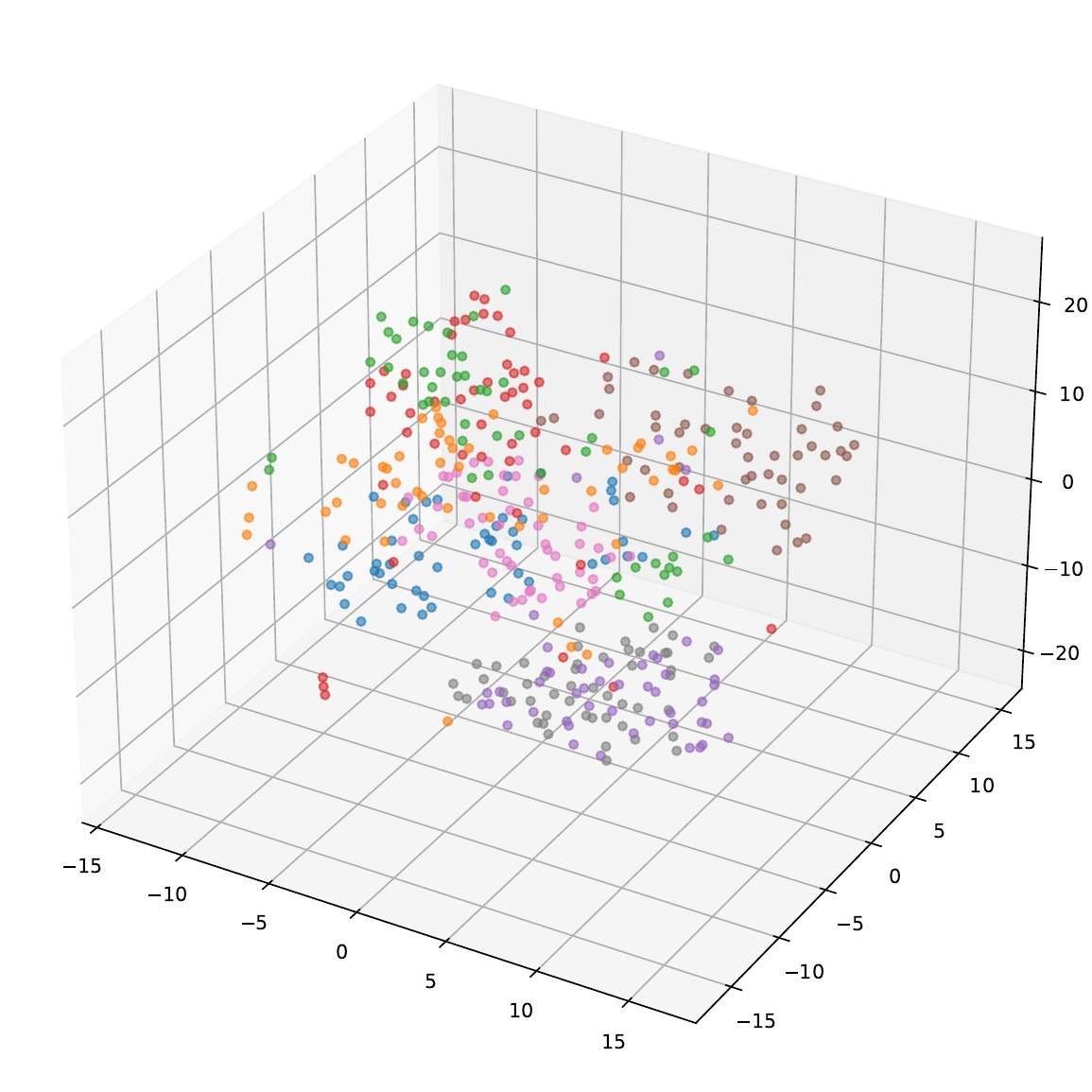}\label{fig:tsne_3d_byol_better}}
     \subfigure[Barlow Twins + DSA]{\includegraphics[width=0.3\textwidth]{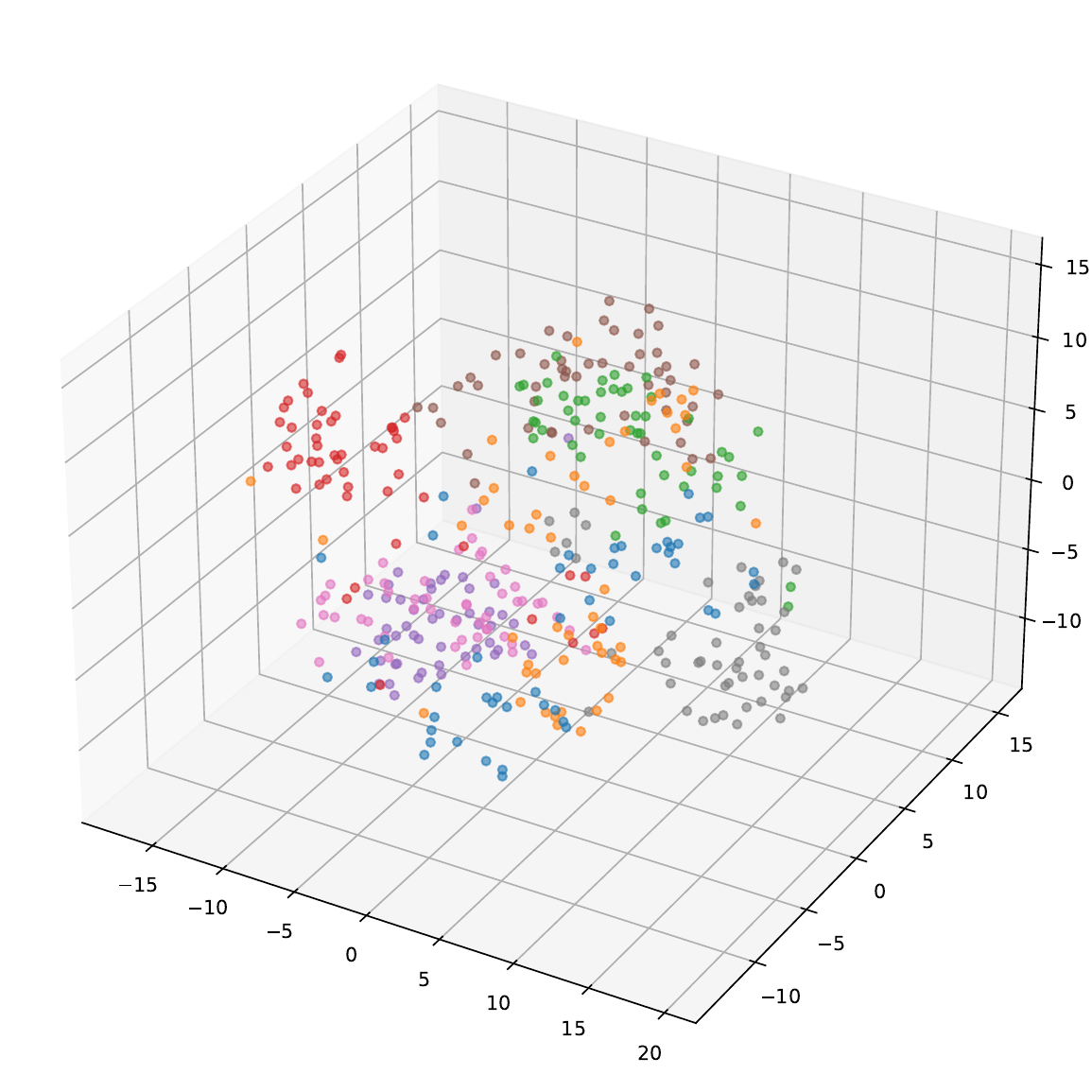}\label{fig:tsne_3d_barlow_better}}
     \subfigure[SwAV + DSA]{\includegraphics[width=0.3\textwidth]{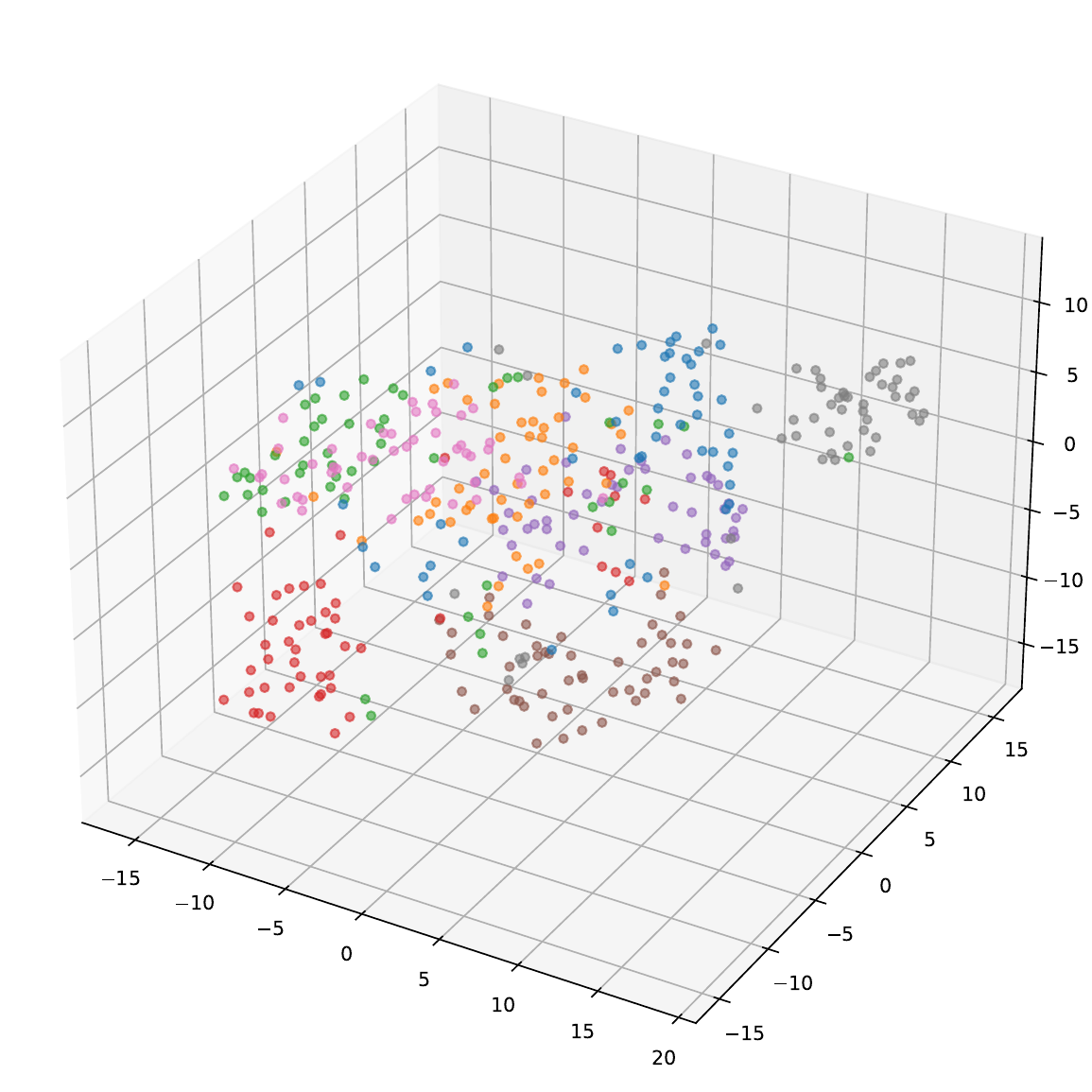}\label{fig:tsne_3d_swav_better}}
     \subfigure[MAE + DSA]{\includegraphics[width=0.3\textwidth]{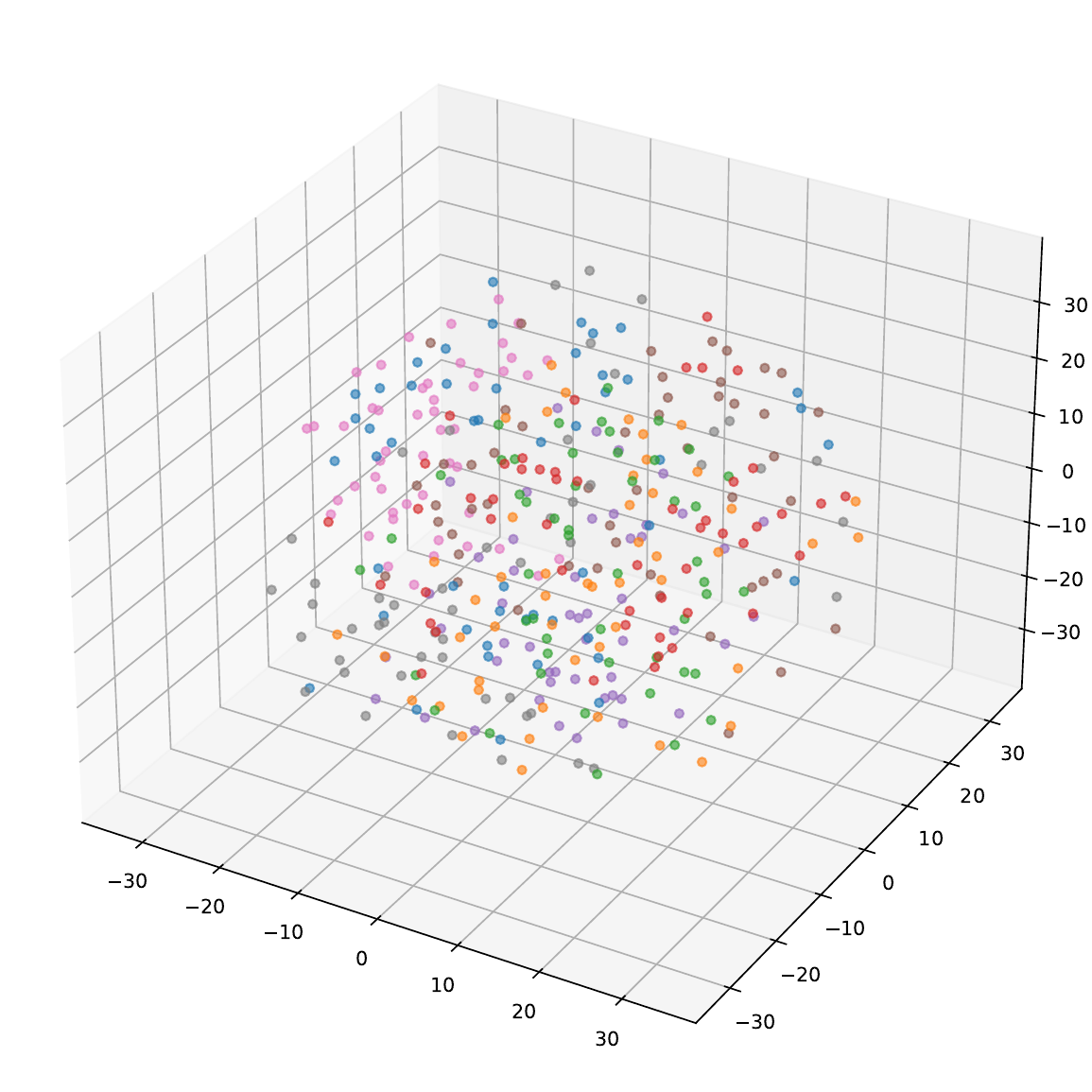}\label{fig:tsne_3d_mae_better}}
    \caption{{Data distribution visualization with 3D t-SNE based on the same set of classes as Figure \ref{fig:tsne} of the test set of ImageNet in the feature space. (a) - (e) corresponds to the visualization results of the self-supervised methods integrated with DSA. }}
    \label{fig:tsne_3d_better}
\end{figure*}

\begin{figure*}[t]
     \centering
     \subfigure[SimCLR]{\includegraphics[width=0.3\textwidth]{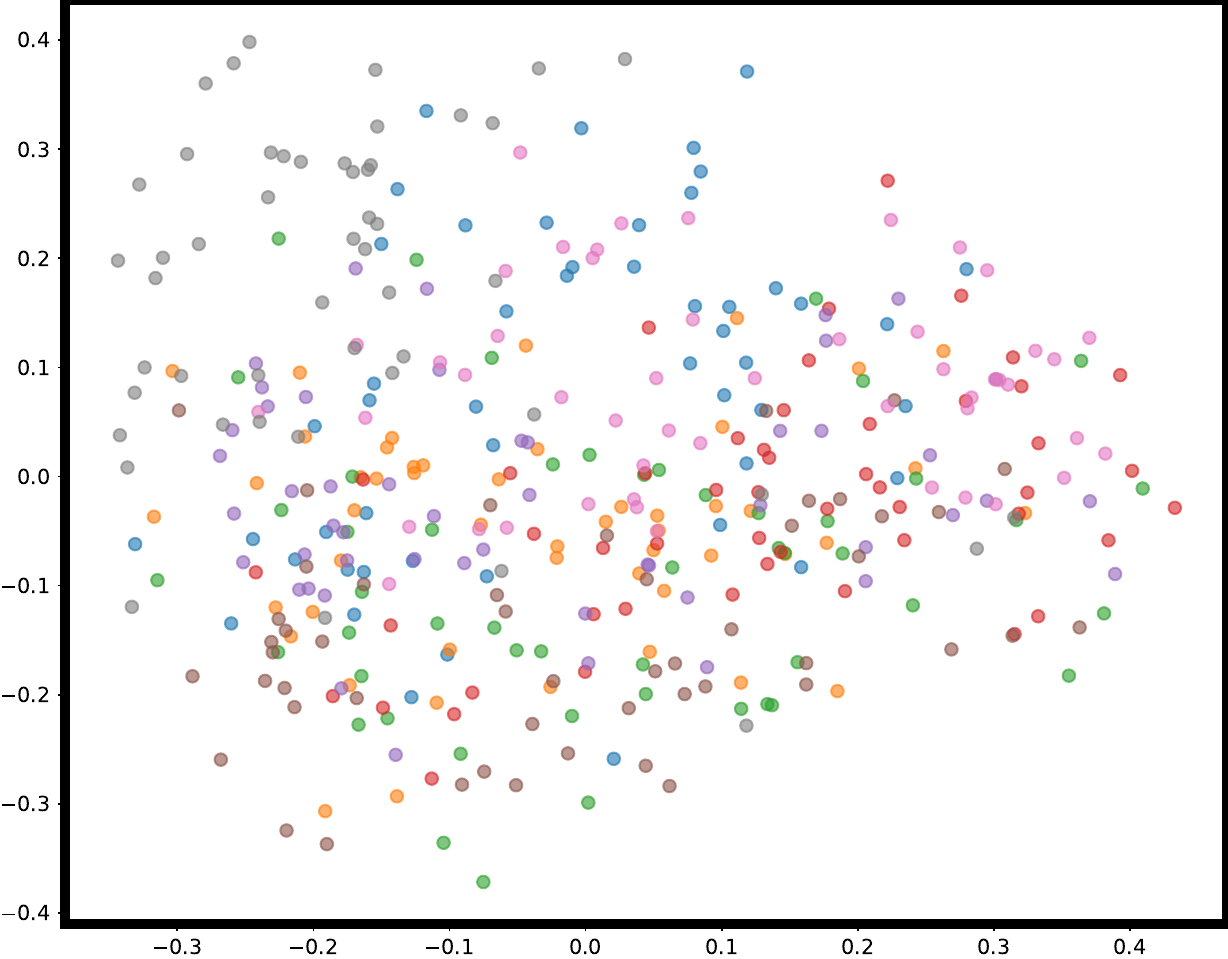}\label{fig:pca_simclr}}
     \subfigure[BYOL]{\includegraphics[width=0.3\textwidth]{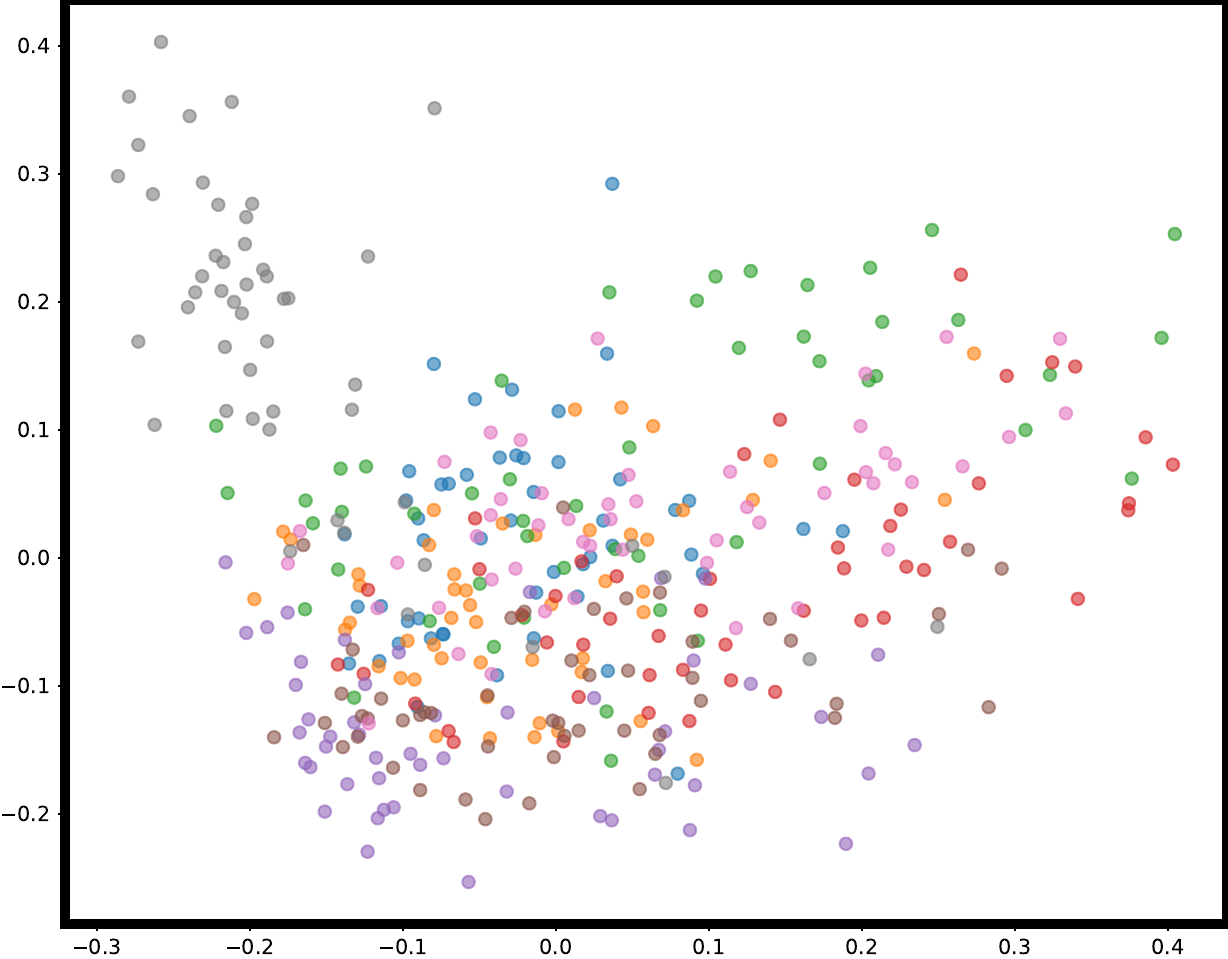}\label{fig:pca_byol}}
     \subfigure[Barlow Twins]{\includegraphics[width=0.3\textwidth]{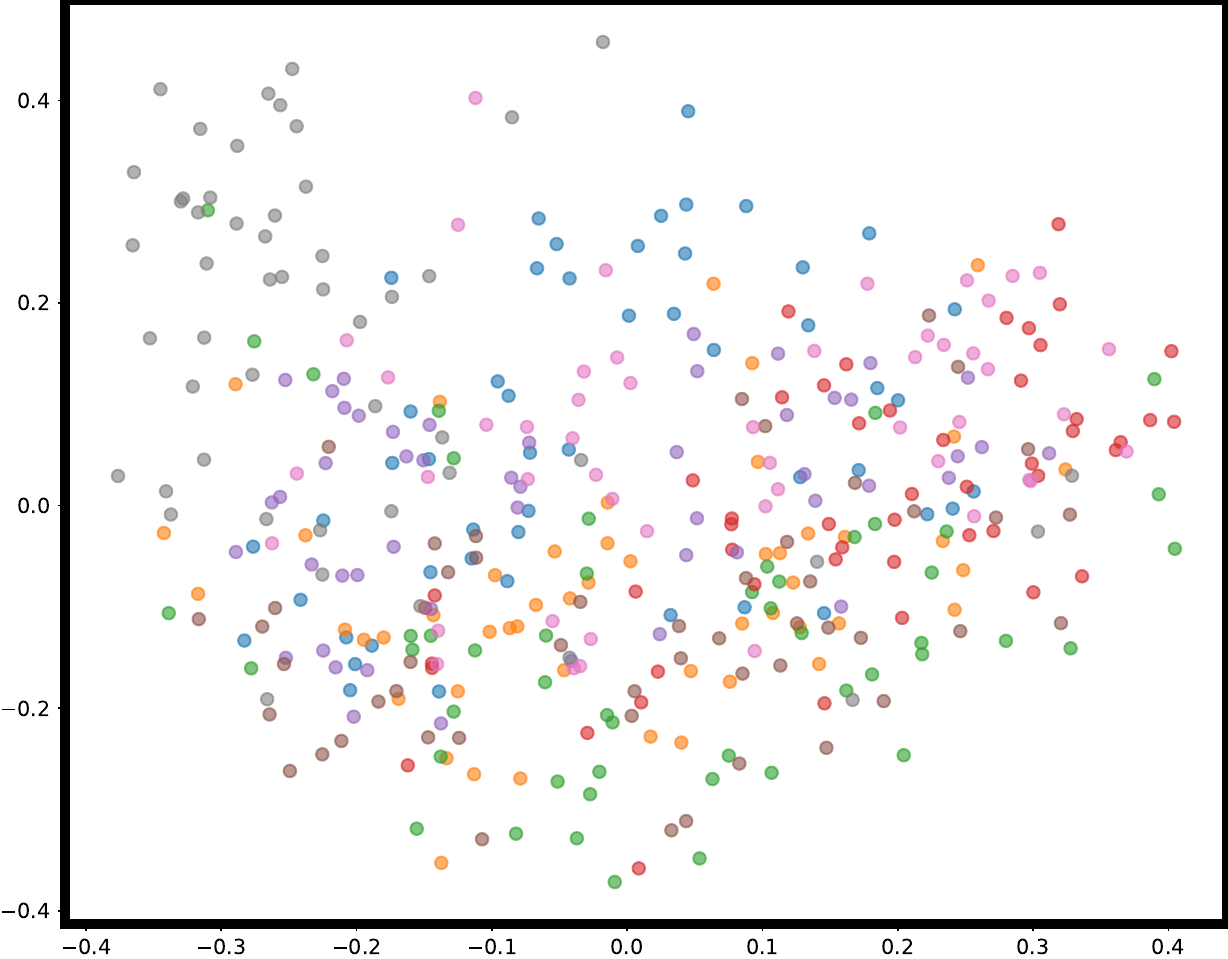}\label{fig:pca_barlow}}
     \subfigure[SwAV]{\includegraphics[width=0.3\textwidth]{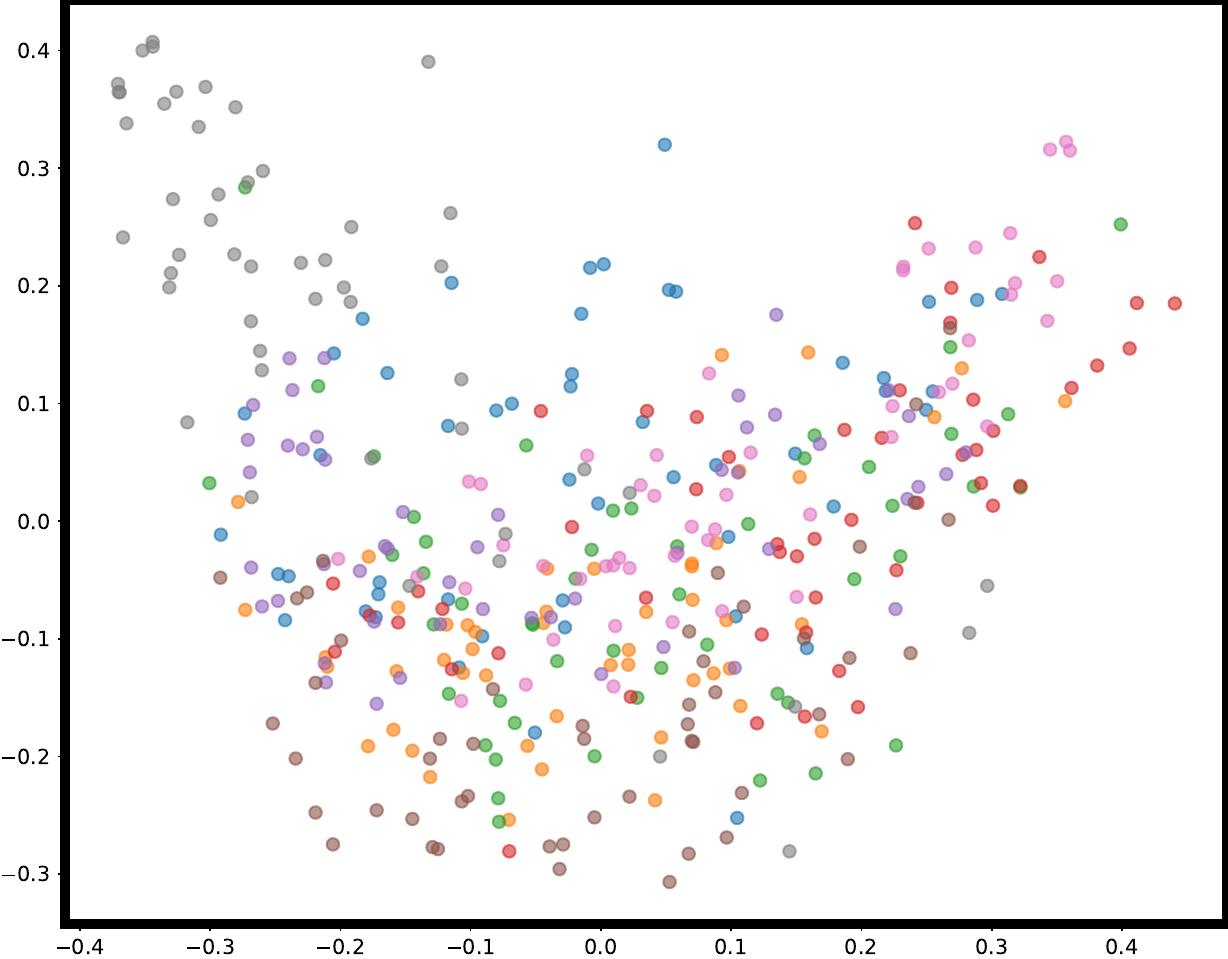}\label{fig:pca_swav}}
     \subfigure[MAE]{\includegraphics[width=0.3\textwidth]{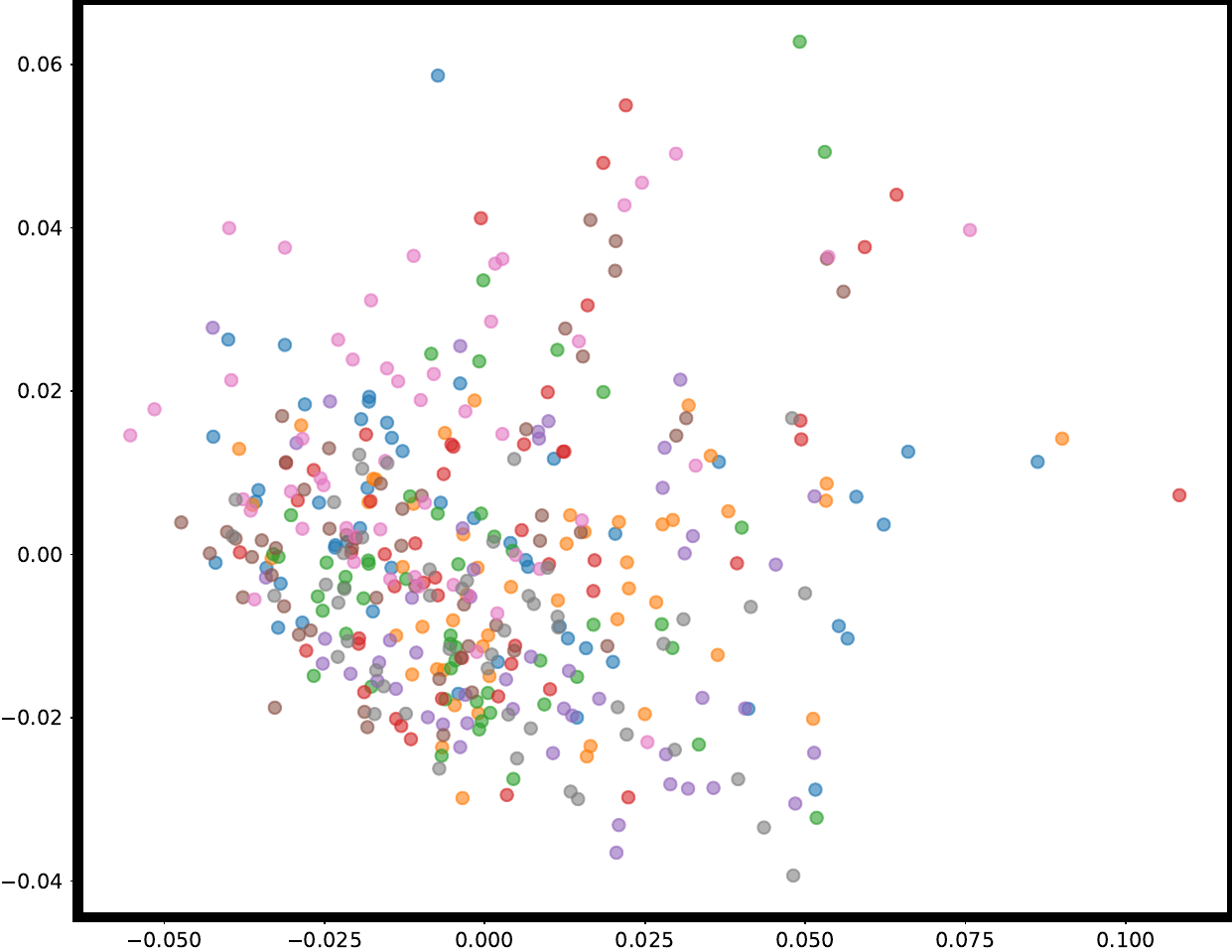}\label{fig:pca_mae}}
     \subfigure[Supervised]{\includegraphics[width=0.3\textwidth]{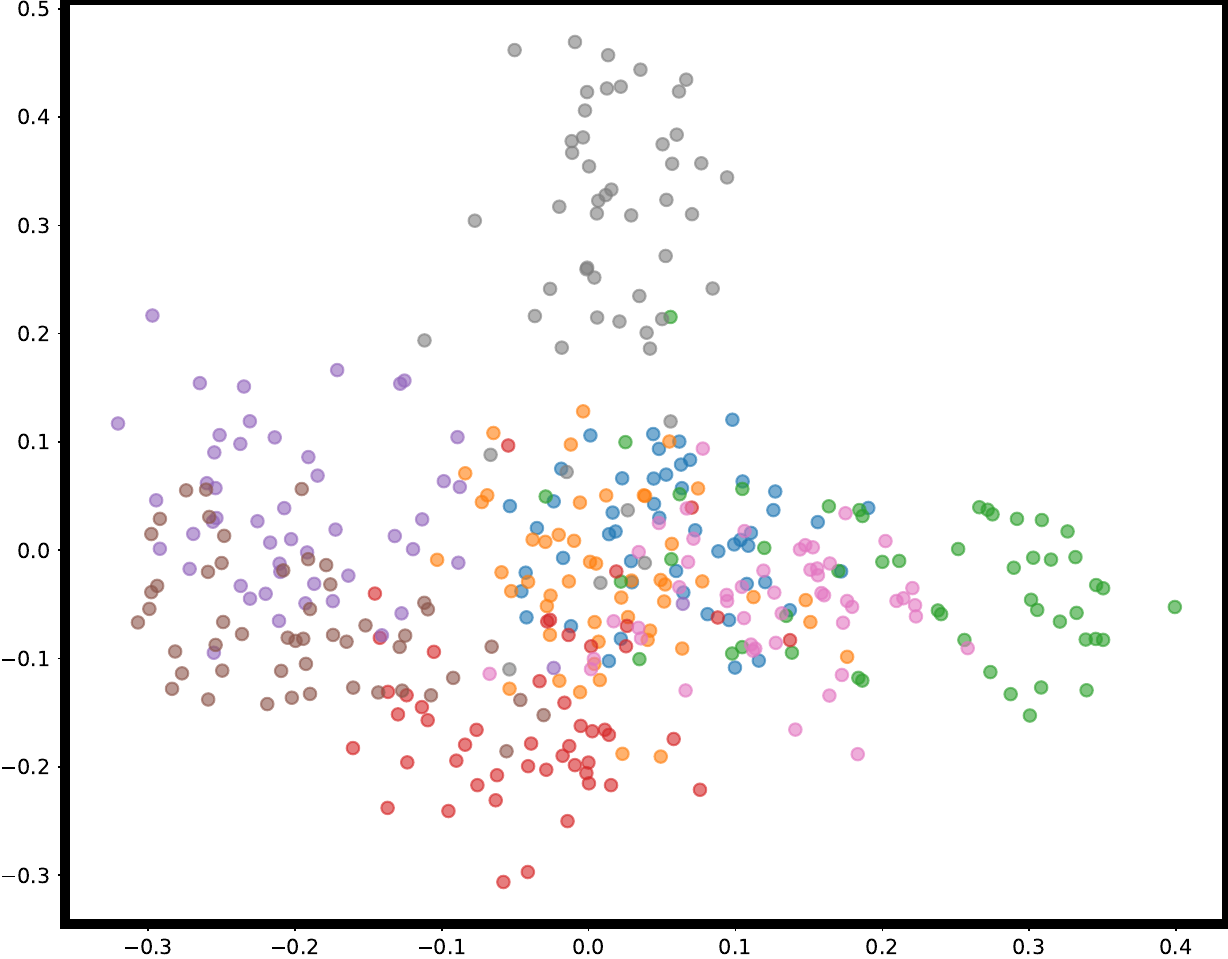}\label{fig:pca_supervised}}
    \caption{{Data distribution visualization with PCA based on the same set of classes as Figure \ref{fig:tsne} of the test set of ImageNet in the feature space. (a) - (e) corresponds to the visualization results of the self-supervised method, while (f) corresponds to the visualization results of the supervised method. }}
    \label{fig:pca}
\end{figure*}

\begin{figure*}[htb]
    \centering
    \subfigure[SimCLR + DSA]{\includegraphics[width=0.3\textwidth]{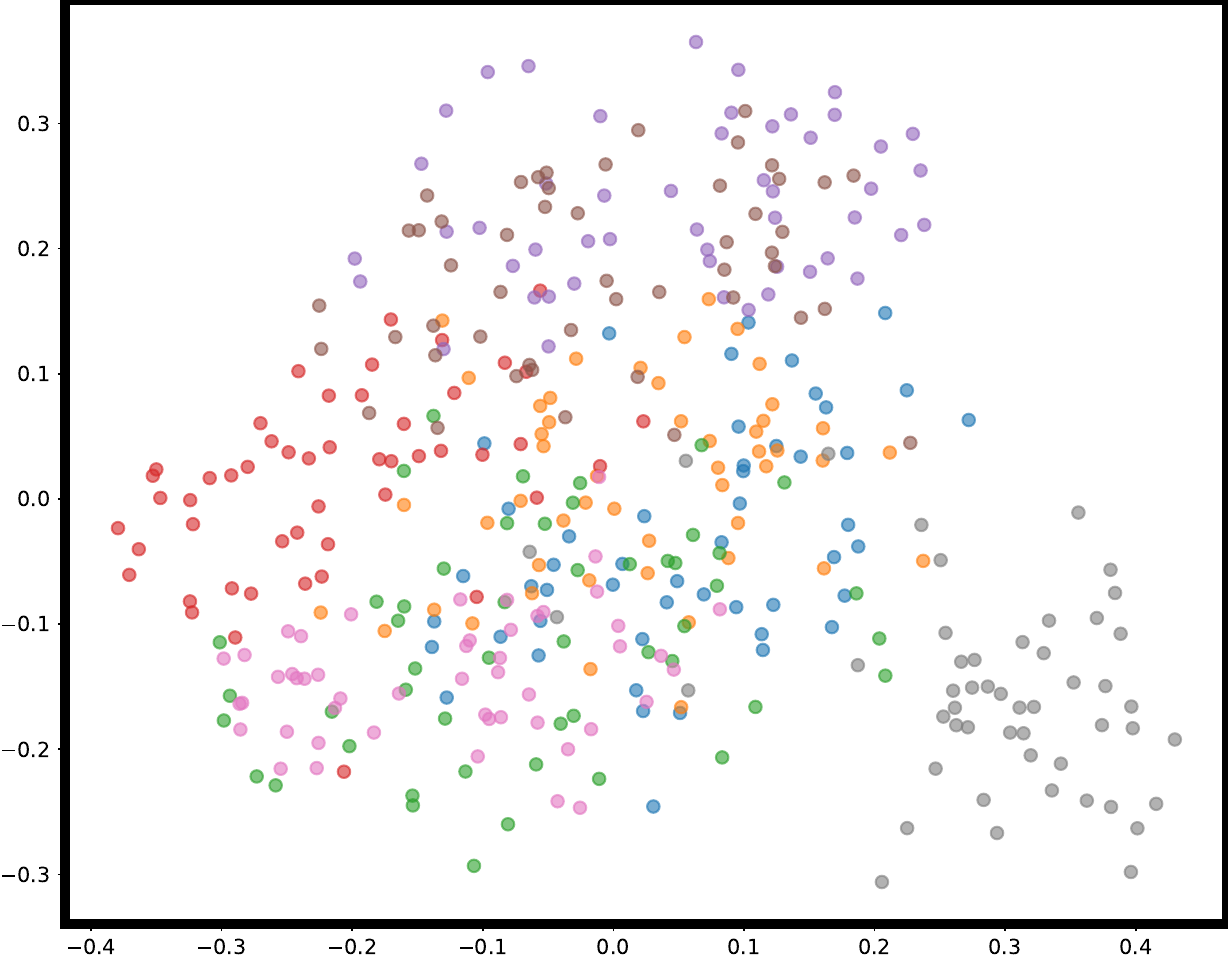}\label{fig:pca_simclr_better}}
     \subfigure[BYOL + DSA]{\includegraphics[width=0.3\textwidth]{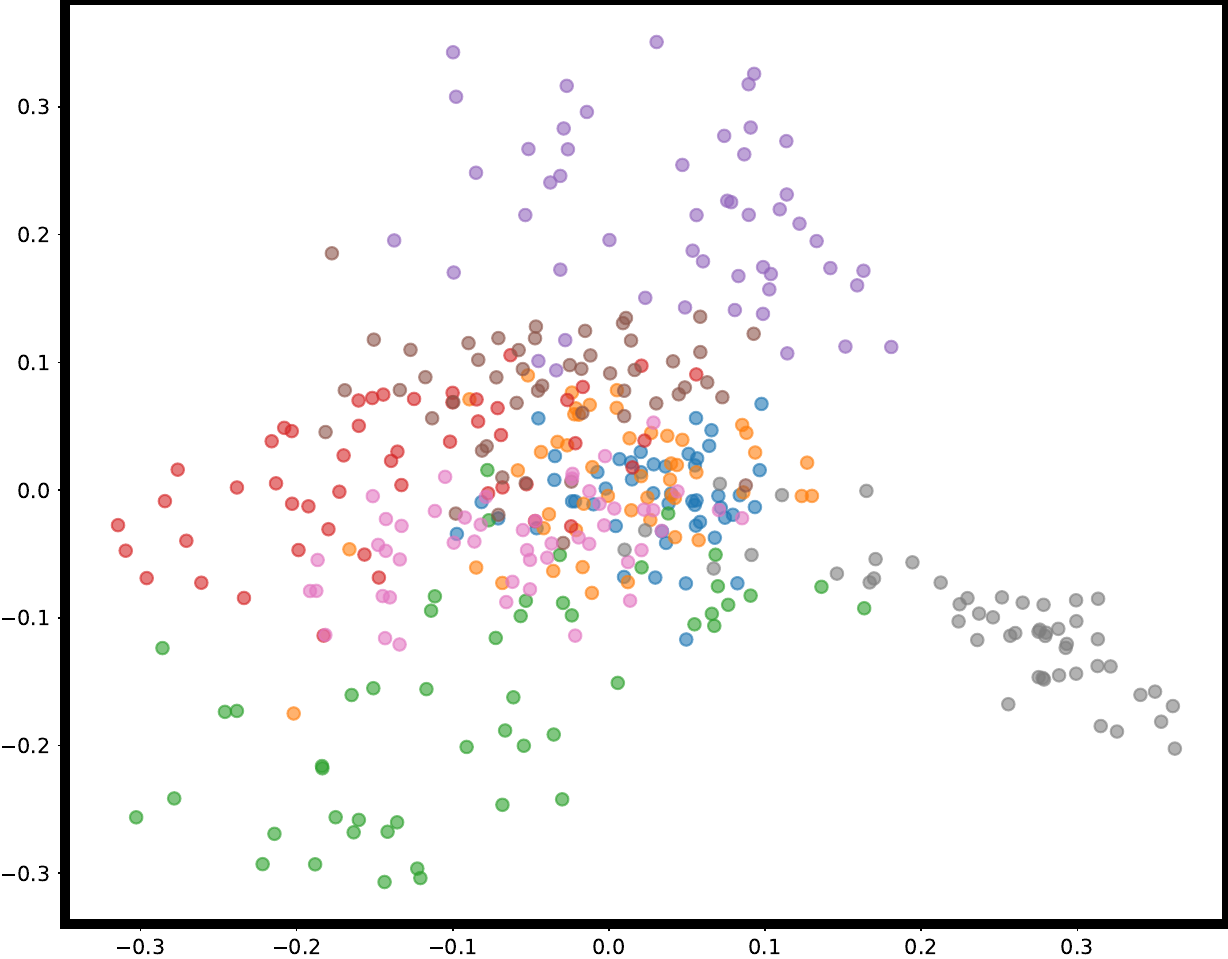}\label{fig:pca_byol_better}}
     \subfigure[Barlow Twins + DSA]{\includegraphics[width=0.3\textwidth]{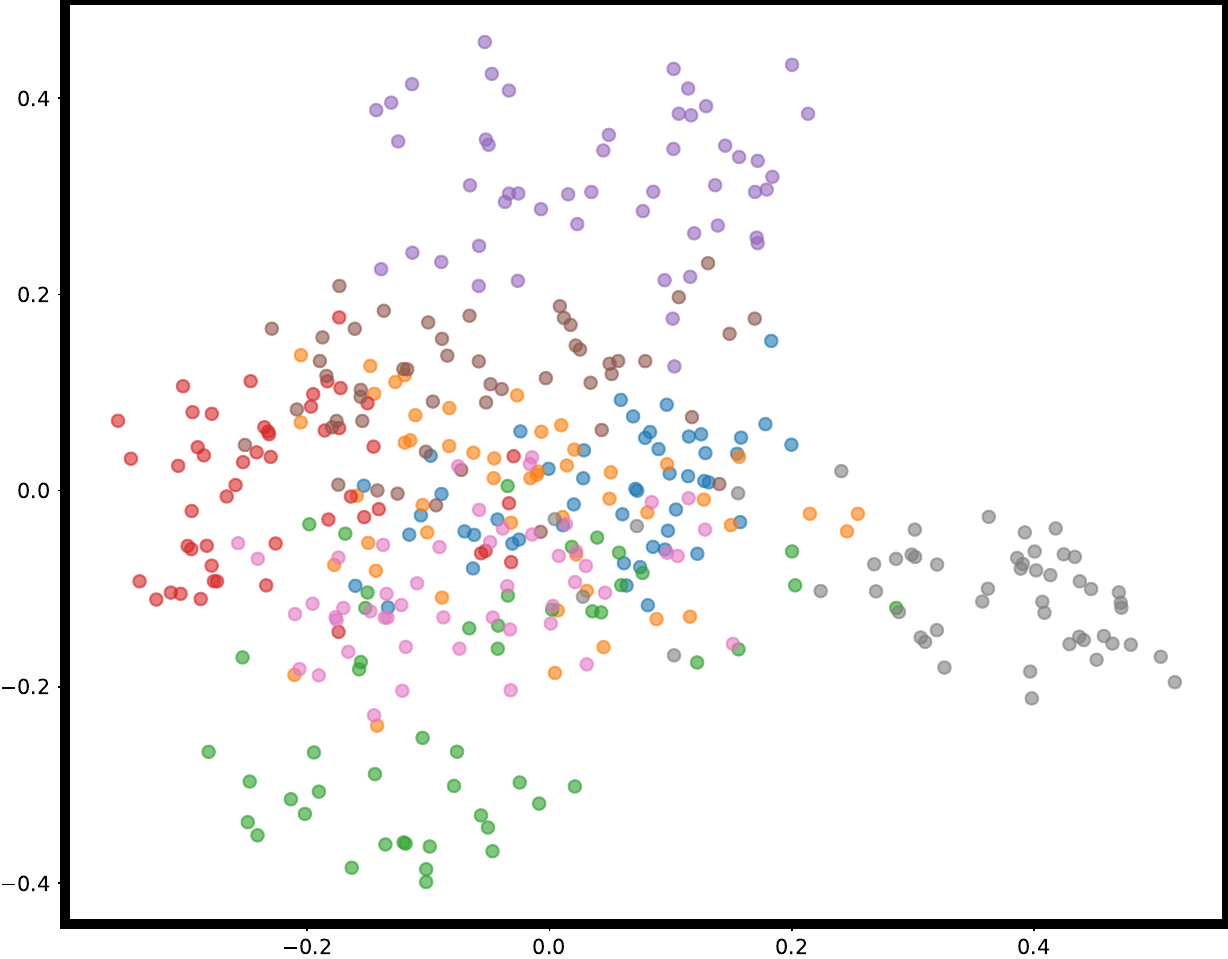}\label{fig:pca_barlow_better}}
     \subfigure[SwAV + DSA]{\includegraphics[width=0.3\textwidth]{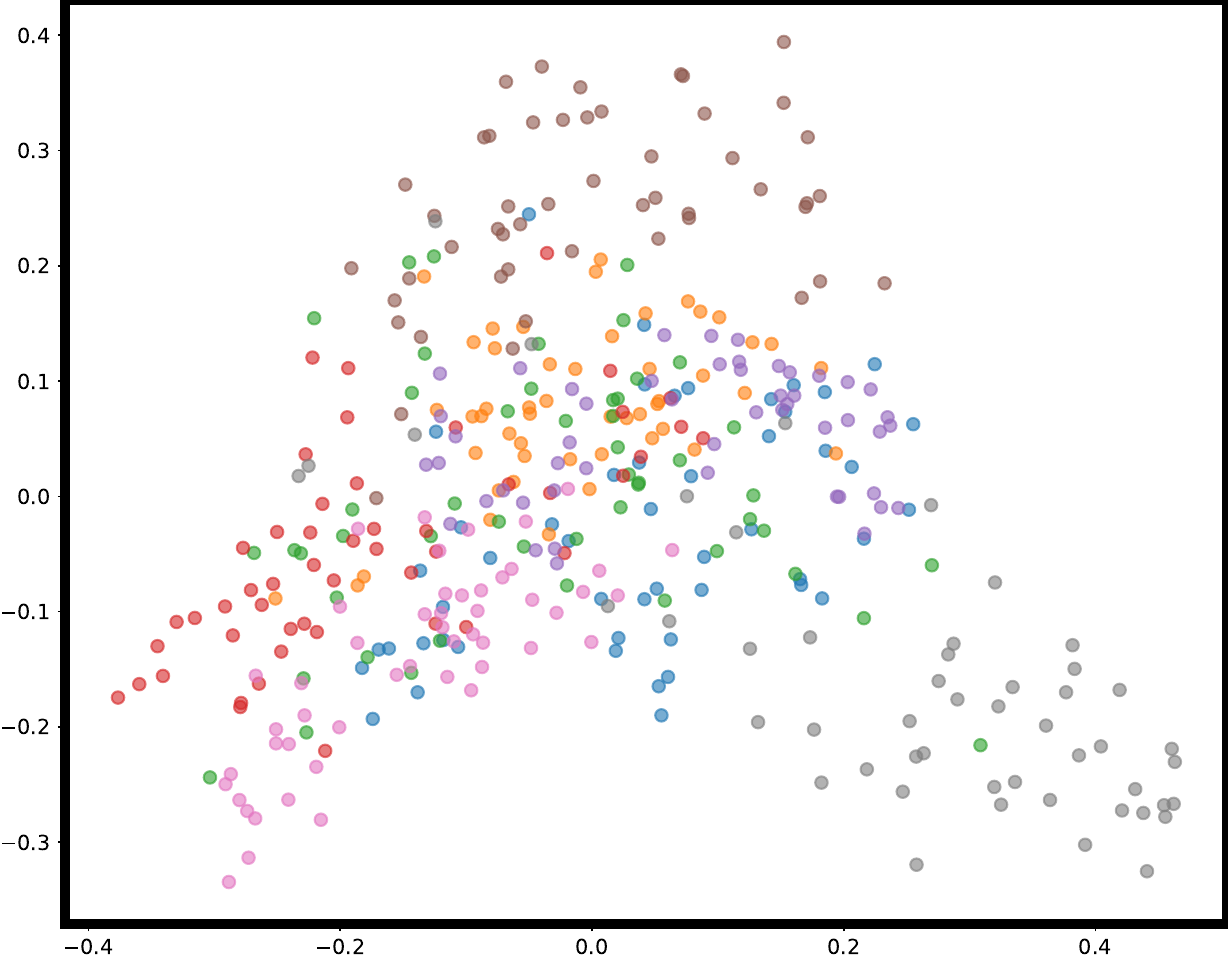}\label{fig:pca_swav_better}}
     \subfigure[MAE + DSA]{\includegraphics[width=0.3\textwidth]{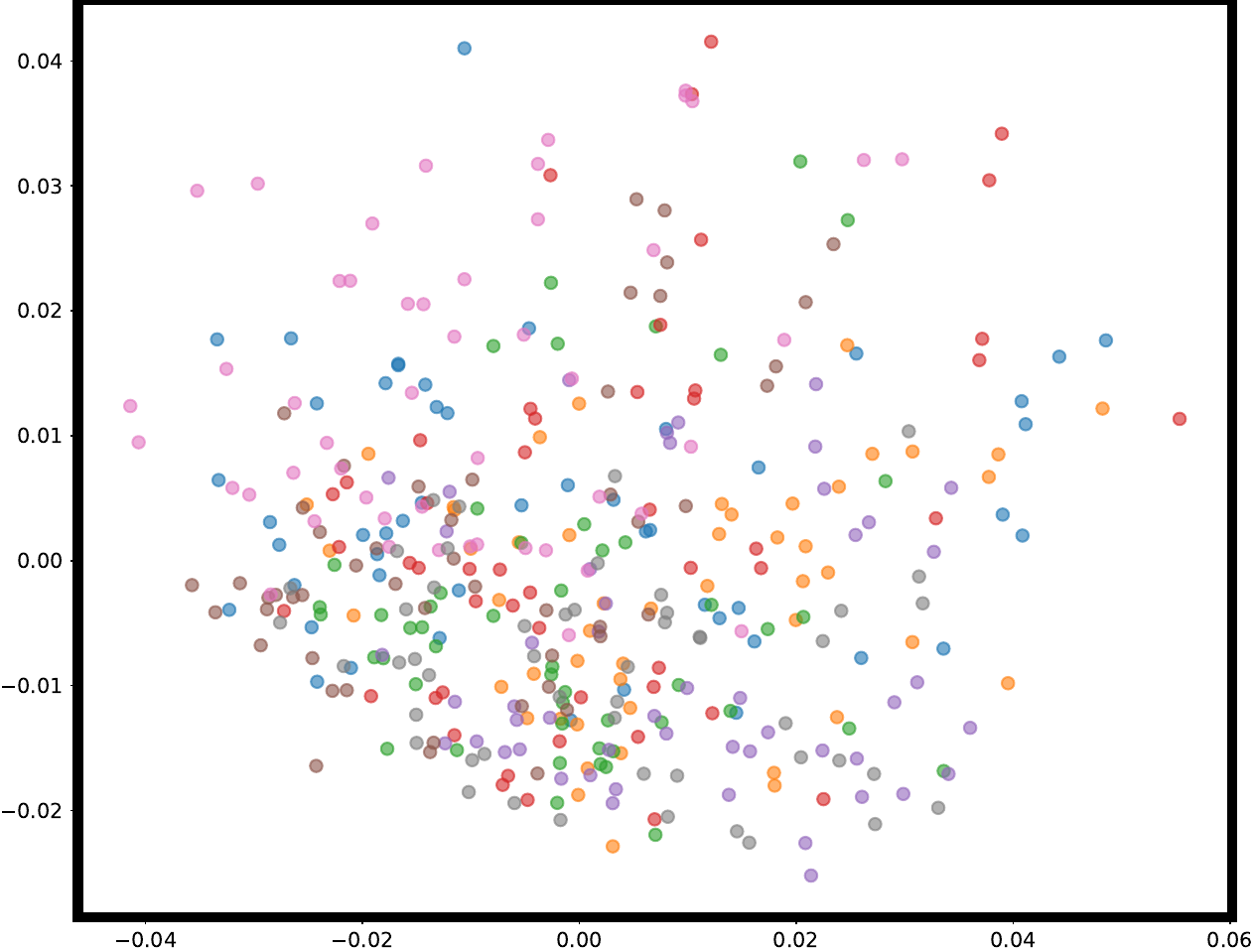}\label{fig:pca_mae_better}}
    \caption{{Data distribution visualization with PCA based on the same set of classes as Figure \ref{fig:tsne} of the test set of ImageNet in the feature space. (a) - (e) corresponds to the visualization results of the self-supervised methods integrated with DSA. }}
    \label{fig:pca_better}
\end{figure*}

\begin{figure*}[t]
     \centering
     \subfigure[SimCLR]{\includegraphics[width=0.3\textwidth]{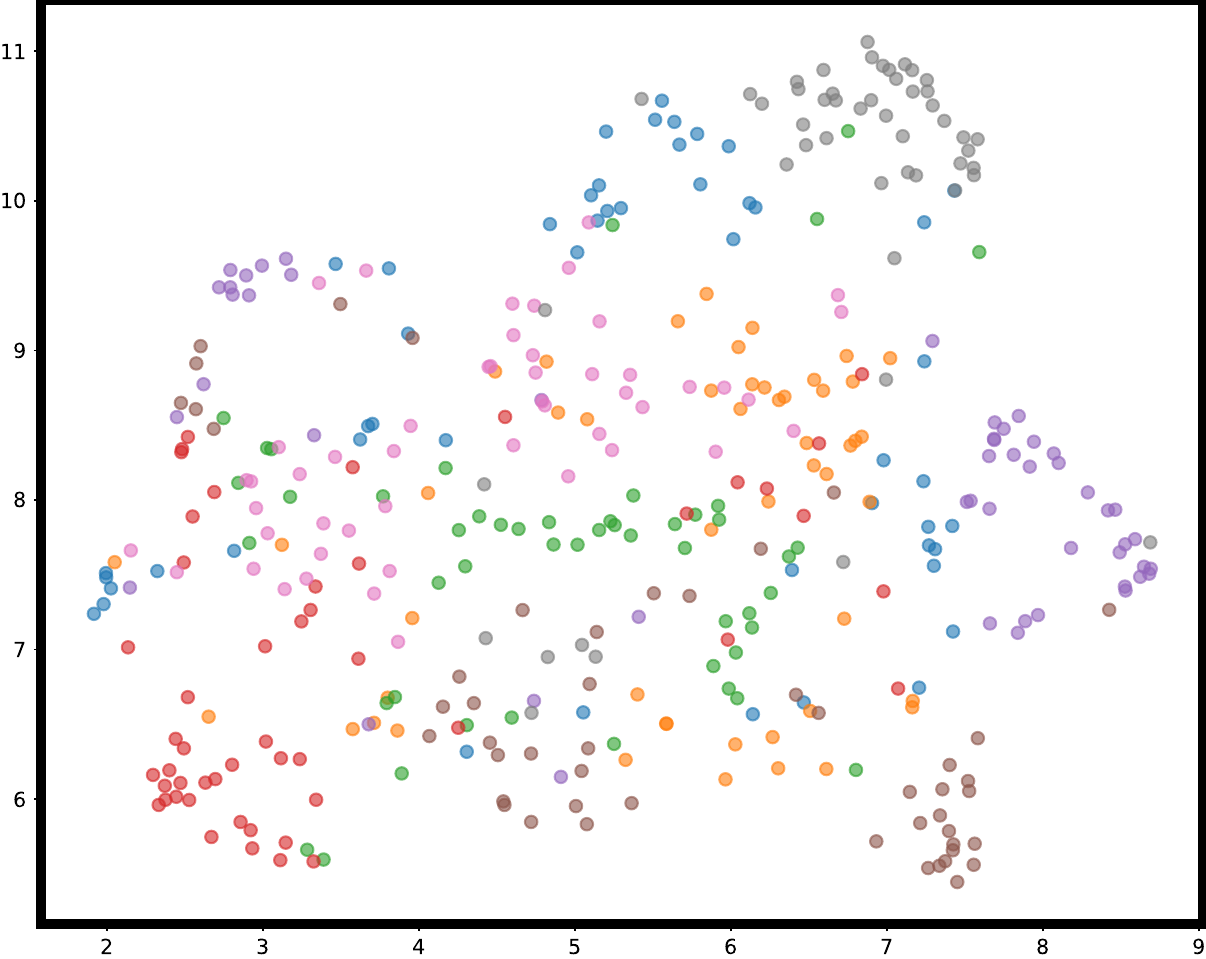}\label{fig:umap_simclr}}
     \subfigure[BYOL]{\includegraphics[width=0.3\textwidth]{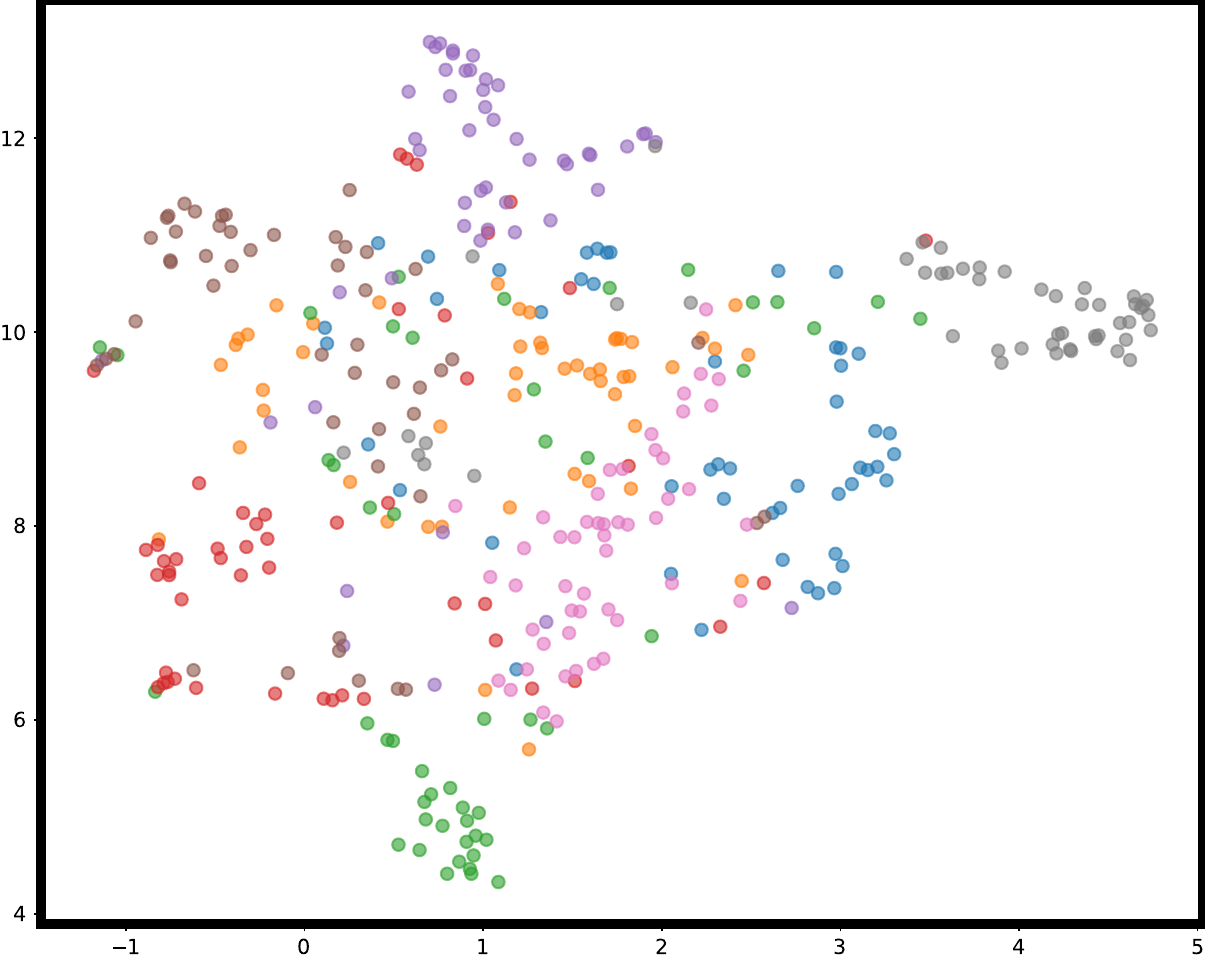}\label{fig:umap_byol}}
     \subfigure[Barlow Twins]{\includegraphics[width=0.3\textwidth]{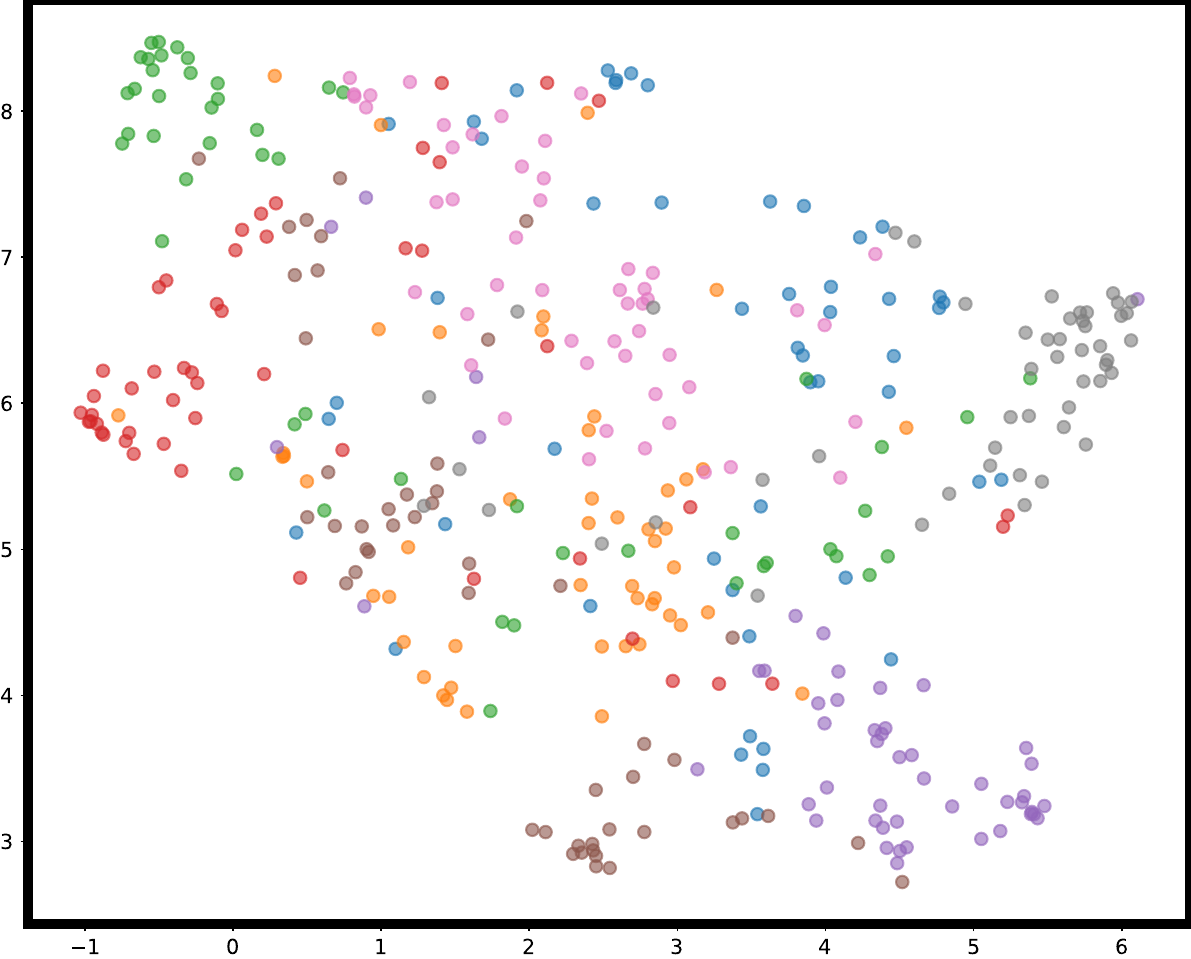}\label{fig:umap_barlow}}
     \subfigure[SwAV]{\includegraphics[width=0.3\textwidth]{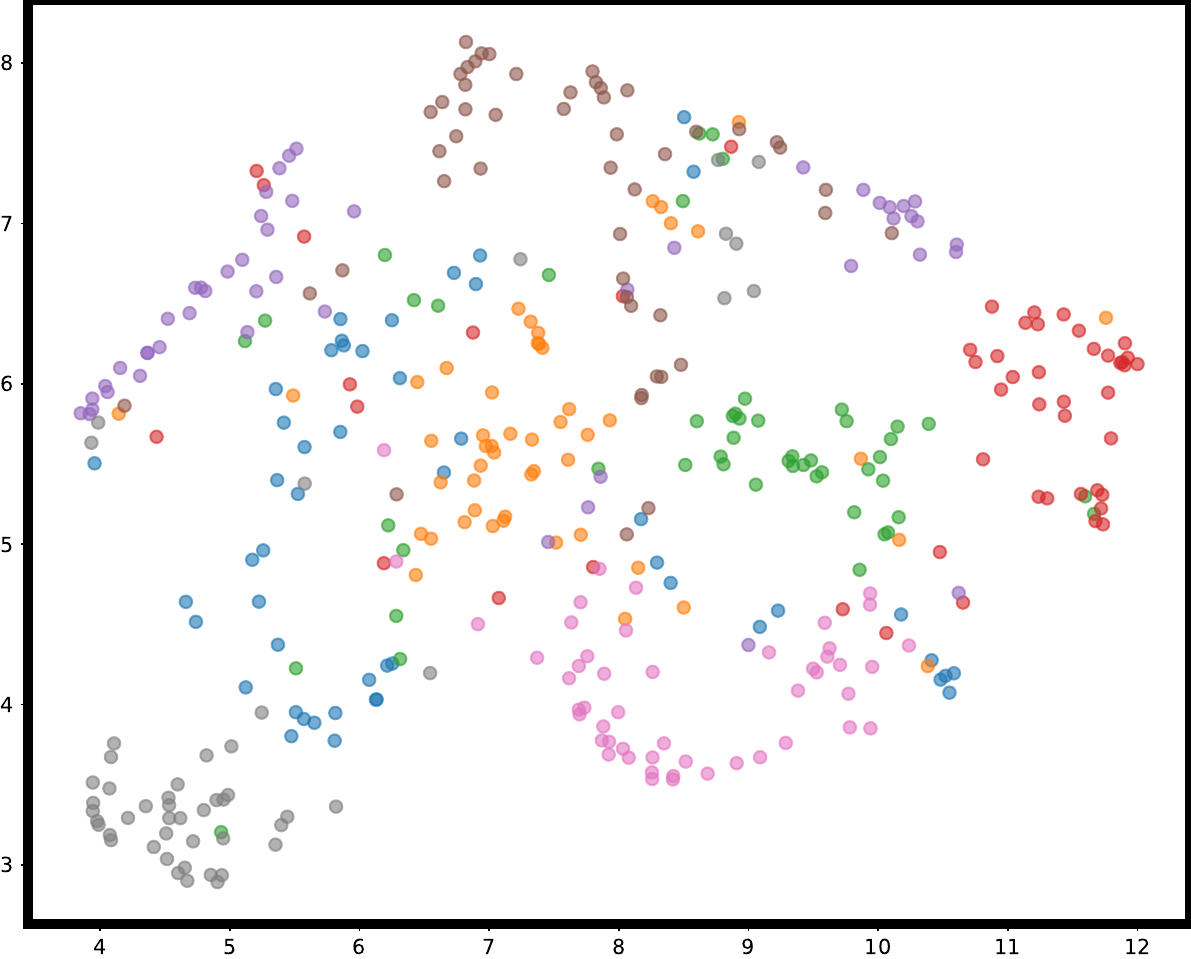}\label{fig:umap_swav}}
     \subfigure[MAE]{\includegraphics[width=0.3\textwidth]{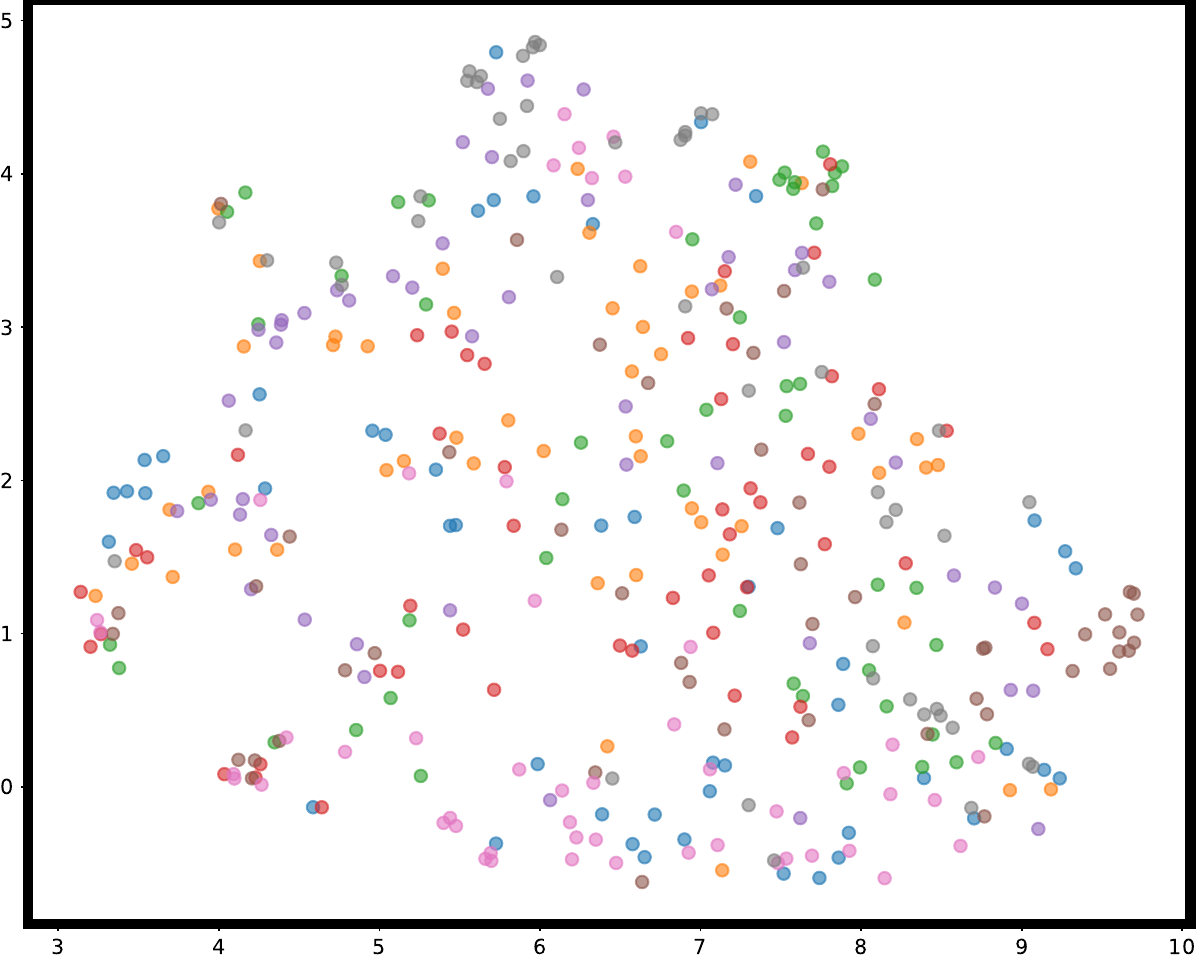}\label{fig:umap_mae}}
     \subfigure[Supervised]{\includegraphics[width=0.3\textwidth]{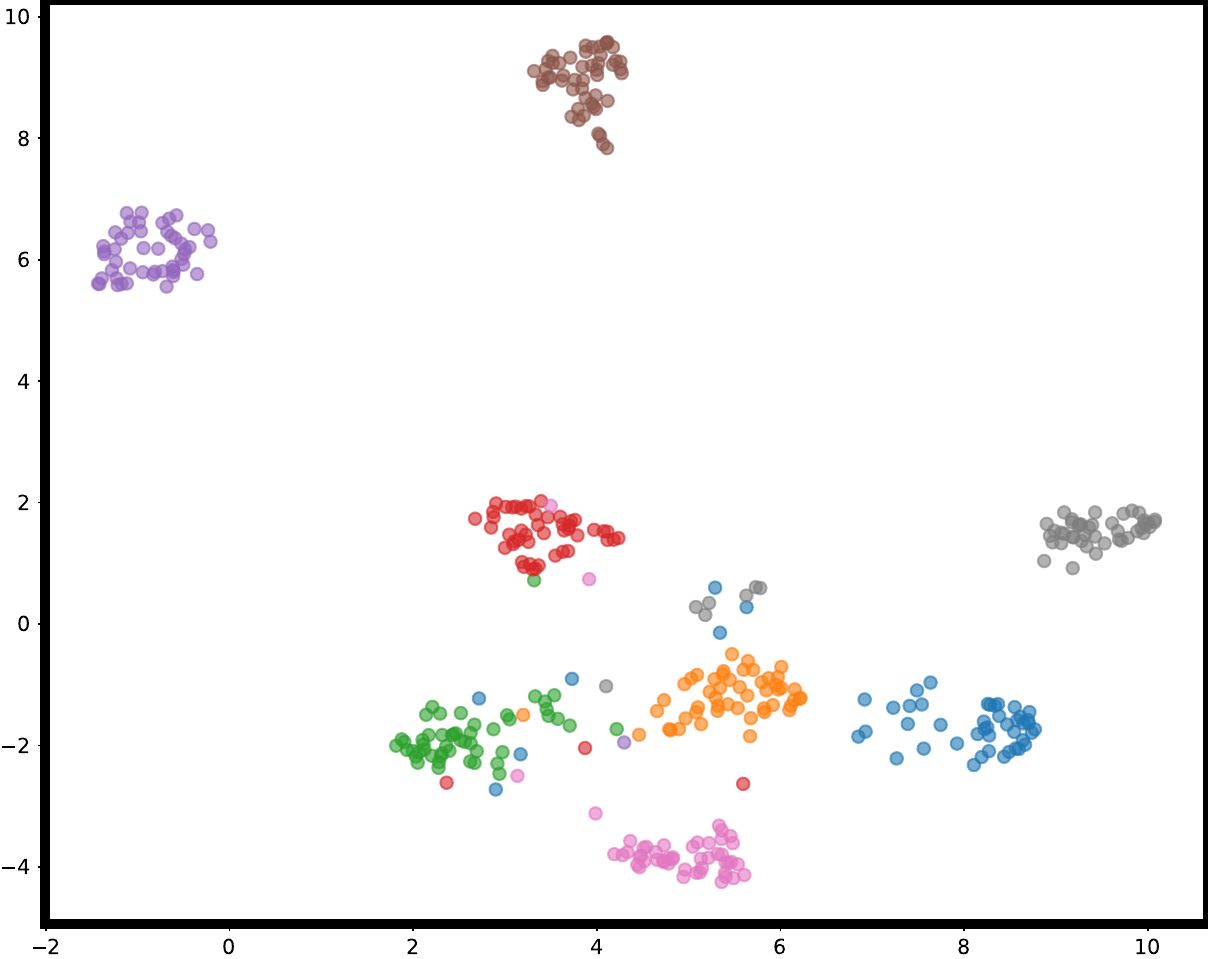}\label{fig:umap_supervised}}
    \caption{{Data distribution visualization with UMAP based on the same set of classes as Figure \ref{fig:tsne} of the test set of ImageNet in the feature space. (a) - (e) corresponds to the visualization results of the self-supervised method, while (f) corresponds to the visualization results of the supervised method. }}
    \label{fig:umap}
\end{figure*}

\begin{figure*}[htb]
    \centering
    \subfigure[SimCLR + DSA]{\includegraphics[width=0.3\textwidth]{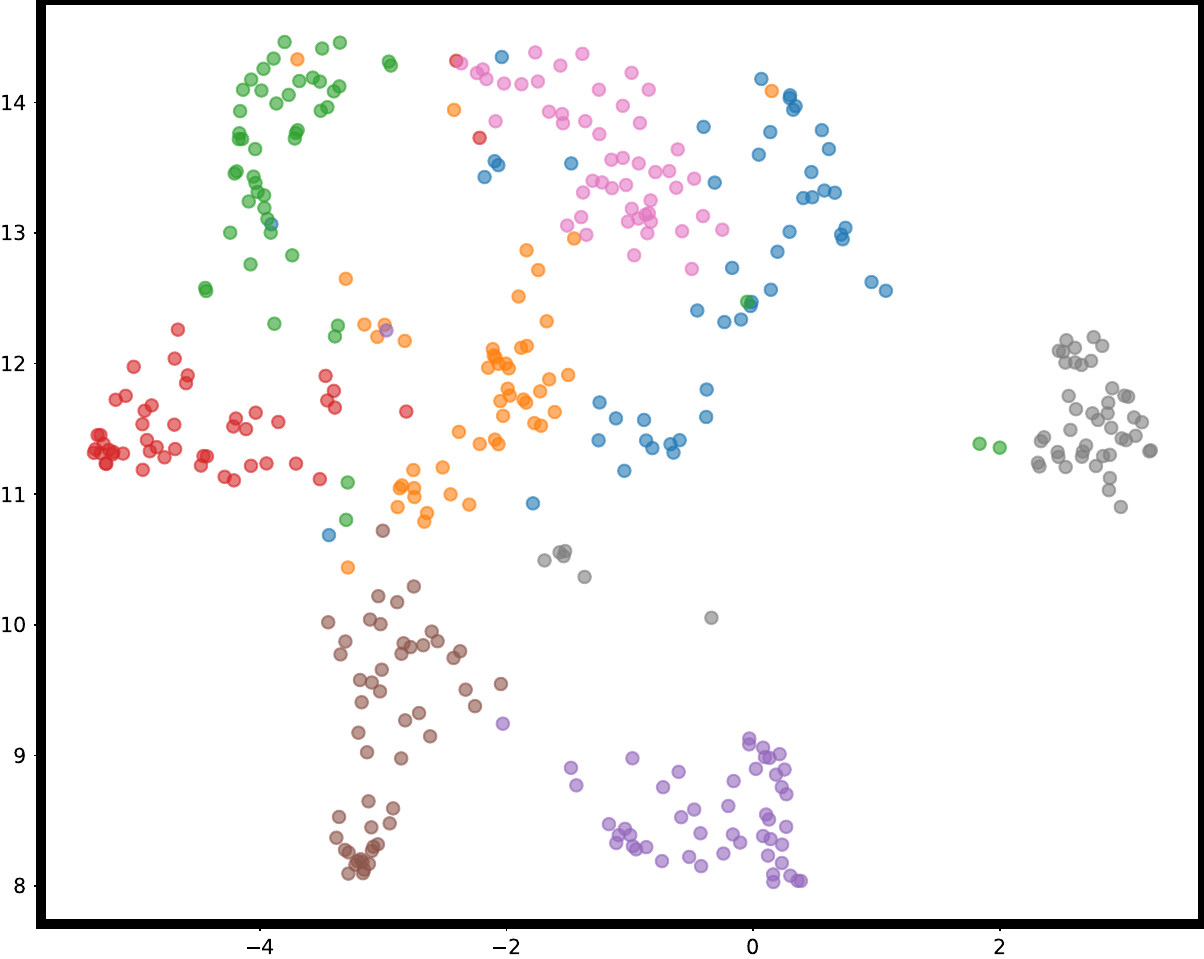}\label{fig:umap_simclr_better}}
     \subfigure[BYOL + DSA]{\includegraphics[width=0.3\textwidth]{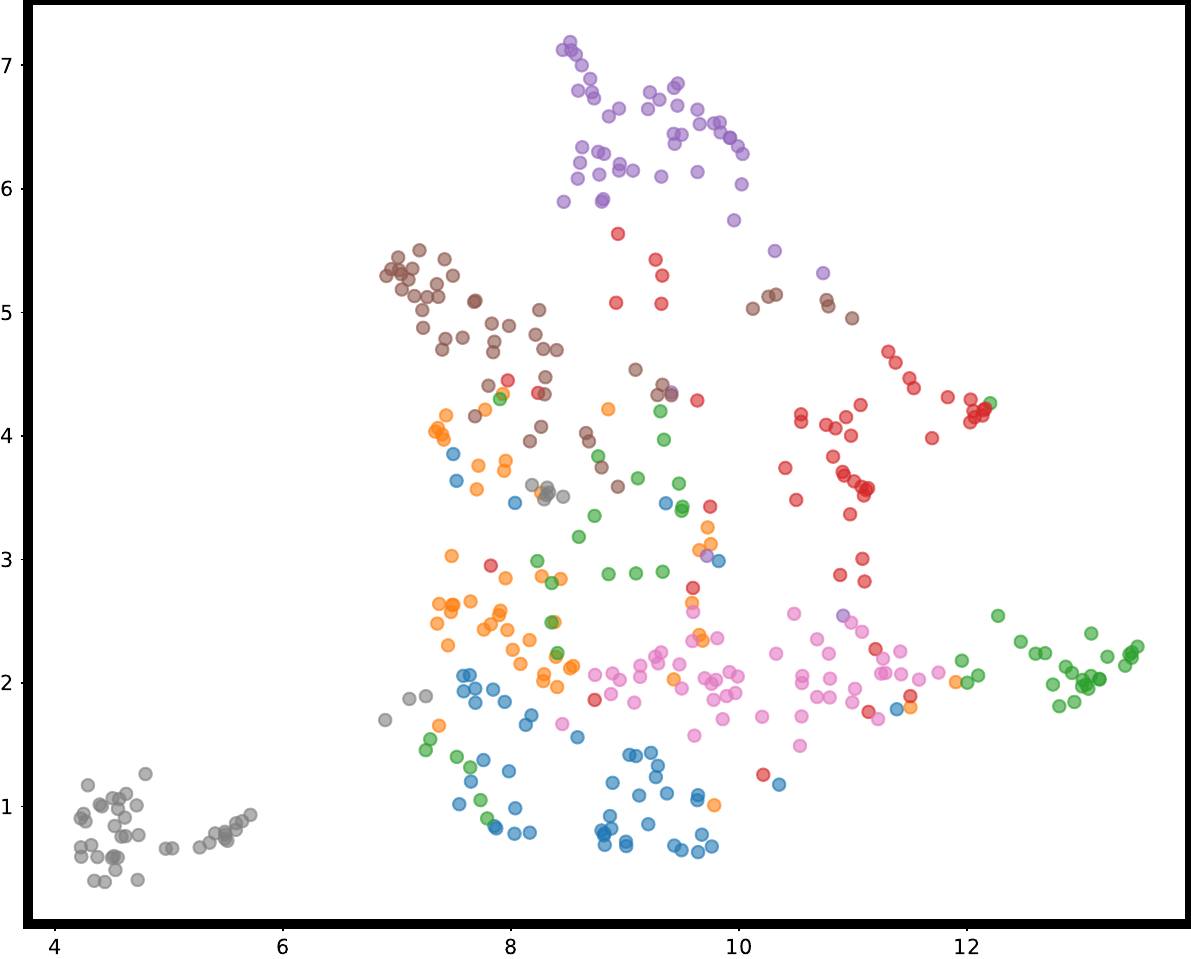}\label{fig:umap_byol_better}}
     \subfigure[Barlow Twins + DSA]{\includegraphics[width=0.3\textwidth]{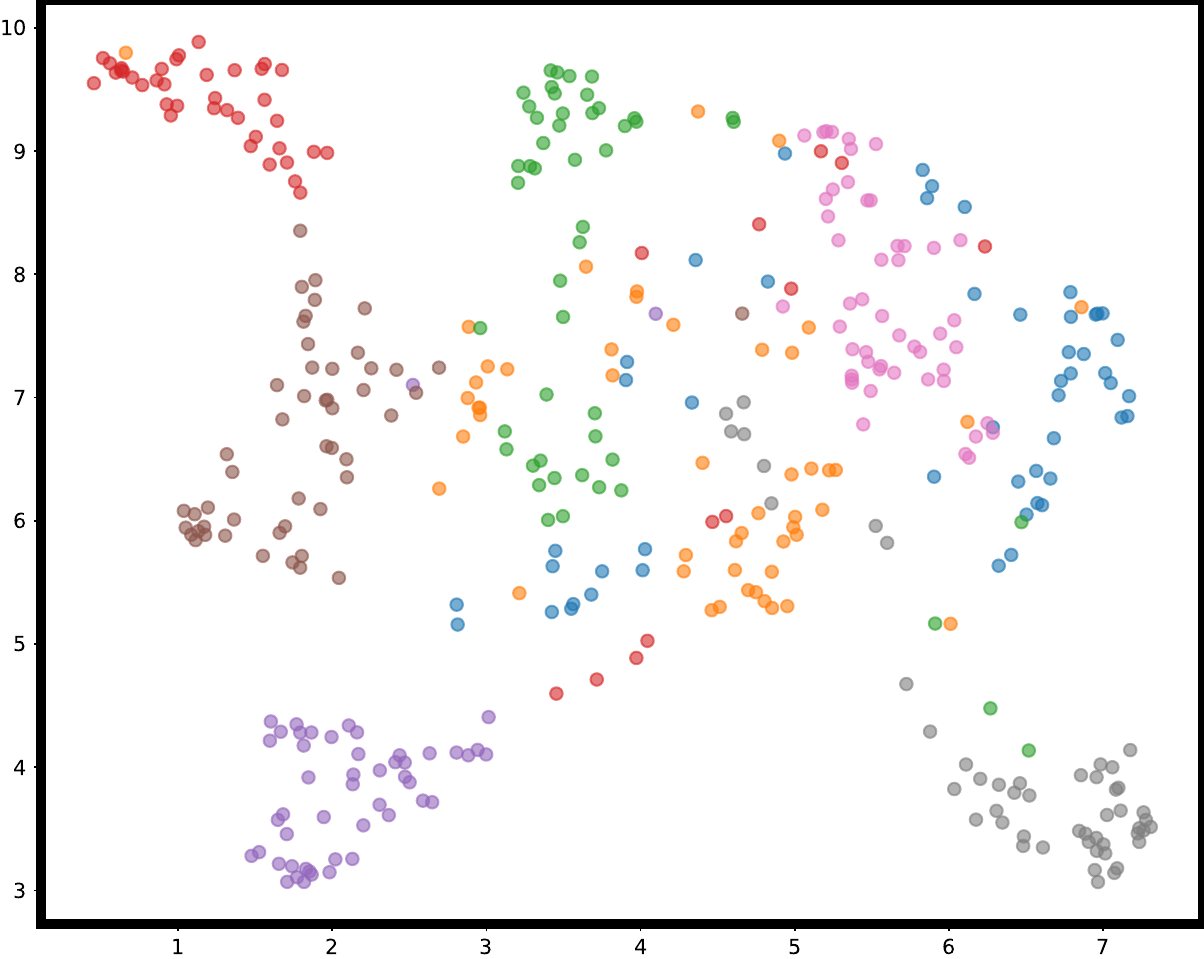}\label{fig:umap_barlow_better}}
     \subfigure[SwAV + DSA]{\includegraphics[width=0.3\textwidth]{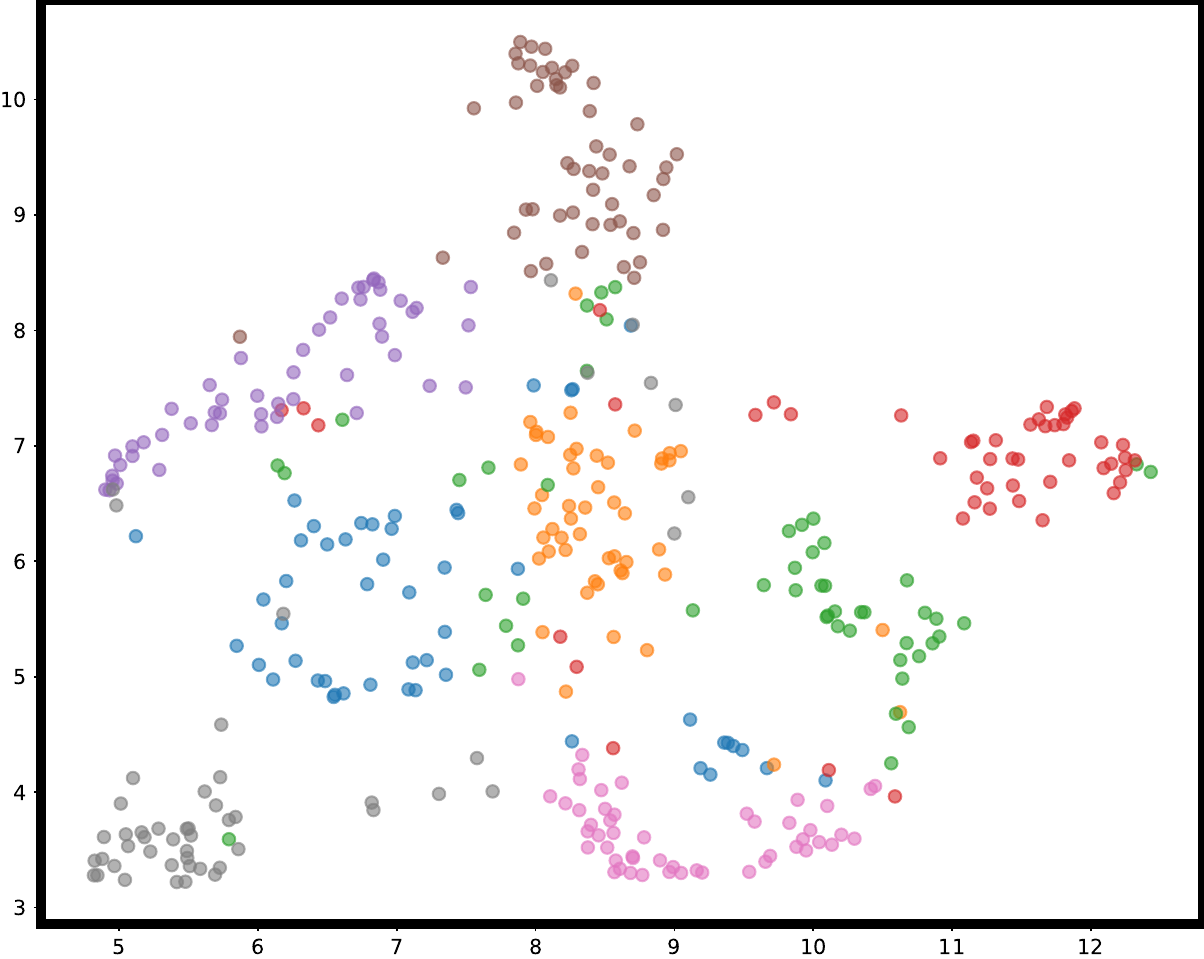}\label{fig:umap_swav_better}}
     \subfigure[MAE + DSA]{\includegraphics[width=0.3\textwidth]{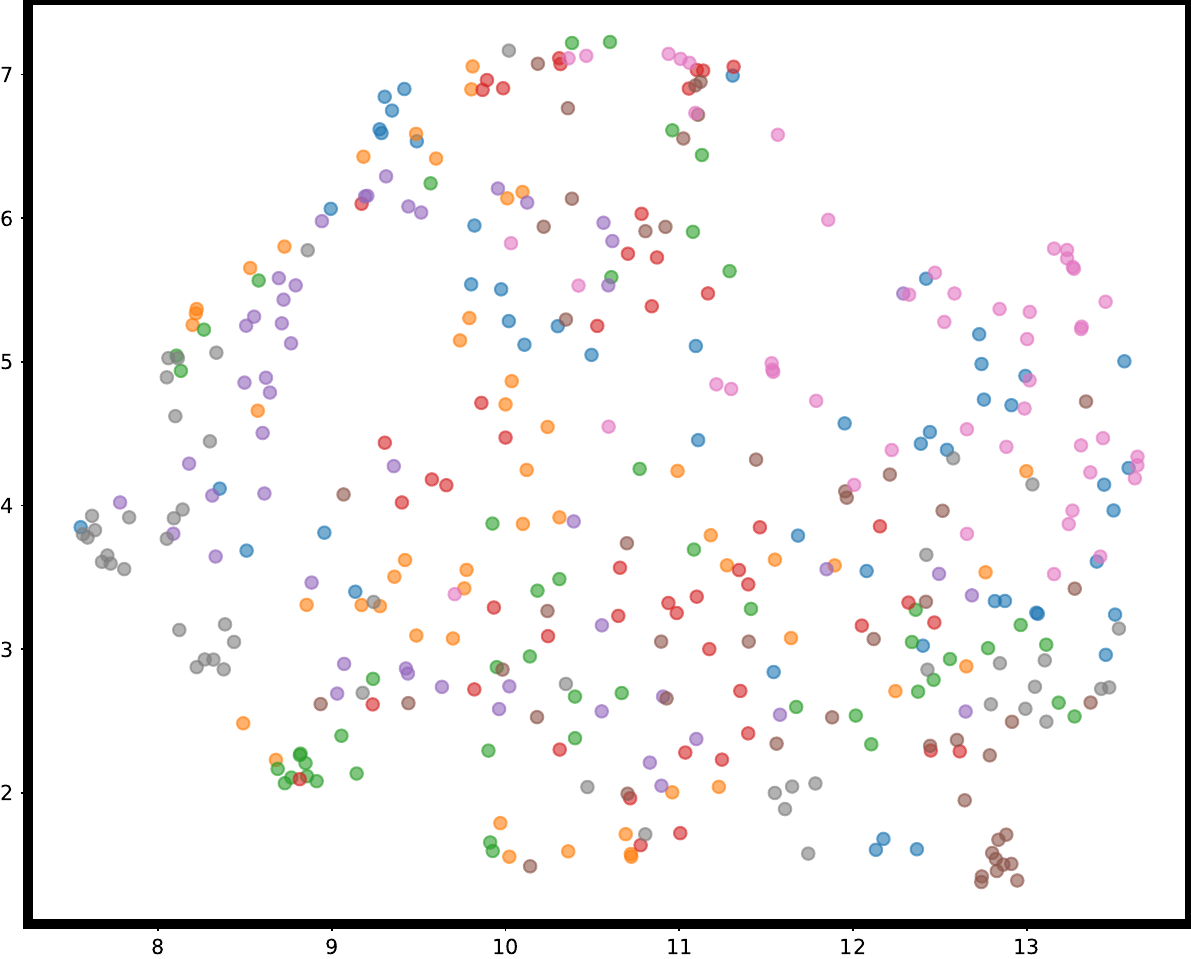}\label{fig:umap_mae_better}}
    \caption{{Data distribution visualization with UMAP based on the same set of classes as Figure \ref{fig:tsne} of the test set of ImageNet in the feature space. (a) - (e) corresponds to the visualization results of the self-supervised methods integrated with DSA. }}
    \label{fig:umap_better}
\end{figure*}

\subsection{{More Results of Hyperparameters Analysis}}

{To evaluate the robustness and generalizability of the proposed DSA framework across different downstream tasks, we conducted extensive hyperparameter sensitivity analyses on several datasets. Specifically, Figures~\ref{fig:hyperparam_3},~\ref{fig:hyperparam_4}, and~\ref{fig:hyperparam_5} illustrate how variations in hyperparameters: $\nu$, $\upsilon$, $\alpha$, $\eta$, and $\tau$ affect model performance.}

{Figure~\ref{fig:hyperparam_3} presents the Top-1 linear classification accuracy on the Kinetics-400 dataset under different hyperparameter settings. Figure~\ref{fig:hyperparam_4} shows the 5-way 1-shot and 5-way 5-shot classification accuracies on the FC100 dataset, where solid lines denote the 1-shot results and dashed lines denote the 5-shot results. Figure~\ref{fig:hyperparam_5} displays the performance on UCF-101 and HMDB-51, with solid lines representing UCF-101 and dashed lines representing HMDB-51.}

{From these experiments, we observe several consistent patterns across datasets. Although the precise optimal value for each hyperparameter may vary depending on the specific task, the performance generally remains stable within certain ranges for all hyperparameters tested. Notably, DSA-enhanced models consistently outperform their baseline counterparts across a wide range of hyperparameter values, confirming that DSA is robust and does not rely on fine-tuned parameter settings to achieve improvements. These findings provide further evidence that DSA can generalize effectively to diverse downstream scenarios without the need for exhaustive hyperparameter tuning.}

\begin{figure*}[htb]
    \centering
    \subfigure[$\nu$]{\includegraphics[width=0.3\textwidth]{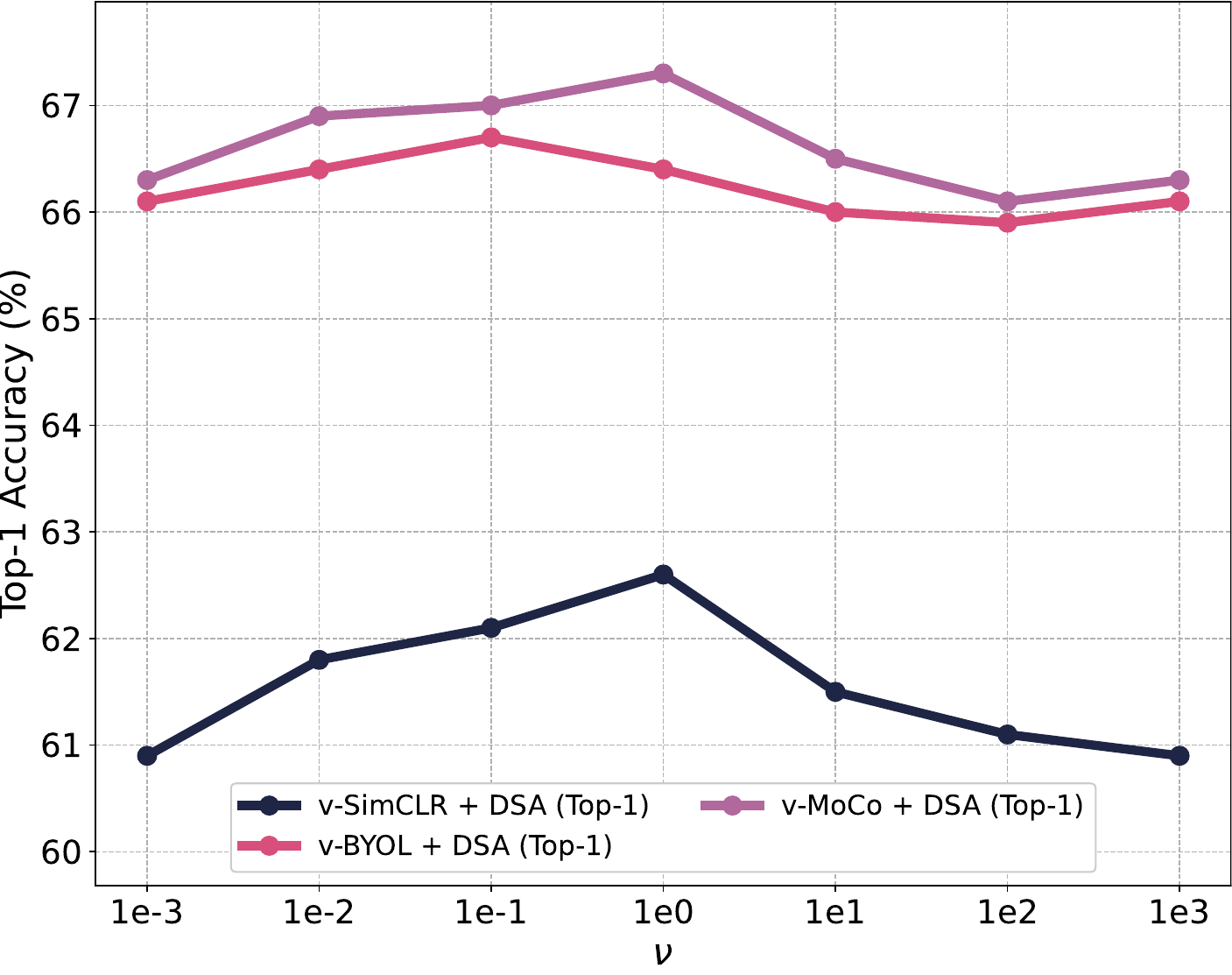}\label{fig:nu_k400}}
    \subfigure[$\upsilon$]{\includegraphics[width=0.3\textwidth]{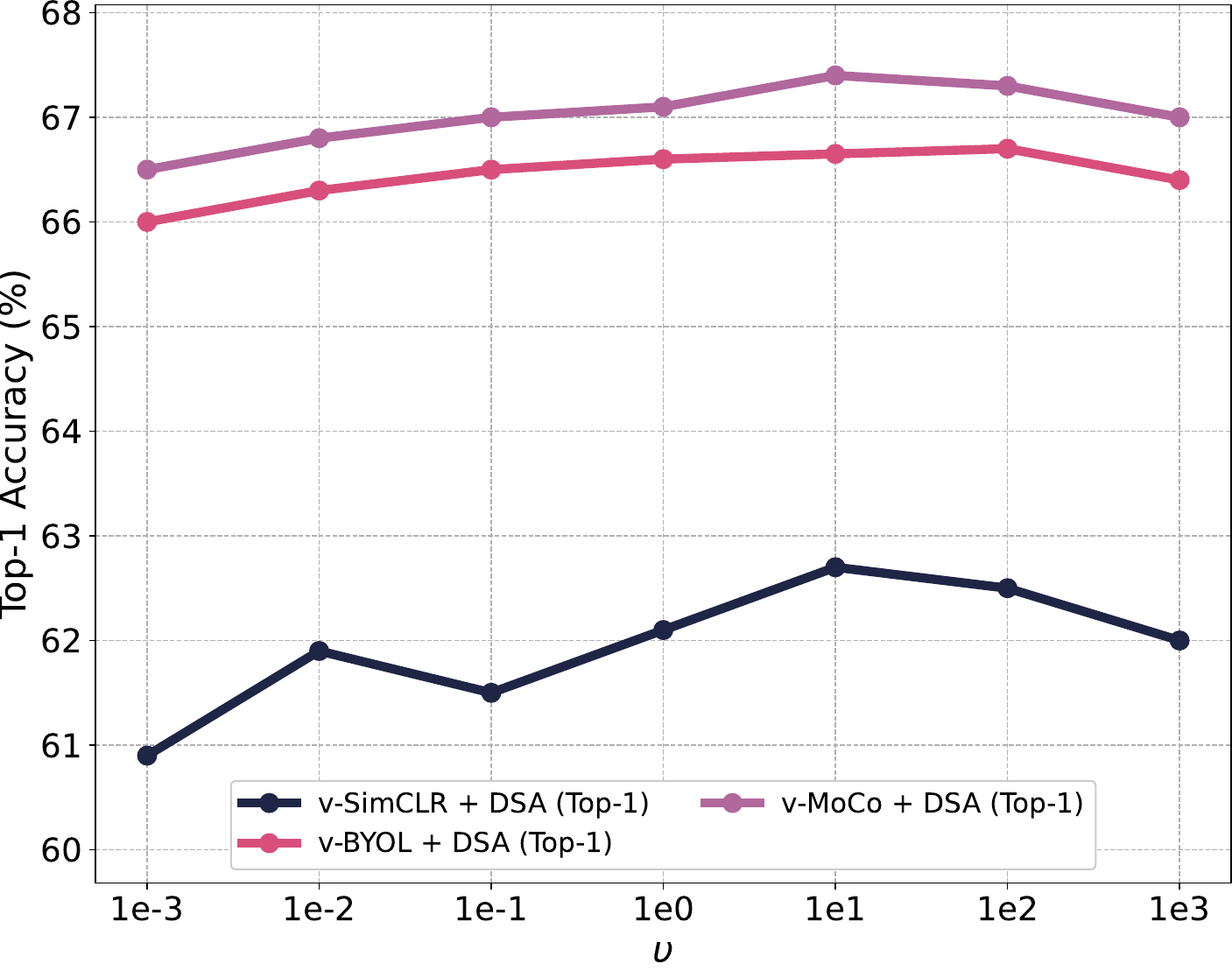}\label{fig:upsilon_k400}}
    \subfigure[$\alpha$]{\includegraphics[width=0.3\textwidth]{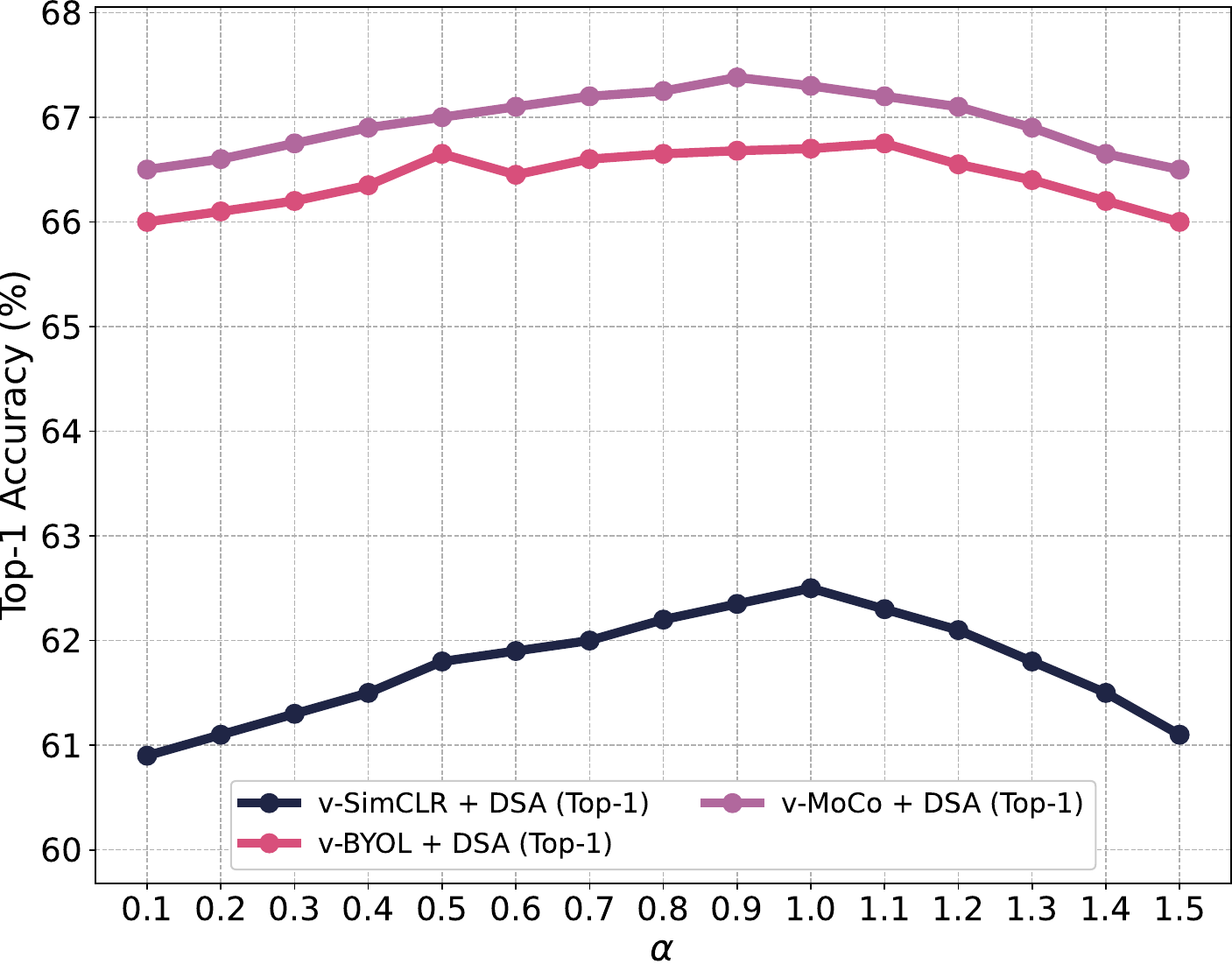}\label{fig:alpha_k400}}
    \subfigure[$\eta$]{\includegraphics[width=0.3\textwidth]{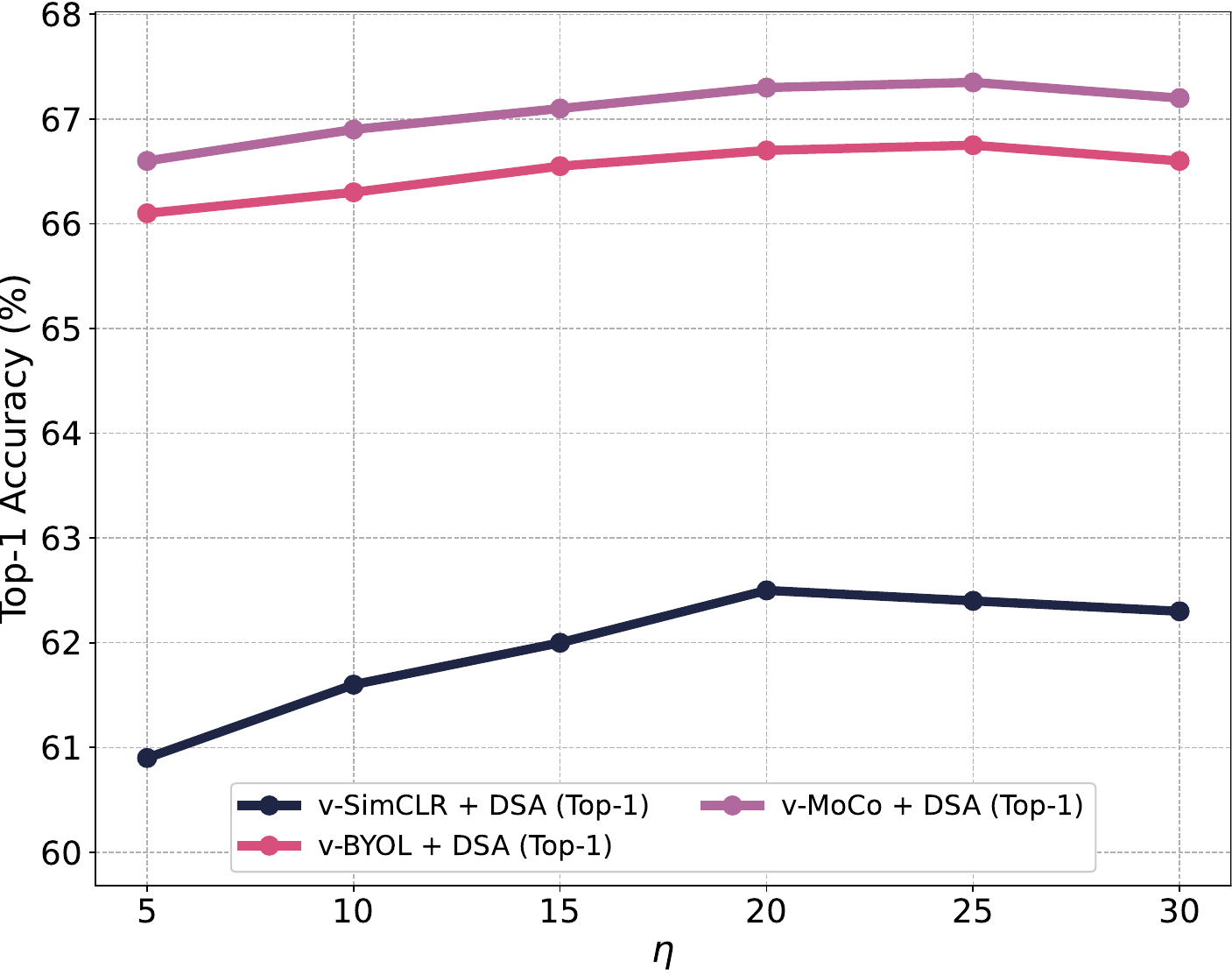}\label{fig:eta_k400}}
    \subfigure[$\tau$]{\includegraphics[width=0.3\textwidth]{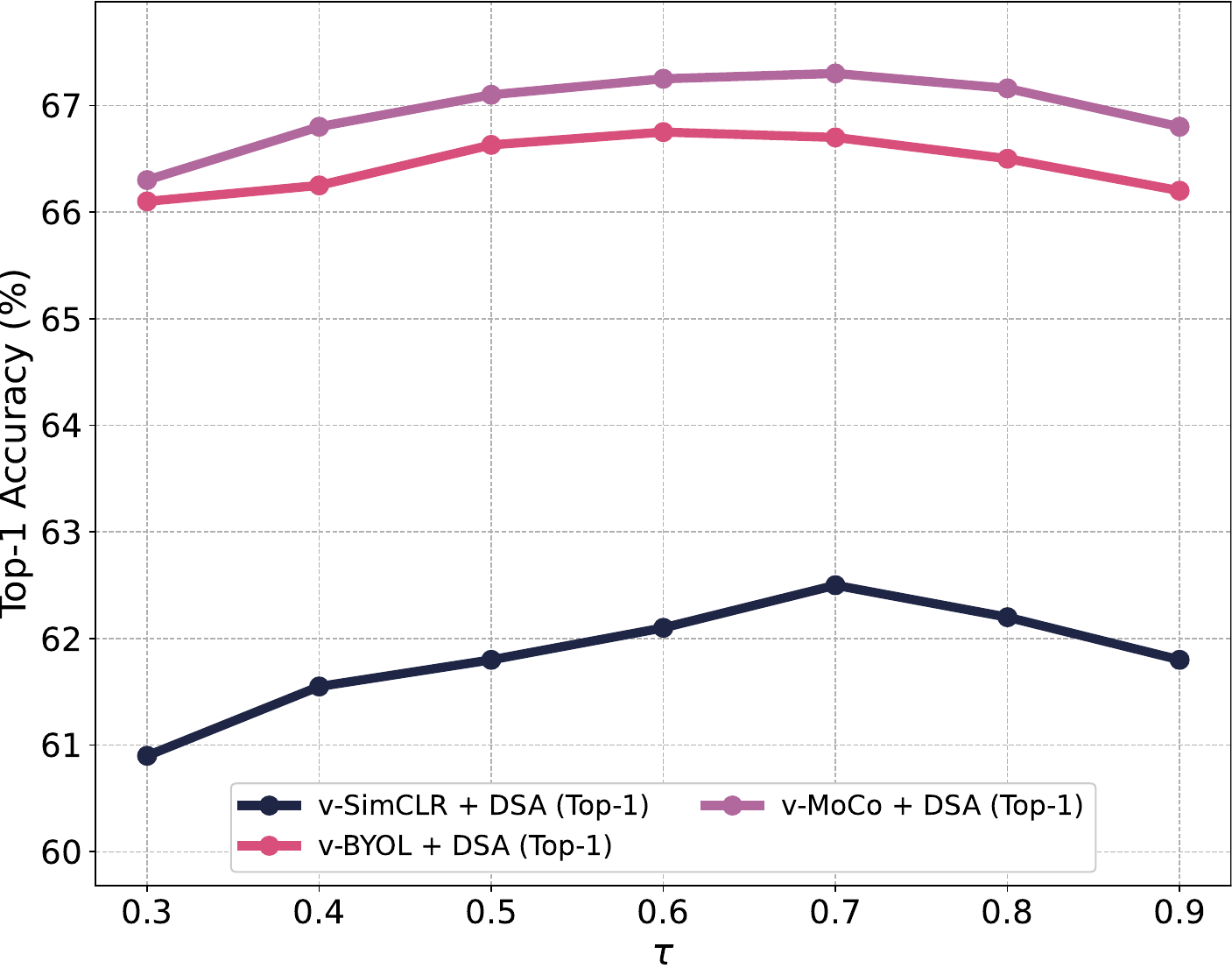}\label{fig:tau_k400}}

    \caption{{The Top-1 linear classification accuracies of Kinetics 400 correspond to different values of hyper-parameters $\nu$, $\upsilon$, $\alpha$, $\eta$, and $\tau$. }}
    \label{fig:hyperparam_3}
\end{figure*}

\begin{figure*}[htb]
    \centering
    \subfigure[$\nu$]{\includegraphics[width=0.3\textwidth]{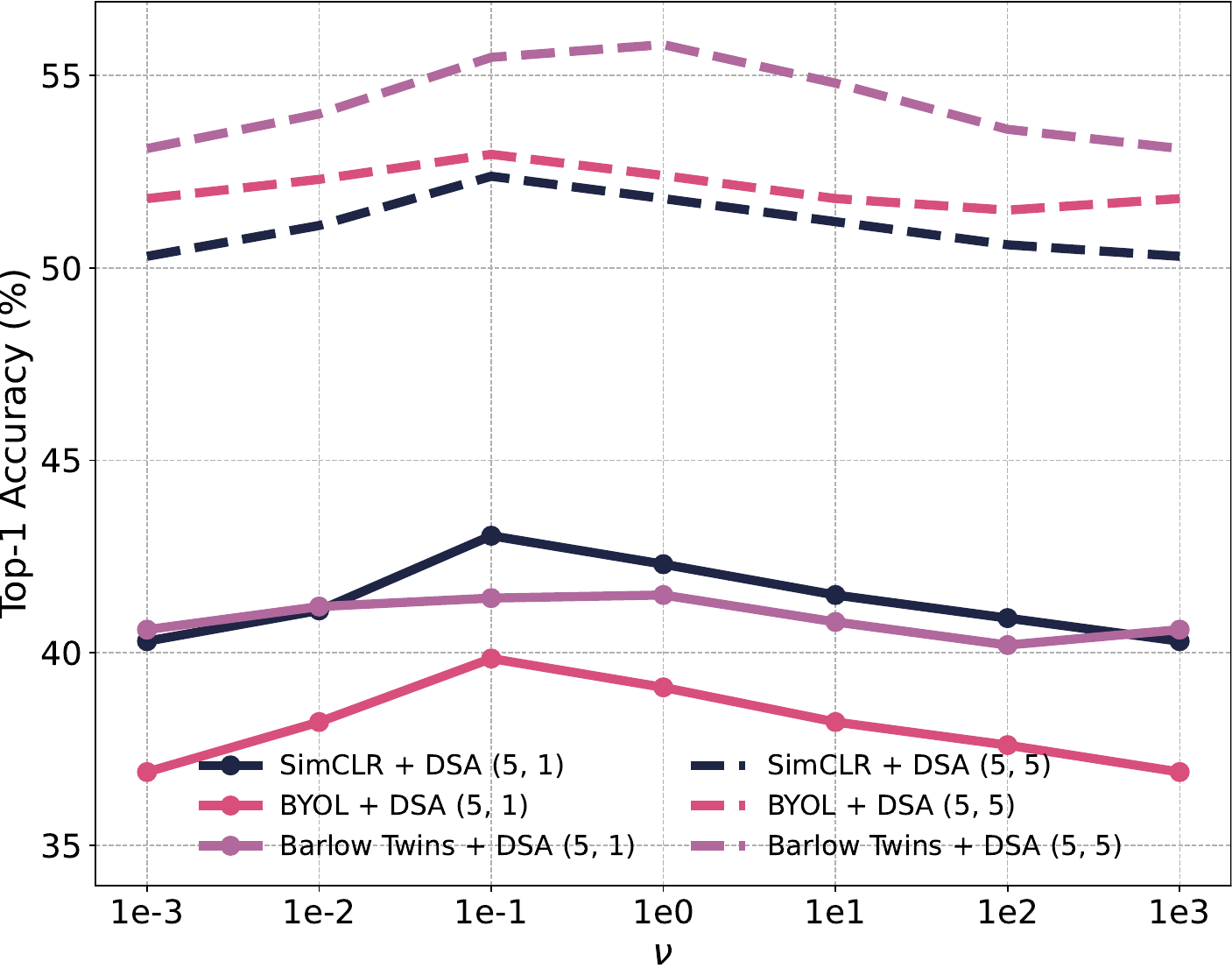}\label{fig:nu_fc100}}
    \subfigure[$\upsilon$]{\includegraphics[width=0.3\textwidth]{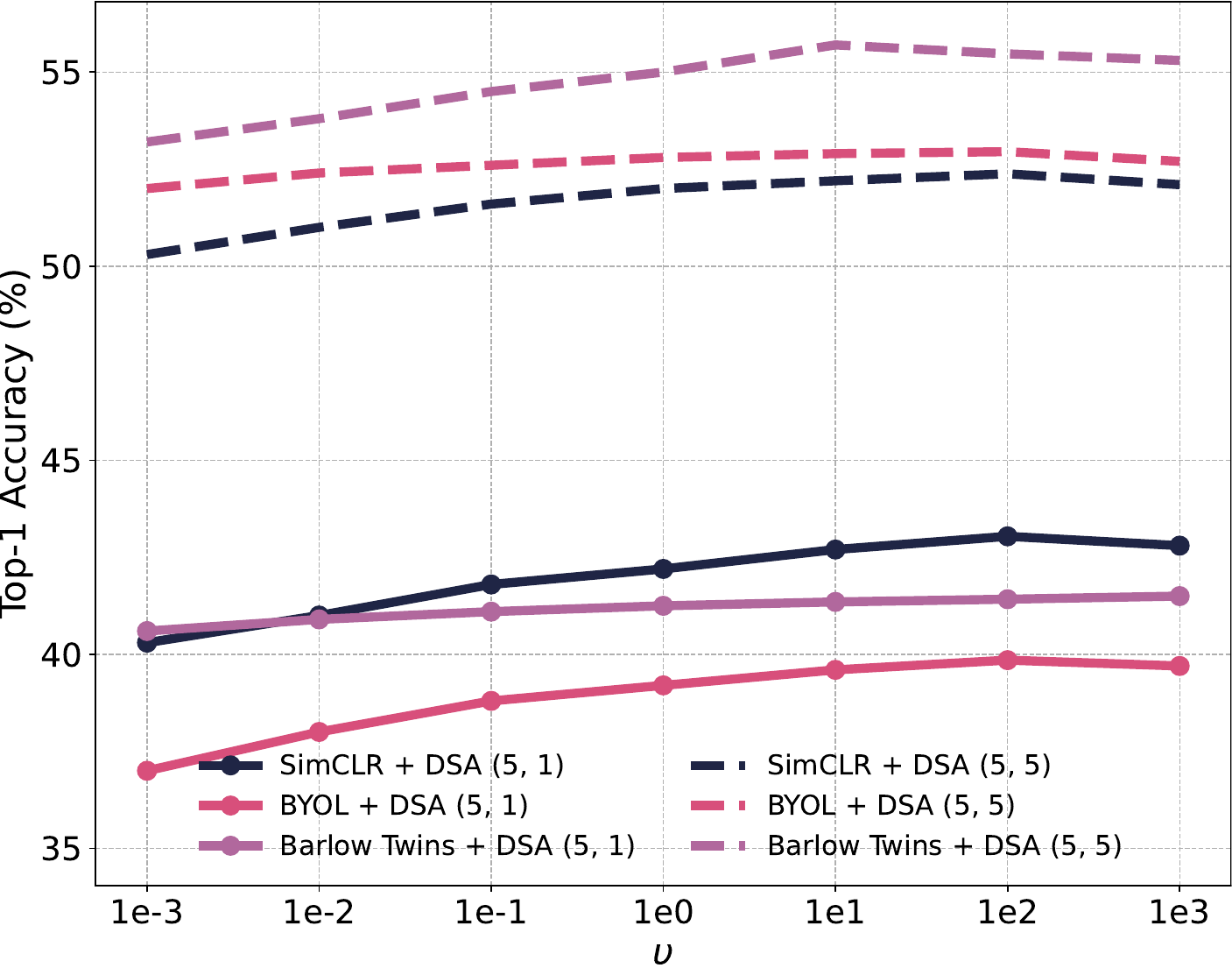}\label{fig:upsilon_fc100}}
    \subfigure[$\alpha$]{\includegraphics[width=0.3\textwidth]{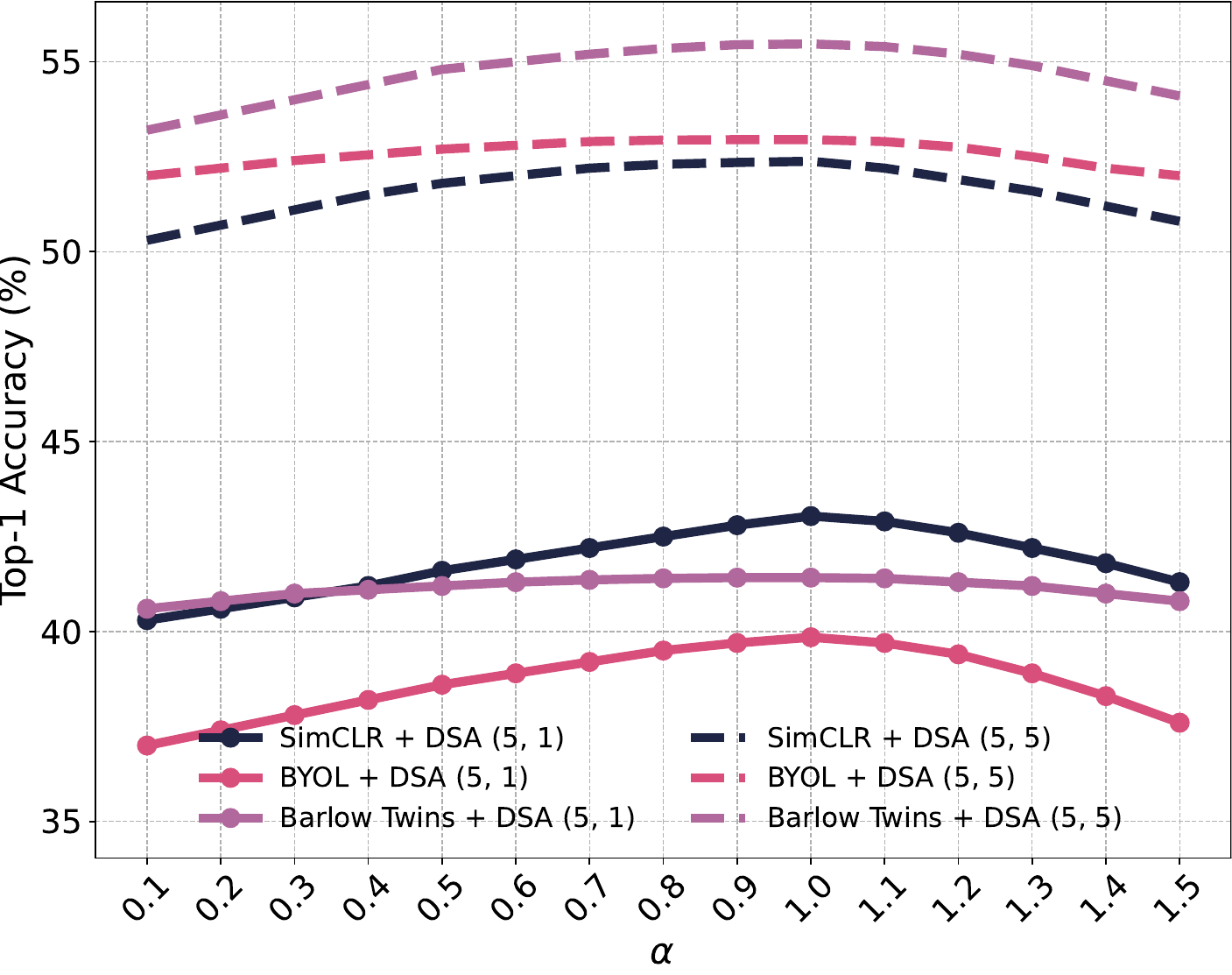}\label{fig:alpha_fc100}}
    \subfigure[$\eta$]{\includegraphics[width=0.3\textwidth]{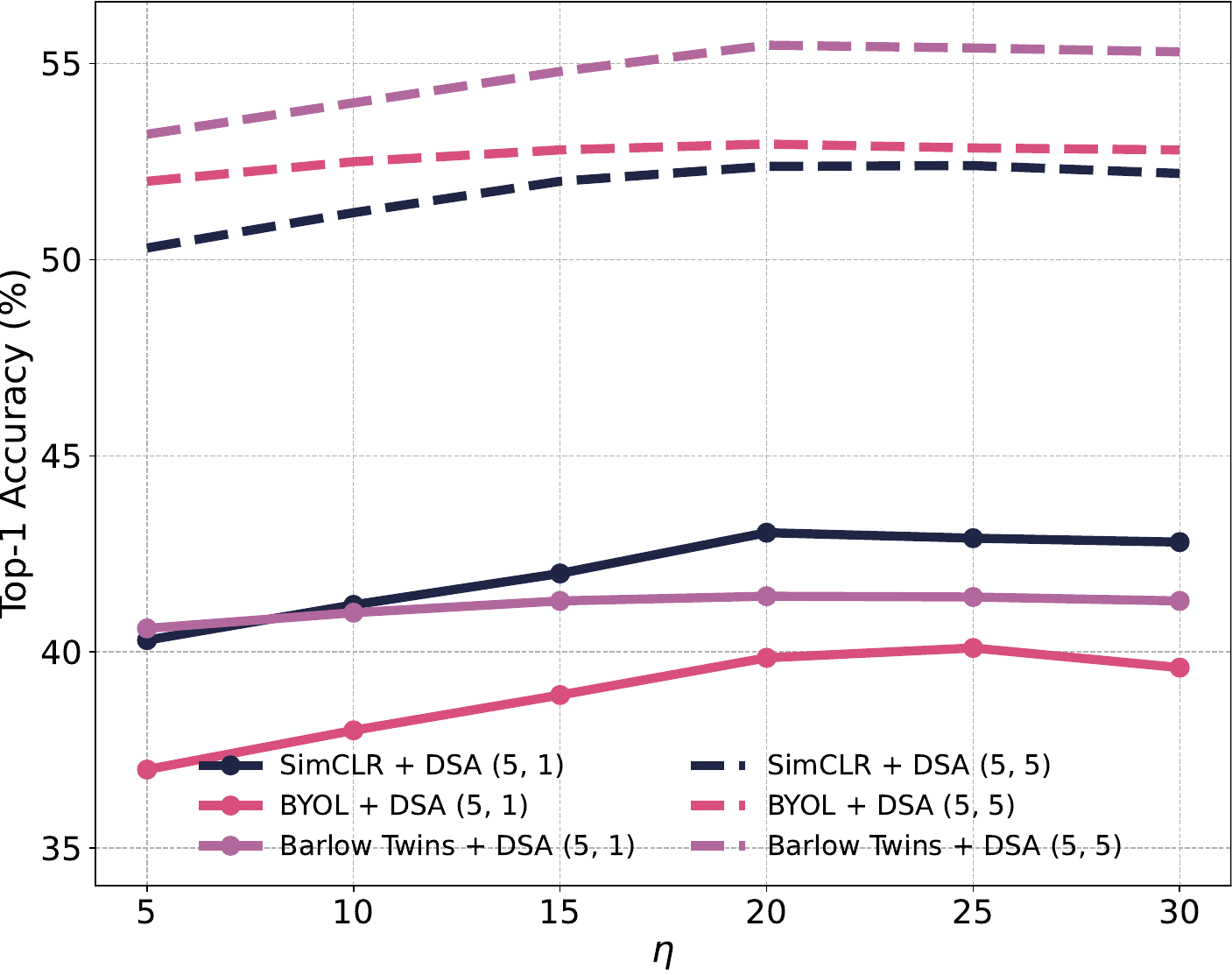}\label{fig:eta_fc100}}
    \subfigure[$\tau$]{\includegraphics[width=0.3\textwidth]{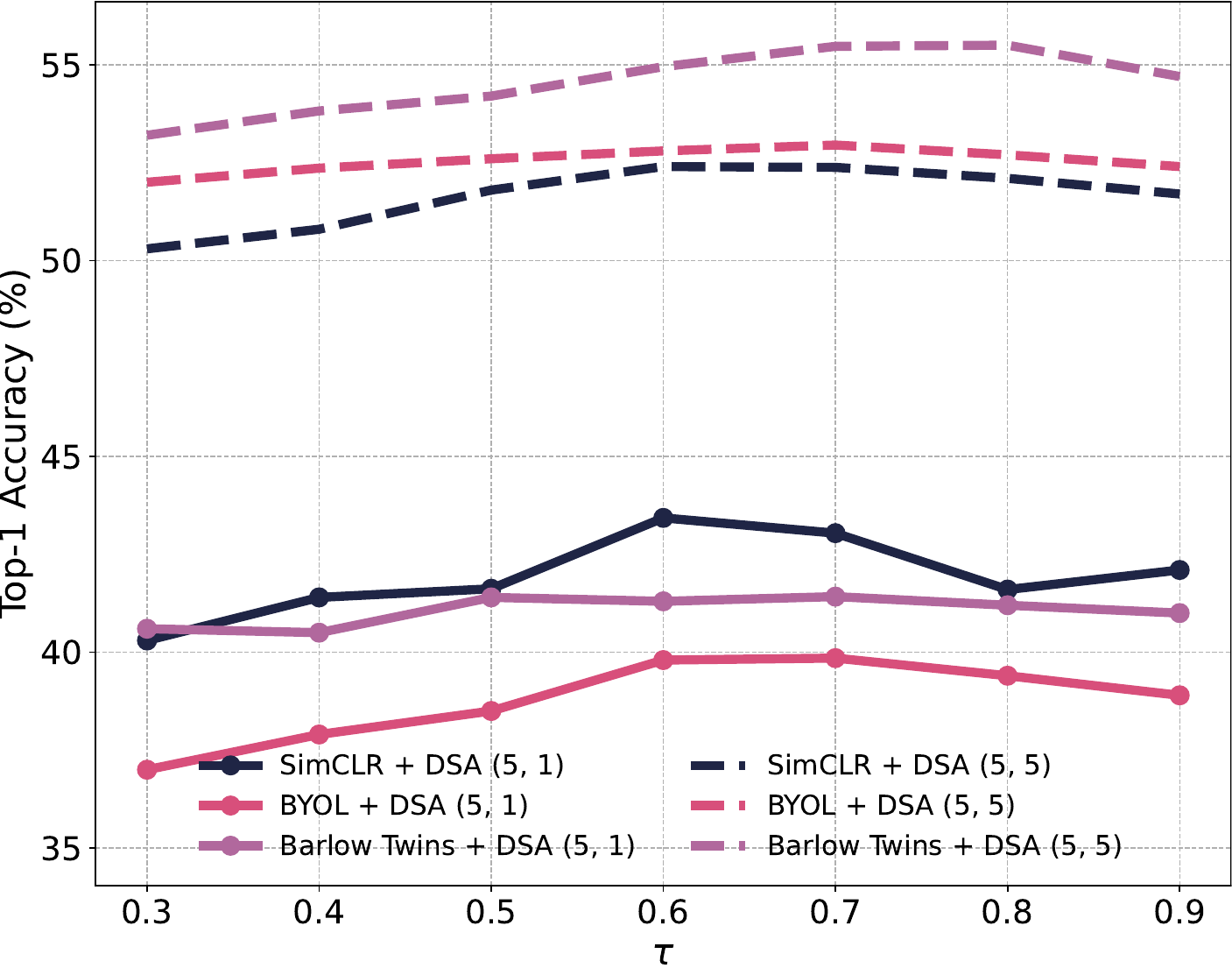}\label{fig:tau_fc100}}

    \caption{{The 5-way 1-shot and 5-way 5-shot accuracy correspond to different values of hyper-parameters $\nu$, $\upsilon$, $\alpha$, $\eta$, and $\tau$. The solid lines in figures (a) - (e) represent the 5-way 1-shot accuracies, while the dashed lines represent the 5-way 5-shot accuracies.}}
    \label{fig:hyperparam_4}
\end{figure*}

\begin{figure*}[htb]
    \centering
    \subfigure[$\nu$]{\includegraphics[width=0.3\textwidth]{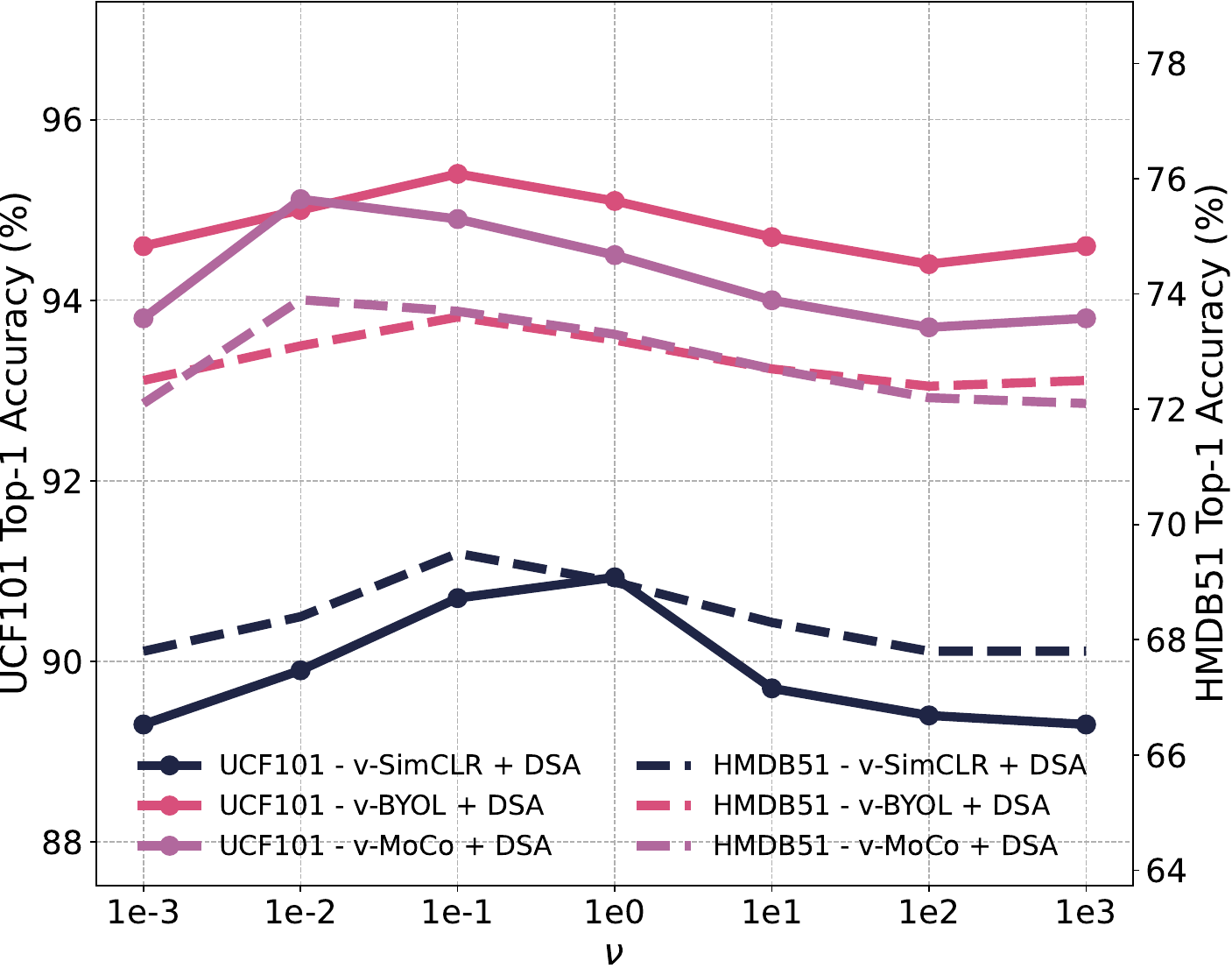}\label{fig:nu_ucfhmdb}}
    \subfigure[$\upsilon$]{\includegraphics[width=0.3\textwidth]{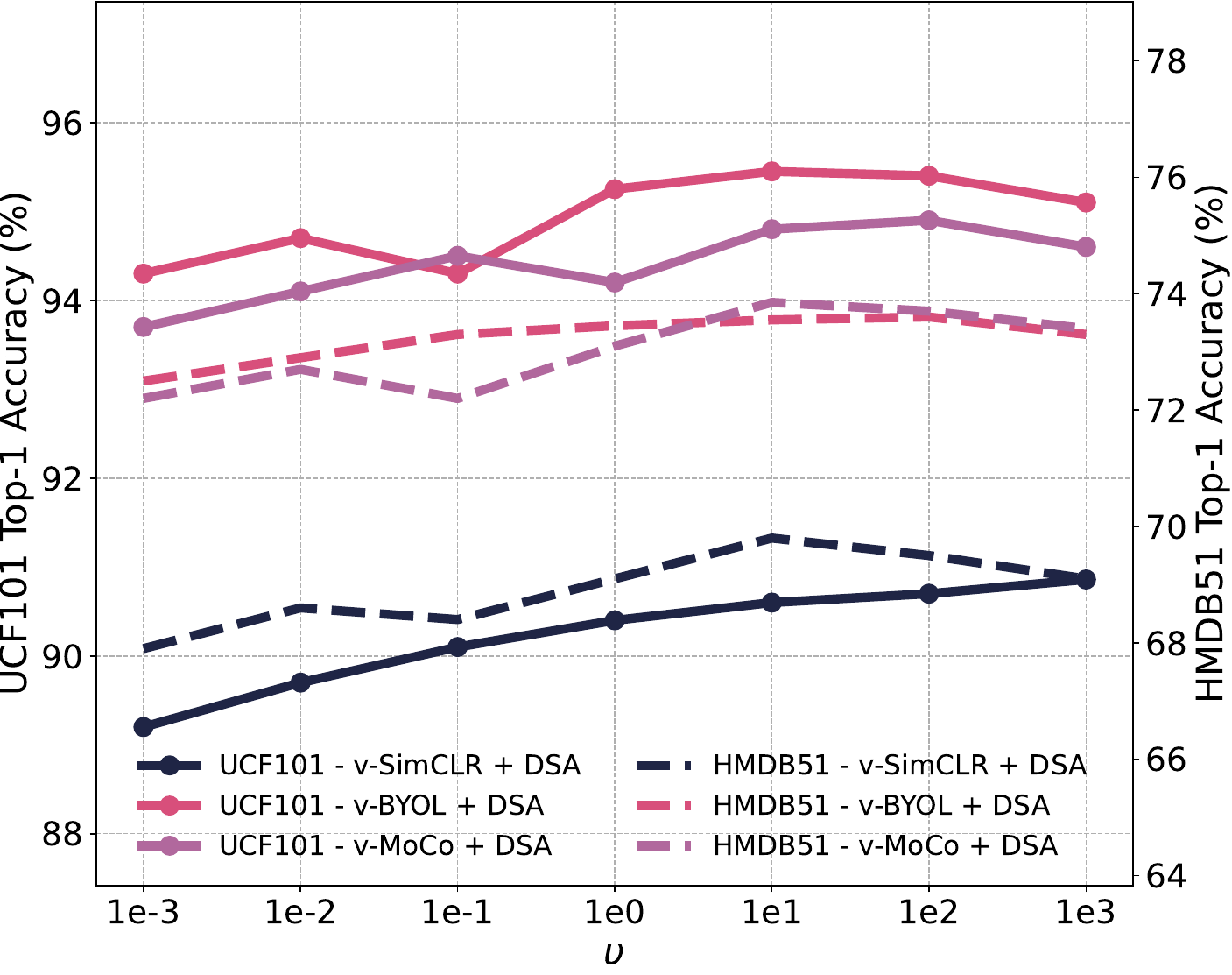}\label{fig:upsilon_ucfhmdb}}
    \subfigure[$\alpha$]{\includegraphics[width=0.3\textwidth]{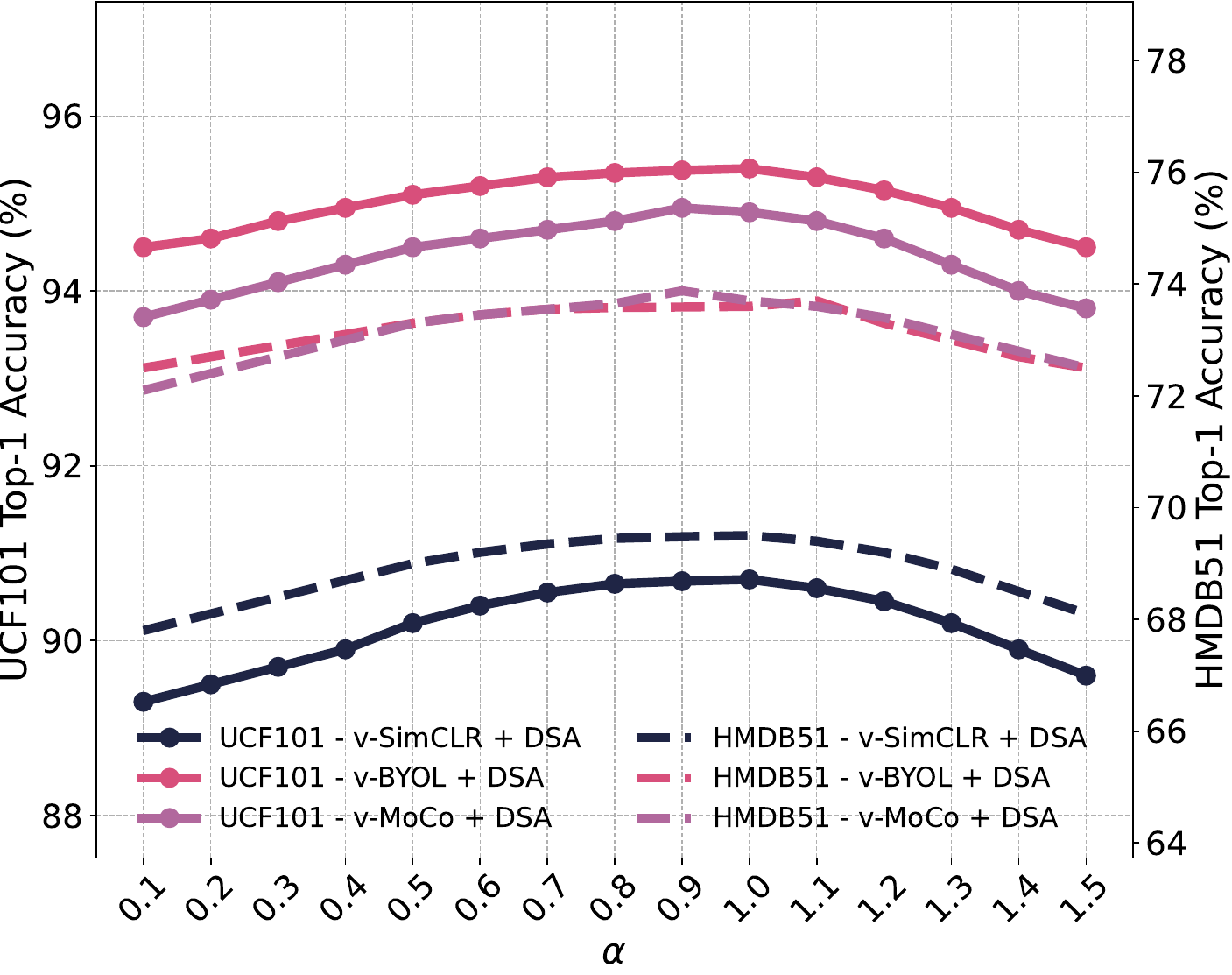}\label{fig:alpha_ucfhmdb}}
    \subfigure[$\eta$]{\includegraphics[width=0.3\textwidth]{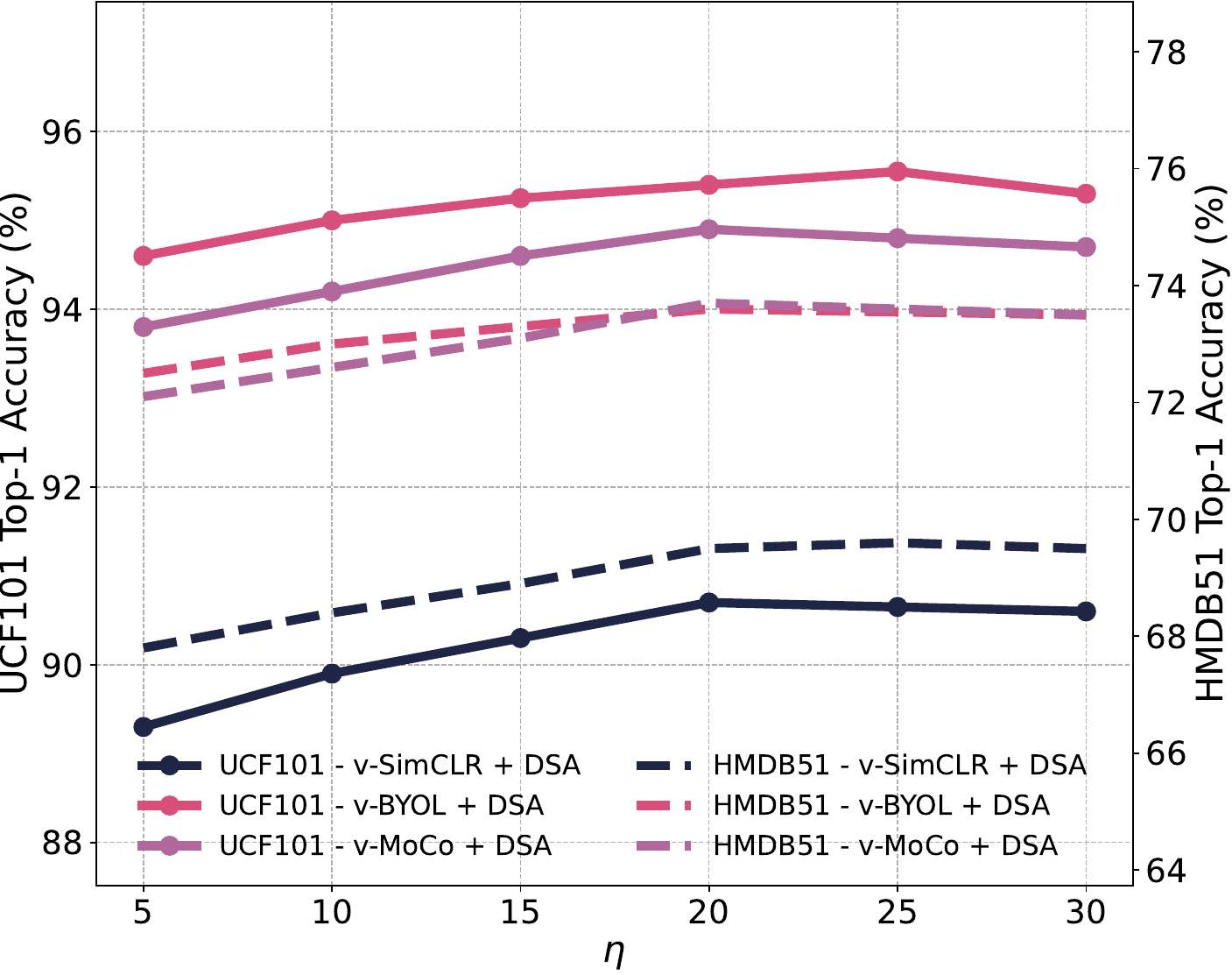}\label{fig:eta_ucfhmdb}}
    \subfigure[$\tau$]{\includegraphics[width=0.3\textwidth]{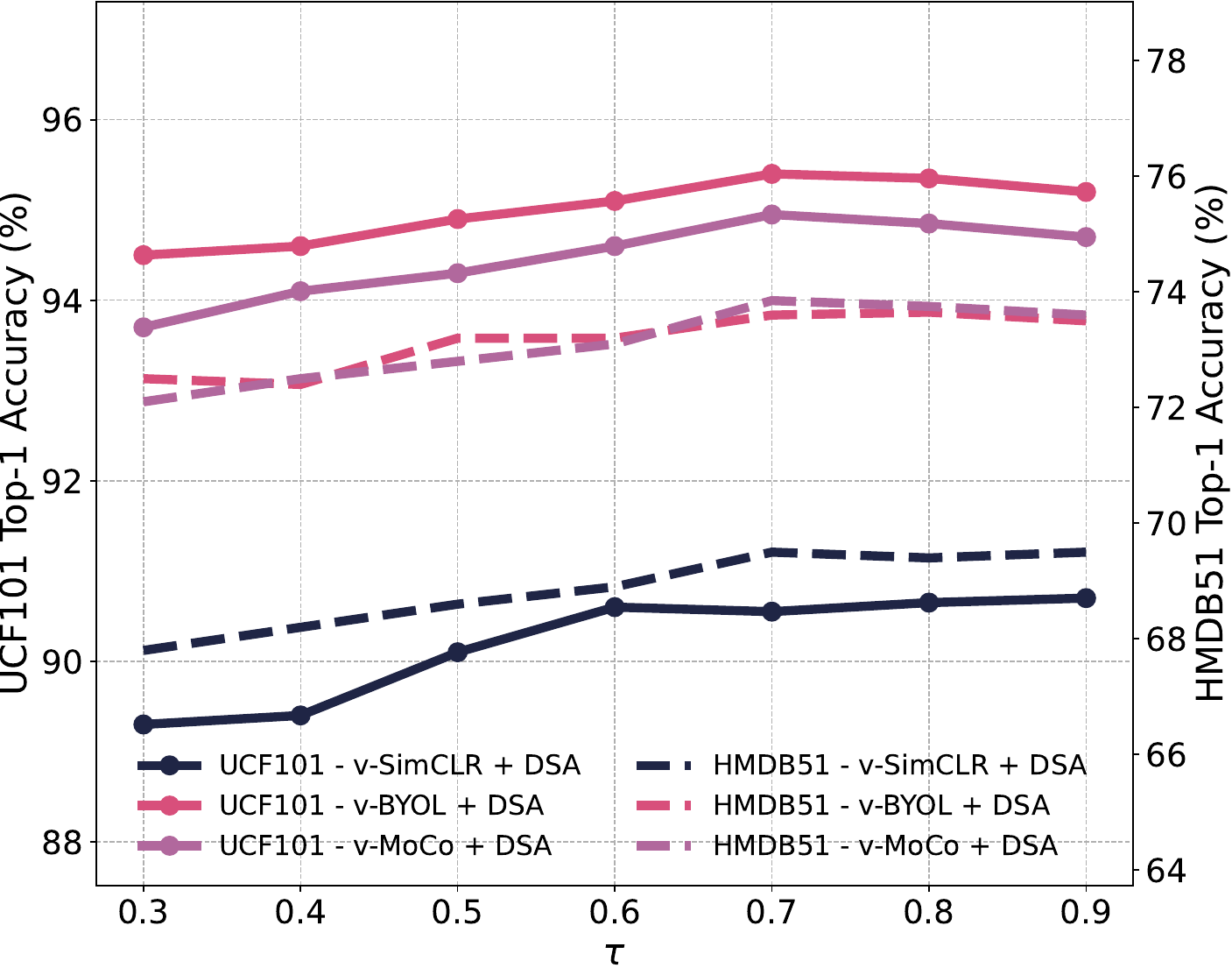}\label{fig:tau_ucfhmdb}}

    \caption{{The accuracies of UCF-101 and HMDB-51 datasets correspond to different values of hyper-parameters $\nu$, $\upsilon$, $\alpha$, $\eta$, and $\tau$. The solid lines in figures (a) - (e) represent the accuracies of UCF-101, while the dashed lines represent the accuracies of HMDB-51.}}
    \label{fig:hyperparam_5}
\end{figure*}






\end{document}